\documentclass{article}

\usepackage[margin=1in]{geometry}
\usepackage[utf8]{inputenc} 
\usepackage[T1]{fontenc}    
\usepackage{hyperref}       
\usepackage{url}            
\usepackage{booktabs}       
\usepackage{amsfonts}       
\usepackage{nicefrac}       
\usepackage{microtype}      
\usepackage[dvipsnames]{xcolor}         
\usepackage{natbib}
\setcitestyle{authoryear,open={(},close={)}} 
\usepackage{mathrsfs}
\usepackage{bm}
\usepackage[linesnumbered,ruled,vlined]{algorithm2e}
\usepackage{amsmath,amssymb,amsfonts,mathtools}
\usepackage{amsthm}
\usepackage{hyperref}\hypersetup{colorlinks,linkcolor=red,anchorcolor=blue,citecolor=blue}
\usepackage{accents}
\usepackage{enumitem}
\usepackage{xfrac}

\allowdisplaybreaks

\makeatletter
\newtheorem*{rep@theorem}{\rep@title}
\newcommand{\newreptheorem}[2]{%
	\newenvironment{rep#1}[1]{%
		\def\rep@title{#2 \ref{##1}}%
		\begin{rep@theorem}}%
		{\end{rep@theorem}}}
\makeatother

\usepackage{color}
\definecolor{yc}{RGB}{255,69,0}
\definecolor{pv}{RGB}{0,102,204}

\newcommand{\topk}{\text{top}_k}
\newtheorem{theorem}{Theorem}
\newreptheorem{theorem}{Theorem}
\newtheorem{lemma}{Lemma}

\newtheorem{definition}{Definition}
\newreptheorem{definition}{Definition}
\newtheorem{assumption}{Assumption}
\newreptheorem{assumption}{Assumption}
\usepackage{subcaption}

\DeclareMathOperator{\argmin}{argmin}

\title{Communication-efficient Vertical Federated Learning via Compressed Error Feedback}

\author{
Anonymous Author(s)\\
Affiliation
\footnote{Full affiliation}
}

\author{%
	Pedro Valdeira\footnote{Department of Electrical and Computer Engineering, Carnegie Mellon University; email: \texttt{pvaldeira@cmu.edu}.} \footnote{Institute for Systems and Robotics, Instituto Superior T\'{e}cnico.}\\
	CMU \& IST
	\and Jo\~{a}o Xavier$^\dagger$\\
	IST
	\and Cl\'{a}udia Soares\footnote{Department of Computer Science, NOVA School of Science and Technology.}\\
	NOVA
	\and
	Yuejie Chi$^{*}$\\
	CMU
}

\date{June 2024; Revised \today}

\begin{document}

	\maketitle
	
	\footnotetext[4]{A preliminary version of this work was accepted for publication at EUSIPCO 2024~\citep{valdeira2024efvfl}.}
	\renewcommand*{\thefootnote}{\arabic{footnote}}
		
	\begin{abstract}

Communication overhead is a known bottleneck in federated learning (FL). To address this, lossy compression is commonly used on the information communicated between the server and clients during training. In horizontal FL, where each client holds a subset of the samples, such communication-compressed training methods have recently seen significant progress. However, in their vertical FL counterparts, where each client holds a subset of the features, our understanding remains limited. To address this, we propose an error feedback compressed vertical federated learning (\texttt{EF-VFL}) method to train split neural networks. In contrast to previous communication-compressed methods for vertical FL, \texttt{EF-VFL} does not require a vanishing compression error for the gradient norm to converge to zero for smooth nonconvex problems. By leveraging error feedback, our method can achieve a $\mathcal{O}(\sfrac{1}{T})$ convergence rate for a sufficiently large batch size, improving over the state-of-the-art $\mathcal{O}(\sfrac{1}{\sqrt{T}})$ rate under $\mathcal{O}(\sfrac{1}{\sqrt{T}})$ compression error, and matching the rate of uncompressed methods. Further, when the objective function satisfies the Polyak-Łojasiewicz inequality, our method converges linearly. In addition to improving convergence, our method also supports the use of private labels. Numerical experiments show that \texttt{EF-VFL} significantly improves over the prior art, confirming our theoretical results. The code for this work can be found at \url{https://github.com/Valdeira/EF-VFL}.

\end{abstract}
	
\section{Introduction}
Federated learning (FL) is a machine learning paradigm where a set of clients holding local datasets collaborate to train a model without exposing their local data~\citep{McMahan2017,zhang2021fedpd,sery2021over}. FL can be divided into two categories, based on how the data is partitioned across the clients: \textit{horizontal} FL, where each client holds a different set of samples but all clients share the same features, and \textit{vertical} FL, where each client holds a different subset of features but all clients share the same samples. Note that we cannot gather and redistribute the data because it must remain at the clients. Thus, we do not choose under which category a given task falls. Rather, the category is a consequence of how the data arises.

In this work, we focus on vertical FL (VFL)~\citep{liu2024vertical}. In VFL, a global dataset~$\mathcal{D}=\{\bm{\xi}_n \}_{n=1}^N$ with $N$ samples is partitioned by features across a set of clients~$[K]\coloneqq\{1,\dots,K\}$. Each sample has $K$ disjoint blocks of features $\bm{\xi}_{n}=(\bm{\xi}_{n1},\dots,\bm{\xi}_{nK})$ and the local dataset of each client~$k\in[K]$ is $\mathcal{D}_k=\{\bm{\xi}_{nk}\}_{n=1}^N$, where $\mathcal{D}=\bigcup_k\mathcal{D}_k$. Since different datasets~$\mathcal{D}_k$ have different features, VFL suits collaborations of clients with complementary types of information, who tend to have fewer competing interests. This can lead to a greater incentive to collaborate, compared to horizontal FL. A common application of VFL is in settings where multiple entities own distinct features concerning a shared set of users and seek to collaboratively train a predictor; for example, WeBank partners with other companies to jointly build a risk model from data regarding shared customers~\citep{cheng2020federated}.

To jointly train a model from~$\{\mathcal{D}_k\}$ without sharing local data, split neural networks~\citep{ceballos2020splitnn} are often considered. To learn the parameters $\bm{x}$ of such models, we aim to solve the following nonconvex optimization problem:
\begin{equation}
	\label{eq:splitnn}
	\min_{\bm{x}\in\mathbb{R}^d}
	\;
	f(\bm{x})
	\coloneqq
	\frac{1}{N}
	\sum_{n=1}^N
	\phi_n
	\left(\bm{x}_0,\bm{h}_{1n}(\bm{x}_1),\dots,\bm{h}_{Kn}(\bm{x}_K)\right),
\end{equation}
where $\bm{x}\coloneqq(\bm{x}_0,\bm{x}_1,\dots,\bm{x}_K)$. Here, $\bm{h}_{kn}(\bm{x}_k)\coloneqq \bm{h}_{k}(\bm{x}_k;\bm{\xi}_{nk})$ is the \textit{representation} of $\bm{\xi}_{nk}$ extracted by the local model of client~$k\in[K]$, which is parameterized by $\bm{x}_k$. This representation is then sent to the server. The server, in turn, uses $\{\bm{h}_{kn}(\bm{x}_k)\}_{k=1}^K$ as input to $\phi_n$, which corresponds to the composition of the loss function and the \textit{fusion} model and is parameterized by $\bm{x}_0$.\footnote{The server often aggregates the representations $\bm{h}_{kn}$ via some nonparameterized operation (for example, a sum or an average) before inputting them into the server model. We consider this aggregation to be included in $\phi_n$.}

Most FL methods, including ours, assume that the server can communicate with all the clients and that the clients do not communicate with each other. These methods typically require many rounds of client-server communications. Such communications can significantly slow down training. In fact, they can become the main bottleneck during training~\citep{dean2012large,Lian2017}. To address this, a plethora of communication-efficient FL methods have been proposed. In particular, a popular technique to mitigate the communication overhead is lossy compression. Compression operators, or simply compressors, are operators that map a given vector into another vector that is easier to communicate (that is, requires fewer bits).

Optimization methods employing communication compression have seen great success~\citep{seide20141,alistarh2017qsgd}. Most of these works focus on the prevalent horizontal FL setup and thus consider \textit{gradient} compression, as these are the vectors being communicated in the horizontal setting~\citep{alistarh2018convergence,stich2018sparsified}. Yet, in vertical FL, the clients send \textit{representations}~$\{\bm{h}_{kn}(\bm{x}_k)\}$ instead. In contrast to gradient compression, compressing these intermediate representations leads the compression-induced error to undergo a nonlinear function $\phi_n$ before impacting gradient-based updates. Thus, compression in VFL is not covered by these works and our understanding of it remains limited. In fact, to the best of our knowledge, \citet{Castiglia2022} is the only work providing convergence guarantees for compressed VFL. Yet, \citet{Castiglia2022} employs a direct compression method, requiring the compression error to go to zero as the number of gradient steps $T$ increases, achieving a $\mathcal{O}(\sfrac{1}{\sqrt{T}})$ rate when the compression error is $\mathcal{O}(\sfrac{1}{\sqrt{T}})$. This makes the method in~\citet{Castiglia2022} unsuitable for applications with strict per-round communication limitations, such as bandwidth constraints, and leads to the following question:
\begin{center}
	\emph{Can we design a communication-compressed VFL method that preserves the convergence rate of uncompressed methods without decreasing the amount of compression as training progresses?}
\end{center}

\subsection{Our contributions}
In this work, we answer the question above in the affirmative. Our main contributions are as follows.
\begin{itemize} [leftmargin=1.2em]
	\item We propose error feedback compressed VFL (\texttt{EF-VFL}), which leverages an error feedback technique to improve the stability of communication-compressed training in VFL.
	
	\item We show that our method achieves a convergence rate of $\mathcal{O}(\sfrac{1}{{T}})$ for nonconvex objectives under nonvanishing compression error for a sufficiently large batch size, improving over the state-of-the-art $\mathcal{O}(\sfrac{1}{\sqrt{T}})$ rate under $\mathcal{O}(\sfrac{1}{\sqrt{T}})$ compression error of~\citet{Castiglia2022}, and matching the rate of the centralized setting. We further show that, under the Polyak-Łojasiewicz (PL) inequality, our method converges linearly to the solution, in the full-batch case, and, more generally, to a neighborhood of size proportional to the mini-batch variance, thus obtaining the first linearly convergent compressed VFL method. Unlike the method in~\citet{Castiglia2022}, \texttt{EF-VFL} supports the use of private labels, broadening its applicability.
	
	\item We run numerical experiments and observe empirically that our method improves the state-of-the-art, achieving a better communication efficiency than existing methods.
\end{itemize}

\begin{figure*}[t]
	\centering
	\includegraphics[width=.9\textwidth]{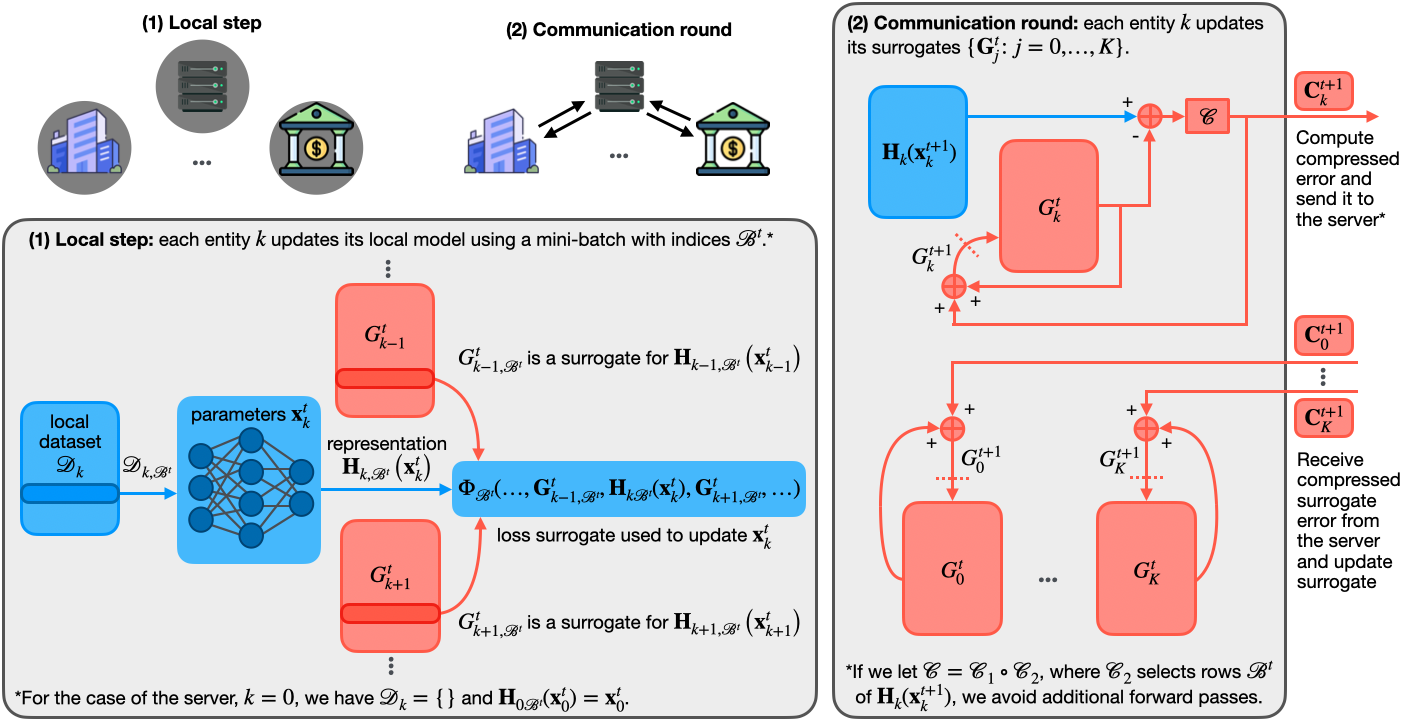}
	\caption{An illustration of an iteration of the \texttt{EF-VFL} algorithm. Step (1) concerns the model update and step (2) concerns the surrogate update.}
	\label{fig:efvfl_model}
\end{figure*}

\subsection{Related work}
\label{sec:related_work}

\paragraph{Communication-efficient FL.}
The aforementioned communication bottleneck in FL~\citep{Lian2017} makes communication-efficient methods a particularly active area of research. FL methods often employ multiple local updates between rounds of communication, use only a subset of the clients at a time~\citep{McMahan2017}, or even update the global model asynchronously~\citep{xie2019asynchronous}. Another line of research considers fully decentralized methods~\citep{Lian2017}, dispensing with the server and, instead, exploiting communications between clients. This can alleviate the bandwidth limit ensuing from the centralized role of the server. Another popular technique is lossy compression, which is the focus of this work. In both FL and, more generally, in the broader area of distributed optimization~\citep{chen2012diffusion, mota2013d}, communication-compressed methods have recently received significant attention~\citep{amiri2020machine, du2020high, shlezinger2020uveqfed, nassif2023quantization}.

Communication-compressed optimization methods can exploit different families of compressors. A popular choice is the family of \textit{unbiased} compressors~\citep{alistarh2017qsgd}, which is appealing in that its properties facilitate the theoretical analysis of the resulting methods. Yet, some widely adopted compressors do not belong to this class, such as top-$k$ sparsification~\citep{alistarh2018convergence,stich2018sparsified} and deterministic rounding~\citep{sapio2021scaling}.
Thus, the broader family of \textit{contractive} compressors~\citep{beznosikov2023biased} has recently attracted a lot of attention.
Yet, methods employing them for direct compression are often prone to instability or even divergence~\citep{beznosikov2023biased}. To address this, error feedback techniques have been proposed; first, as a heuristic~\citep{seide20141}, but, more recently, significant progress has been made on our theoretical understanding of the application of these methods to gradient compression~\citep{stich2018sparsified,alistarh2018convergence,richtarik2021ef21}. In the horizontal setting, some works have combined error feedback compression with the aforementioned communication-efficient techniques, such as communication-compressed fully decentralized methods~\citep{koloskova2019decentralized, zhao2022beer}.

\vspace{1mm}
\paragraph{Vertical FL.}
To mitigate the communication bottleneck in VFL, we often employ techniques akin to those used in horizontal FL. In particular, \citet{Liu2022} performed multiple local updates between communication rounds, \citet{Chen2020} updated the local models asynchronously, and \citet{ying2018supervised} proposed a fully decentralized approach. Recently, \citet{valdeira2023multi} proposed a semi-decentralized method leveraging both client-server and client-client communications to avoid the slow convergence of fully-decentralized methods on large and sparse networks while alleviating the server bottleneck.

In this work, we focus on communication-compressed methods tailored to VFL. While, as mentioned above, most work in communication-compressed methods focuses on gradient compression and thus does not apply directly to VFL, recently, a few empirical works on VFL have employed compression. Namely, \citet{khan2022communication} compressed the local data before sending it to the server, where the model is trained, and \citet{li2020efficient} proposed an asynchronous method with bidirectional (sparse) gradient compression. Nevertheless, \citet{Castiglia2022}, where direct compression is used, is the only work on compressed VFL with theoretical guarantees. For a more a detailed discussion on VFL, see~\citet{liu2024vertical,yang2023survey}.
	
\section{Preliminaries}

We now define the class of contractive compressors, which we consider throughout this paper.
\begin{definition}[Contractive compressor] \label{def:contractive_compressor}
	A map $\mathcal{C}\colon\mathbb{R}^d\mapsto\mathbb{R}^d$ is a contractive compressor if there exists $\alpha\in(0,1]$ such that,\footnote{In Definition~\ref{def:contractive_compressor}, we use a common, simplified notation, omitting the randomness in $\mathcal{C}\colon\mathbb{R}^d\times\Omega\to\mathbb{R}^d$. We assume that the randomness $\omega$ in $\mathcal{C}(\bm{v},\omega)$ at different applications of $\mathcal{C}$ is independent.}
	\begin{equation} \label{eq:biased_compressor}
		\forall\bm{v}\in\mathbb{R}^d\colon
		\quad
		\mathbb{E}\lVert\mathcal{C}(\bm{v})-\bm{v}\rVert^2
		\leq (1-\alpha)\lVert\bm{v}\rVert^2,
	\end{equation}
	where the expectation is taken with respect to the (possible) randomness in $\mathcal{C}$.
\end{definition}

\subsection{Error feedback}

{When transmitting a converging sequence of vectors $\{\bm{v}^t\}$, error feedback mechanisms can reduce the compression error compared to direct compression. In communication-compressed optimization, this allows for faster and more stable convergence.}

{In direct compression, each $\bm{v}^{t}$ is compressed independently, with $\mathcal{C}(\bm{v}^t)$ simply replacing $\bm{v}^t$ at the receiver. In contrast, in error feedback compression, the receiver employs a surrogate for $\bm{v}^t$ that incorporates information from previous steps $i=0,\dots,t-1$. This is achieved by resorting to an auxiliary vector that is stored in memory and updated at each step, leveraging \textit{feedback} from the compression of $\{\bm{v}^i\colon i=0,\dots,t-1 \}$ to refine the surrogate for $\bm{v}^t$.}

{Earlier communication-compressed optimization methods employing error feedback mechanisms were motivated by sigma-delta modulation~\citep{inose1962telemetering}. In these methods, the auxiliary vector accumulated the compression \textit{error} across steps, adding the accumulated error to the current vector $\bm{v}^t$ before compressing and transmitting it~\citep{seide20141, alistarh2018convergence, karimireddy2019error}.}

{More recently, a new type of error feedback mechanism has been proposed, where the auxiliary vector $\bm{s}^t$ tracks $\bm{v}^t$ directly, rather than the accumulated compression error. This mechanism uses the (compressed) difference between $\bm{v}^{t+1}$ and $\bm{s}^t$, that is, the \textit{error} of the surrogate, as the feedback. This approach was first introduced in~\citet{mishchenko2024distributed} for unbiased compressors and was later extended to the more general class of contractive compressors in EF21~\citep{richtarik2021ef21}. More formally, in EF21, the surrogate $\bm{s}^t$ is initialized as $\bm{s}^0=\mathcal{C}(\bm{v}^0)$ and updated recursively as:
	\begin{equation} \label{eq:error_feedback}
		\forall t\geq0\colon
		\quad
		\bm{s}^{t+1}
		=
		\bm{s}^{t}+\mathcal{C}(\bm{v}^{t+1}-\bm{s}^{t}).
	\end{equation}
	Note that, unlike $\mathcal{C}(\bm{v}^t)$, $\bm{s}^t$ is not necessarily in the range of $\mathcal{C}$. For example, if $\mathcal{C}$ uses sparsification, then the surrogate at the receiver in direct compression, $\mathcal{C}(\bm{v}^t)$, must be sparse, whereas $\bm{s}^t$ does not need to be. By continually updating $\bm{s}^t$ with the compressed error, we ensure that it tracks $\bm{v}^t$. Moreover, since these updates are compressed, the communication cost remains the same as in direct compression. To maintain consistency, the surrogate $\bm{s}^t$ is updated at both the sender (client or server) and the receiver.
	In this work, we adopt an EF21-based error feedback mechanism and henceforth refer to it simply as error feedback.}

\subsection{Problem setup}
Let $\bm{h}_{0n}(\bm{x}_0)=\bm{x}_0$ and $
\bm{h}_n(\bm{x})= (\bm{h}_{0n}(\bm{x}_0),\dots,\bm{h}_{Kn}(\bm{x}_K))\in\mathbb{R}^{E}$ where $\bm{h}_{kn}(\bm{x}_k) \in\mathbb{R}^{E_k}$ and $E=\sum_{k=0}^K E_k$, for all $n$. Further, let
\[
\bm{H}_{k}(\bm{x}_k)
=
\begin{bmatrix}
	\bm{h}_{k1}(\bm{x}_k)\\
	\vdots\\
	\bm{h}_{kN}(\bm{x}_k)
\end{bmatrix}
\in\mathbb{R}^{N\times E_k}
\]
and
\[
\bm{H}(\bm{x})\coloneqq
[\bm{H}_0(\bm{x}_0),\bm{H}_1(\bm{x}_1), \dots, \bm{H}_K(\bm{x}_K)]
\in\mathbb{R}^{N\times E}.
\]
Further, we define $\Phi\colon \mathbb{R}^{N\times E} \mapsto \mathbb{R}$ as follows:
\[
f(\bm{x})
=
\frac{1}{N} \sum_{n=1}^N \phi_n \left(\bm{h}_n(\bm{x})\right)
\eqqcolon
\Phi
\left(\bm{H}(\bm{x})\right).
\]
Throughout most of the paper, we assume that $\phi_n$ contains the label of $\bm{\xi}_n$ and that $\phi_n$ is known by all the clients and the server. This assumption, known as ``relaxed protocol''~\citep{liu2024vertical}, is sometimes made in VFL~\citep{hu2019fdml,Castiglia2022,castiglia2023flexible} and has applications, for example, in credit score prediction. We also address the case of private labels, proposing a modified version of our method for that setting in Section~\ref{sec:efvfl_private_labels}.

We assume that $f$ has an optimal value $f^\star\coloneqq\min_{\bm{x}}f(\bm{x})>-\infty$ and make the following assumptions, where $\nabla$ denotes not only the gradient of scalar-valued functions but, more generally, the derivative of a (possibly multidimensional) map.

\begin{assumption}[Smoothness] \label{ass:smoothness}
	A function~$h\colon \mathbb{R}^d\mapsto \mathbb{R}$ is $L$-smooth if there exists a positive constant $L$ such that
	\begin{equation} \label{eq:lsmooth}
		\forall
		\bm{x},\bm{y}\in\mathbb{R}^d
		\colon
		\quad
		\lVert \nabla h(\bm{x}) - \nabla h(\bm{y}) \rVert
		\leq
		L
		\lVert \bm{x} - \bm{y}\rVert.
		\tag{A\ref*{ass:smoothness}}
	\end{equation}
	We assume $f$ is $L_f$-smooth and $\Phi$ is $L_\Phi$-smooth and let $L=\max\{L_f,L_\Phi\}$.
\end{assumption}

\begin{assumption}[Bounded derivative] \label{ass:bounded_embedding}
	Map $\bm{F}\colon \mathbb{R}^{p_1}\mapsto \mathbb{R}^{p_2\times p_3}$ has a bounded derivative if there exists a positive constant $H$ such that
	\begin{equation} \label{eq:bounded_embedding}
		\forall
		\bm{x}
		\in
		\mathbb{R}^{p_1}
		\colon
		\quad
		\lVert \nabla \bm{F}(\bm{x}) \rVert\leq H,
		\tag{A\ref*{ass:bounded_embedding}}
	\end{equation}
	where $\lVert \nabla \bm{F}(\bm{x}) \rVert$ is the Euclidean norm of the third-order tensor $\nabla \bm{F}(\bm{x})$. We assume \eqref{eq:bounded_embedding} holds for $\{\bm{H}_k\}_{k=0}^K$.
\end{assumption}

Note that, in Assumption~\ref{ass:bounded_embedding}, we do \textit{not} assume that our objective function~$f$ has a bounded gradient. We only require the local representation-extracting maps $\{\bm{H}_k\}$ to have a bounded derivative. The same assumption is also made in~\citet{Castiglia2022}.

\section{Proposed method} \label{sec:training}
\label{sec:efvfl}

To solve Problem~\eqref{eq:splitnn} with a gradient-based method, we need to perform a forward and a backward pass at each step $t$ to compute the gradient of our objective function. In VFL, a standard uncompressed algorithm employing a single local update per communication round is mathematically equivalent to gradient descent.\footnote{When performing multiple local updates, we lose this mathematical equivalence. In that case, we instead get a parallel (block) coordinate descent method where the simultaneous updates use stale information about the other blocks of variables.} This algorithm, which our method recovers if we set $\mathcal{C}$ to be the identity map, is as follows:
\begin{itemize}
	\item In the forward pass, each client~$k$ computes $\bm{H}_k(\bm{x}_k^t)$ and sends it to the server, which then computes $f(\bm{x}^t) = \Phi\left(\bm{H}(\bm{x}^t)\right)$.
	
	\item In the backward pass, first, the server backpropagates through $\Phi$, obtaining $\nabla_0\Phi ( \bm{x}_0^t,\{\bm{H}_k(\bm{x}_k^t)\}^K_{k=1} )=\nabla_0 f( \bm{x}^t)$ and $\nabla_k \Phi(\{\bm{H}_k(\bm{x}_k^t)\})$, for all $k$, where $\nabla_k$ denotes the derivative with respect to block~$k$. The former is used to update $\bm{x}_0^t$, while the latter is sent to each client $k$, which uses this derivative to continue backpropagation over its local model, allowing it to compute $\nabla_kf(\bm{x}^t)$.
\end{itemize}
We repeat these steps until convergence. Let us now cover the general case, where $\mathcal{C}$ may not be the identity map.

\vspace{1mm}
\paragraph{Forward pass.}
An exact forward pass would require each client~$k$ to send $\bm{H}_k(\bm{x}_k^t)$ to the server, bringing a significant communication overhead. To address this, in our method, the server model does not have as input $\bm{H}_k(\bm{x}_k^t)$, but rather a surrogate for it, $\bm{G}_k^t$, which is initialized as $\bm{G}_k^0=\mathcal{C}( \bm{H}_{k}(\bm{x}_k^{0}))$ and is updated as follows, as in~\eqref{eq:error_feedback}:
\[
\forall k\in[K]:
\quad
\bm{G}_k^t
\coloneqq
\bm{G}_k^{t-1}
+
\bm{C}_{k}^{t},
\quad
\bm{C}_{k}^{t}\coloneqq\mathcal{C}( \bm{H}_{k}(\bm{x}_k^{t}) -\bm{G}_{k}^{t-1}).
\]
This requires keeping $\bm{G}_k^t$, of size $N\times E_k$, in memory at client~$k$ and at the server. (Note that $\bm{G}_k^t$ is often smaller than the local dataset~$\mathcal{D}_k$.) The server thus computes the function
$\Phi \left( \bm{x}_0^t,\bm{G}_1^t,\dots,\bm{G}_K^t \right)$, which acts as a surrogate for the original objective $f(\bm{x}^t)=\Phi ( \bm{x}_0^t,\bm{H}_1(\bm{x}_1^t),\dots,\bm{H}_K(\bm{x}_K^t))$, as illustrated in Figure~\ref{fig:efvfl_model}.

\vspace{1mm}
\paragraph{Backward pass.}
The server performs a backward pass over the server model, obtaining $\nabla_0\Phi(\bm{x}_0^t,\{\bm{G}_k^t\}^K_{k=1})$, which it uses as a surrogate for $\nabla_0\Phi ( \bm{x}_0^t,\{\bm{H}_k(\bm{x}_k^t)\}^K_{k=1} )=\nabla_0 f( \bm{x}^t)$ to update $\bm{x}_0^t$. Similarly, to update each local model $\bm{x}_k^t$, for $k\in[K]$, we want client~$k$ to have a surrogate for
\[
\nabla_kf(\bm{x}^t)
=
\sum_{i,j=1}^{N,E_k}
\left[\nabla_k \Phi(\{\bm{H}_k(\bm{x}_k^t)\}^K_{k=0})\right]_{ij}
\left[\nabla \bm{H}_k (\bm{x}_k^t)\right]_{ij:}
.
\]
However, while $\nabla \bm{H}_k (\bm{x}_k^t)$ can be computed at each client $k$, the $\nabla_k \Phi$ term cannot, as client $k$ does not have access to $\bm{H}_\ell(\bm{x}_\ell^t)$, for $\ell\not=k$. So, we instead use the following surrogate for $\nabla_kf(\bm{x}^t)$:
\begin{equation} \label{eq:grad_surrogate}
	\bm{g}^t_k
	\coloneqq
	\sum_{i,j=1}^{N,E_k}
	\left[\tilde{\nabla}_k^t \Phi\right]_{ij}
	\left[\nabla \bm{H}_k (\bm{x}_k^t)\right]_{ij:}
	,
\end{equation}
where $\tilde{\nabla}_k^t \Phi \coloneqq
\nabla_k \Phi( ..., \bm{G}_{k-1}^t, \bm{H}_{k}(\bm{x}_{k}^t), \bm{G}_{k+1}^t, ... )$. Since the server does not hold $\bm{H}_k(\bm{x}_k^t)$, it cannot compute $\tilde{\nabla}_k^t \Phi$. Thus, the server broadcasts $\{\bm{C}_k^t\}_{k=0}^K$, so that each client~$k$, which does hold $\bm{H}_k(\bm{x}_k^t)$, can compute $\{\bm{G}_\ell^t\}_{\ell\not=k}$ and use it to perform a forward and a backward pass over the server model locally, obtaining $\tilde{\nabla}_k^t \Phi$. Thus, while the forward pass only requires the error feedback module at each client~$k$ to hold the estimate~$\bm{G}^t_k$, when we account for the backward pass too, each machine~$k\in\{0,1,\dots,K\}$ must hold $\{\bm{G}_\ell^t\}_{\ell=0}^K$. We write our update as:
\[
\bm{x}^{t+1}
=
\bm{x}^{t}-\eta\bm{g}^t
\quad
\text{where}
\quad
\bm{g}^t
\coloneqq
(\bm{g}^t_0, \dots, \bm{g}^t_K)
\]
and $\eta$ is the stepsize.

\IncMargin{1.5em}
\begin{algorithm}[t]
	\DontPrintSemicolon
	\SetKwInOut{KwIn}{Input}
	\KwIn{initial point~$\bm{x}^{0}$, stepsize~$\eta$, and initial surrogates $\{\bm{G}_k^0=\mathcal{C}(\bm{H}_k(\bm{x}_k^0)) \}$.}
	\For{$t=0,\dots,T-1$}{
		Update $\bm{x}_k^{t+1} = \bm{x}_k^{t}-\eta\tilde{\bm{g}}^t_k$ in parallel, for $k\in\{0,1,\dots,K\}$, based on a shared sample $\mathcal{B}^{t}\subseteq[N]$.\;
		Compute and send $\bm{C}_{k}^{t+1}=\mathcal{C}( \bm{H}_{k}(\bm{x}_k^{t+1}) -\bm{G}_{k}^{t})$ to the server in parallel, for $k\in[K]$.\;
		Server broadcasts $\{\bm{C}_{k}^{t+1}\}_{k=0}^K$ to all clients.\;
		Update $\bm{G}_{k}^{t+1}=\bm{G}_{k}^{t}+\bm{C}_{k}^{t+1}$ in parallel, for $k\in\{0,1,\dots,K\}$.
	}
	\caption{\texttt{EF-VFL}}
	\label{alg:efvfl}
\end{algorithm}
\DecMargin{1.5em}

\vspace{1mm}
\paragraph{Mini-batch.}
For the sake of computation efficiency, we further allow for the use of mini-batch approximations of the objective. Without compression, or with direct compression, the use of mini-batches allows client~$k$ to send only the entries of $\bm{H}_{k}(\bm{x}_k)$ corresponding to mini-batch $\mathcal{B}\subseteq[N]$, of size $B$, denoted by $\bm{H}_{k\mathcal{B}}(\bm{x}_k)\in\mathbb{R}^{B\times E_k}$, instead of $\bm{H}_{k}(\bm{x}_k)\in\mathbb{R}^{N\times E_k}$. Yet, our method provides all machines with an estimate for all the entries of $\{\bm{H}_k(\bm{x}_k)\}$ at all times, the error feedback states $\{\bm{G}_k\}$. Therefore, our communications, needed to update $\bm{G}_k$, may not depend on $N$ and $B$ at all. This is determined by our choice of $\mathcal{C}$.

Our mini-batch surrogates depend on the entries of $\bm{H}_{k}(\bm{x}_k)$ and $\bm{G}_k$ corresponding to $\mathcal{B}$, $\bm{H}_{k\mathcal{B}}(\bm{x}_k)$ and $\bm{G}_{k\mathcal{B}}$.
(Note that $\bm{H}_{0\mathcal{B}}(\bm{x}_0)=\bm{H}_{0}(\bm{x}_0)$ and $\bm{G}_{0\mathcal{B}}^t=\bm{G}_{0}^t$.) Thus, we approximate the partial derivative of the mini-batch function $f_{\mathcal{B}}(\bm{x})=\Phi_{\mathcal{B}}( \{\bm{H}_{k\mathcal{B}}(\bm{x}_k)\} ) = \frac{1}{B} \sum_{n\in\mathcal{B}} \phi_n \left(\bm{h}_n(\bm{x})\right)$ with respect to $\bm{x}_k$ by:
\[
\tilde{\bm{g}}^t_k
\coloneqq
\sum_{i\in\mathcal{B}^t}
\sum_{j=1}^{E_k}
\left[\tilde{\nabla}_k^t \Phi_{\mathcal{B}^t}\right]_{ij}
\left[\nabla \bm{H}_{k\mathcal{B}^t} (\bm{x}_k^t)\right]_{ij:}
,
\]
where
$\tilde{\nabla}_k^t \Phi_{\mathcal{B}^t}
\coloneqq
\nabla_k \Phi_{\mathcal{B}^t}(
...,
\bm{G}_{k-1,\mathcal{B}^t}^t,
\bm{H}_{k\mathcal{B}^t}(\bm{x}_{k}^t),
\bm{G}_{k+1,\mathcal{B}^t}^t,
...
)$.
We write the mini-batch version of the update as:
\[
\bm{x}^{t+1}=\bm{x}^{t}-\eta\tilde{\bm{g}}^t
\quad
\text{where}
\quad
\tilde{\bm{g}}^t \coloneqq \left( \tilde{\bm{g}}^t_0 , \dots , \tilde{\bm{g}}^t_K \right)
.
\]

We now describe our method, which we summarize in Algorithm~\ref{alg:efvfl}.
\begin{itemize} 
	\item \textbf{Initialization:}
	We initialize our model parameters as $\bm{x}^0$. Each machine~$k\in\{0,1,\dots,K\}$ must hold $\bm{x}_k^0$ and our compression estimates $\{\bm{G}_k^0=\mathcal{C}(\bm{H}_k(\bm{x}_k^0))\}_{k=0}^K$.
	
	\item \textbf{Model parameters update:}
	In parallel, all machines~$k\in\{0,1,\dots,K\}$ take a (stochastic) coordinate descent step with respect to their local surrogate objective, updating $\bm{x}^{t+1}_k = \bm{x}^{t}_k-\eta\tilde{\bm{g}}^t_k$ based on a shared batch~$\mathcal{B}^{t}$, sampled locally at each client following a shared seed.
	
	\item \textbf{Compressed communications:}
	All clients $k\in[K]$ compute $\bm{C}_{k}^{t+1}$ and send it to the server, who broadcasts $\{\bm{C}_{k}^{t+1}\}$.
	
	\item \textbf{Compression estimates update:}
	Lastly, all machines~$k\in\{0,1,\dots,K\}$ use the compressed error feedback $\{\bm{C}_{k}^{t+1}\}$ to update their (matching) compression estimates $\{\bm{G}_{k}^{t}\}$.
\end{itemize}

Note that, in the absence of compression (that is, if $\mathcal{C}$ is the identity operator), each client~$k$ must send $\bm{C}_{k}^{t+1}$, of size $N\times E_k$, at each iteration. Further, the server must broadcast $\{\bm{C}_{k}^{t+1}\}_{k=0}^K$. Thus, the total communication complexity of \texttt{EF-VFL} in the absence of compression is $\mathcal{O}(N \cdot E \cdot T)$. When compression is used, $N \cdot E$ is replaced by some smaller amount which depends on the compression mechanism. For example, for top-$k$ sparsification (defined in Section~\ref{sec:experiments}), the communication complexity is reduced to $\mathcal{O}(k \cdot T)$.

In Algorithm~\ref{alg:efvfl}, we formulate our method in a general setting which allows for the compression of both $\{\bm{H}_{k}(\bm{x}_k)\}_{k=1}^K$ and $\bm{x}_0$. This setting is covered by our analysis in Section~\ref{sec:convergence_guarantees}. However, in general, the bottleneck lies in the uplink (client-to-server) communications, rather than the server broadcasting~\citep{haddadpour2021federated}. Therefore, our experiments in Section~\ref{sec:experiments} focus on the compression of $\{\bm{H}_{k}(\bm{x}_k)\}_{k=1}^K$, rather than $\bm{x}_0$.

\subsection{Adapting our method for handling private labels}
\label{sec:efvfl_private_labels}

In this section, we propose an adaptation of \texttt{EF-VFL} to allow for private labels. That is, we remove the assumption that all clients hold $\phi_n$, which contains the label of $\bm{\xi}_n$. Instead, only the server holds the labels. Further, in this adaptation, the parameters of the server model, $\bm{x}_0$, are not shared with the clients either.

Note that, without holding $\phi_n$, the clients cannot perform the entire forward pass locally. Instead, in this setting, the forward and backward pass over $\phi_n$ take place at the server, while the forward and backward pass over $\bm{H}_{k}(\bm{x}_k)$ takes place at client~$k$. More precisely, in the forward pass, each client~$k$ sends $\bm{C}_k$ to the server, who holds $\bm{x}_0$ and the labels, and can thus compute the loss. Then, for the backward pass, the server backpropagates over its model and sends only the derivative of the loss function with respect to $\bm{G}_k$ to each client~$k$. This requires replacing our surrogate of $\nabla_kf(\bm{x}^t)$ in~\eqref{eq:grad_surrogate}, which uses the exact local representation $\bm{H}_{k}(\bm{x}_k)$ at each client, by one based on our error feedback surrogates:
\begin{equation} \label{eq:private_label_surrogate}
	\bm{\nabla}^t_k
	\coloneqq
	\sum_{i,j=1}^{N,E_k}
	\left[\nabla_k \Phi\left(\bm{H}_0(\bm{x}_0^t),
	\{\bm{G}_j^t\}^K_{j=1}\right)\right]_{ij}
	\left[\nabla \bm{H}_k (\bm{x}_k^t)\right]_{ij:}
	.
\end{equation}
Note that we do not backpropagate through the error-feedback update.

More generally, we use the following (possibly) mini-batch update vector:
\[
\underbrace{\sum_{i\in\mathcal{B}^t}
	\sum_{j=1}^{E_k}
	\left[\nabla_k \Phi_{\mathcal{B}^t}\left(\bm{H}_{0\mathcal{B}^t}(\bm{x}_0^t),
	\{\bm{G}_{j\mathcal{B}^t}^t\}^K_{j=1}\right)\right]_{ij}
	\left[\nabla \bm{H}_{k\mathcal{B}^t} (\bm{x}_k^t)\right]_{ij:}}_{\eqqcolon\tilde{\bm{\nabla}}^t_k}
.
\]
We summarize the adaption of the \texttt{EF-VFL} method to the private labels setting in Algorithm~\ref{alg:efvfl_private_labels}.

Allowing for private labels broadens the range of applications for our method, since many VFL applications require not only private features, but also private labels, rendering any method requiring public labels inapplicable. However, as we will see in Section~\ref{sec:private_labels_exps}, using surrogate~\eqref{eq:private_label_surrogate} instead of \eqref{eq:grad_surrogate} can slow down convergence. Further, unlike Algorithm~\ref{alg:efvfl}, which can be easily extendable to allow for multiple local updates at the clients between rounds of communication, Algorithm~\ref{alg:efvfl_private_labels} works only for a single local update. This is because, for an VFL algorithm to perform multiple local updates, $\phi_n$ and $\bm{x}_0$ must be available at the clients, so that the forward and backward passes over the server model and the loss function can be performed locally after each update.

\IncMargin{1.5em}
\begin{algorithm}[t]
	\DontPrintSemicolon
	\SetKwInOut{KwIn}{Input}
	\KwIn{initial point~$\bm{x}^{0}$, stepsize~$\eta$, and initial surrogates $\{\bm{G}_k^0=\mathcal{C}(\bm{H}_k(\bm{x}_k^0)) \}$.}
	\For{$t=0,\dots,T-1$}{
		Update $\bm{x}_k^{t+1} = \bm{x}_k^{t}-\eta\tilde{\bm{\nabla}}^t_k$ in parallel, for $k\in\{0,1,\dots,K\}$, based on a shared sample $\mathcal{B}^{t}\subseteq[N]$.\;
		Compute and send $\bm{C}_{k}^{t+1}=\mathcal{C}( \bm{H}_{k}(\bm{x}_k^{t+1}) -\bm{G}_{k}^{t})$ to the server in parallel, for $k\in[K]$.\;
		Server sends $\nabla_k \Phi_{\mathcal{B}^t}\left(\bm{H}_{0\mathcal{B}^t}(\bm{x}_0^t),
		\{\bm{G}_{j\mathcal{B}^t}^t\}^K_{j=1}\right)$ to client~$k$, in parallel, for $k\in\{0,1,\dots,K\}$.\;
		Update $\bm{G}_{k}^{t+1}=\bm{G}_{k}^{t}+\bm{C}_{k}^{t+1}$ in parallel, for $k\in\{0,1,\dots,K\}$.
	}
	\caption{\texttt{EF-VFL} with private labels}
	\label{alg:efvfl_private_labels}
\end{algorithm}
\DecMargin{1.5em}
	
\section{Convergence guarantees}
\label{sec:convergence_guarantees}

In this section, we provide convergence guarantees for \texttt{EF-VFL}. We present our results for Algorithm~\ref{alg:efvfl} only, rather then stating them again for Algorithm~\ref{alg:efvfl_private_labels}, since they exhibit only a minor difference in Lemma~\ref{lemma:surrogate_offset_bound} and in the main theorem, where the constant $K$ is replaced with $K+1$. We defer the details to Appendix~\ref{sec:preliminaries}.

First, let us define the following sigma-algebra
\[
\mathcal{F}_t\coloneqq \sigma(\bm{G}^0,\bm{x}^1,\bm{G}^1,
\dots, \bm{x}^{t},\bm{G}^{t}),
\]
where $\bm{G}^t\coloneqq\{\bm{G}^t_0,\dots,\bm{G}^t_K\}$. For the sake of conciseness, we further let $\mathbb{E}_{\mathcal{F}}$ denote the conditional expectation $\mathbb{E}[\cdot\mid \mathcal{F}]$ with sigma-algebra $\mathcal{F}$. We use the following assumptions on our stochastic update vector~$\tilde{\bm{g}}^t$ and the use of mini-batches.

\begin{assumption}[Unbiased] \label{ass:unbiased_grad}
	We assume that our stochastic update vector is unbiased:
	\begin{equation} \label{eq:unbiased_grad}
		\forall (\bm{x},t) \in\mathbb{R}^d\times\{0,1,\dots,T-1\}
		\colon
		\quad
		\mathbb{E}_{\mathcal{F}_t} \left[\tilde{\bm{g}}^t\right]
		=
		\bm{g}^t.
		\tag{A\ref*{ass:unbiased_grad}}
	\end{equation}
\end{assumption}

\begin{assumption}[Bounded variance] \label{ass:bounded_var}
	We assume that there exists a constant $\sigma\geq0$ such that
	\begin{equation} \label{eq:bounded_var}
		\forall (\bm{x},t) \in\mathbb{R}^d\times\{0,1,\dots,T-1\}
		\colon
		\quad
		\mathbb{E}_{\mathcal{F}_t}
		\lVert
		\tilde{\bm{g}}^t - \bm{g}^t
		\rVert^2
		\leq
		\frac{\sigma^2}{B}
		.
		\tag{A\ref*{ass:bounded_var}}
	\end{equation}
\end{assumption}

We now present Lemma~\ref{lemma:surrogate_offset_bound} and Lemma~\ref{lemma:distortion_recursive_bound_w_embedding_diff}, which we will use to prove our main theorems. We let $D^{(t)}\coloneqq \sum_{k=0}^K \left\lVert\bm{G}_k^{t}-\bm{H}_k(\bm{x}_k^t)\right\rVert^2$ denote the total distortion (caused by compression) at time~$t$.

\begin{lemma}[Surrogate offset bound] \label{lemma:surrogate_offset_bound}
	If $\Phi$ is $L$-smooth~\eqref{eq:lsmooth} and $\{\bm{H}_k\}$ have bounded derivatives~\eqref{eq:bounded_embedding}, then, for all $t\geq0$,
	\begin{equation} \label{eq:surrogate_offset_bound}
		\lVert \bm{g}^{t} - \nabla f(\bm{x}^{t}) \rVert^2
		\leq
		KH^2L^2D^{(t)}.
	\end{equation}
\end{lemma}

\begin{proof}
	See Appendix~\ref{app:lemma1_proof}.
\end{proof}

\begin{lemma}[Recursive distortion bound] \label{lemma:distortion_recursive_bound_w_embedding_diff}
	Let $\{\bm{x}^t\}$ be a sequence generated by Algorithm~\ref{alg:efvfl}. If $\mathcal{C}$ is a  contractive compressor~\eqref{eq:biased_compressor}, $\{\bm{H}_k\}$ have bounded derivatives~\eqref{eq:bounded_embedding},
	and \eqref{eq:unbiased_grad} and \eqref{eq:bounded_var} hold, then, for all $t\geq0$ and $\epsilon>0$:
	\begin{equation}\label{eq:distortion_recursive_bound}
		\begin{split}
			\mathbb{E} D^{(t+1)}
			&\leq
			(1-\alpha)(1+\epsilon) \mathbb{E} D^{(t)}
			\\&\quad+
			(1-\alpha)(1+\epsilon^{-1}) \eta^2 H^2 \left(\mathbb{E} \lVert \bm{g}^{t} \rVert^2+ \frac{\sigma^2}{B}\right).
		\end{split}
	\end{equation}
\end{lemma}

\begin{proof}
	See Appendix~\ref{app:lemma2_proof}.
\end{proof}

\subsection{Nonconvex setting}

We now present our main convergence result for \texttt{EF-VFL}.

\begin{theorem} \label{thm:efvfl_thm}
	Let $\{\bm{x}^t\}$ be a sequence generated by Algorithm~\ref{alg:efvfl}, $\mathcal{C}$ be a contractive compressor~\eqref{eq:biased_compressor}, and $f^\star>-\infty$. If \eqref{eq:lsmooth} to \eqref{eq:bounded_var} hold, then, for $0<\eta\leq 1/ (\sqrt{\rho_{\alpha1}} L+L)$:
	\begin{equation} \label{eq:nonconvex_thm}
		\frac{1}{T} \sum_{t=0}^{T-1} \mathbb{E} \lVert \nabla f(\bm{x}^{t})\rVert^2
		\leq
		\frac{2\Delta}{\eta T}
		+
		\left( 1+\eta L \rho_{\alpha1} \right) \frac{\eta L\sigma^2}{B}
		+
		\rho_{\alpha2} \frac{\mathbb{E} D^{(0)}}{T}
		,
	\end{equation}
	where the expectation is over the randomness in $\mathcal{C}$ and in $\{\mathcal{B}_t\}$, $\Delta\coloneqq f \left( \bm{\theta}^{0} \right)-f^\star$, and
	\[
	\rho_{\alpha1} \coloneqq KH^4 \left(\frac{1+\sqrt{1-\alpha}}{\alpha}-1\right)^2
	\;
	\text{and}
	\;\;
	\rho_{\alpha2} \coloneqq \frac{KH^2L^2}{1-\sqrt{1-\alpha}}
	.
	\]
\end{theorem}

\begin{proof}
	See Appendix~\ref{sec:efvfl_thm_proof}.
\end{proof}

If the batch size is large enough, $B=\Omega(\sigma^2/\delta)$, the iteration complexity to reach $\frac{1}{T} \sum_{t=0}^{T-1}
\mathbb{E}\left\lVert \nabla f \left( \bm{\theta}^{t}\right) \right\rVert^2\leq\delta$ matches the  $\mathcal{O}(1/T)$ rate of the centralized, uncompressed setting. Also, in the absence of compression ($\alpha=1$ and $D^{(t)}=0$), we recover that, for $\eta\in(0,1/L]$, we can output an $\bm{x}^{\text{out}}$ such that $\mathbb{E} \lVert \nabla f(\bm{x}^{\text{out}})\rVert^2\leq \frac{2\Delta}{\eta T} + \frac{\eta L\sigma^2}{B}$. If we are also in the full-batch case ($\sigma=0$), we recover the gradient descent bound exactly: $\mathbb{E} \lVert \nabla f(\bm{x}^{\text{out}})\rVert^2\leq \frac{2\Delta}{\eta T}$.

{Note that, if we start our method by sending noncompressed representations (at $t=0$ only), we can drop the last term in the upper bound in~\eqref{eq:nonconvex_thm}.}

Our results improve over the prior state-of-the-art compressed vertical FL (\texttt{CVFL})~\citep{Castiglia2022}, whose convergence result for a fixed stepsize is presented below:
\[
\frac{1}{T} \sum_{t=0}^{T-1} \mathbb{E} \lVert \nabla f(\bm{x}^{t})\rVert^2
\leq
\frac{4\Delta}{\eta T}
+
\mathcal{O}\left(\frac{\eta L\sigma^2}{B}\right)
+
\mathcal{O}\left(\frac{1}{T} \sum_{t=0}^{T-1}D_{d}^{(t)}\right)
.
\]
Note how, even for full-batch updates, the upper bound above does not go to zero as $T\to\infty$ unless $D_{d}^{(t)}\to0$, where $D_{d}^{(t)}$ is the distortion resulting from direct compression $D_{d}^{(t)}=\sum_{k=0}^K \left\lVert\mathcal{C}(\bm{H}_k(\bm{x}_k^t))-\bm{H}_k(\bm{x}_k^t)\right\rVert^2$. That is, to achieve $\frac{1}{T} \sum_{t=0}^{T-1} \mathbb{E} \lVert \nabla f(\bm{x}^{t})\rVert^2\to0$, \texttt{CVFL} requires a vanishing compression error, which necessitates that $\alpha\to1$. This means that, despite reducing the total amount of communications across the training, \texttt{CVFL} does not reduce the maximum amount of communications per round. In contrast, by allowing for nonvanishing compression, \texttt{EF-VFL} ensures small communication cost at every round.

\subsection{Under the PL inequality}
In this section, we establish the linear convergence of \texttt{EF-VFL} under the PL inequality~\citep{polyak1963gradient}.
\begin{assumption}[PL inequality] \label{ass:pl}
	We assume that there exists a positive constant $\mu$ such that
	\begin{equation} \label{eq:pl}
		\forall
		\bm{x}\in\mathbb{R}^d
		\colon
		\quad
		\lVert \nabla f(\bm{x}) \rVert^2
		\geq
		2\mu(f(\bm{x})-f^\star).
		\tag{A\ref*{ass:pl}}
	\end{equation}
\end{assumption}
We resort to the following Lyapunov function to show linear convergence:
\begin{equation}
	\label{eq:lyapunov}
	V_{t}\coloneqq
	\mathbb{E}f(\bm{x}^t)-f^\star+c\mathbb{E} D^{(t)},
\end{equation}
where $c$ is a positive constant. We now present Theorem~\ref{thm:pl}.
\begin{theorem} \label{thm:pl}
	Let $\{\bm{x}^t\}$ be a sequence generated by Algorithm~\ref{alg:efvfl}, $\mathcal{C}$ be a contractive compressor~\eqref{eq:biased_compressor}, and $f^\star>-\infty$. If \eqref{eq:lsmooth} to \eqref{eq:pl} hold, then, for $\eta$ such that $\eta^2 L^2 \left(1-\mu/L\right) + \eta \mu \leq \alpha^2$, we have:
	\[
	V_{T}\leq(1-\eta\mu)^TV_{0}
	+
	{\frac{\sigma^2}{2B\mu}}
	.
	\]
\end{theorem}

\begin{proof}
	See Appendix~\ref{sec:efvfl_thm_proof_pl}.
\end{proof}

Since $\mu\leq L$, we have that $\eta\in(0,1/L)$ implies that $1-\eta\mu\in(0,1)$. Hence, \texttt{EF-VFL} converges linearly to a $\mathcal{O}(\sigma^2)$ neighborhood around the global optimum.

\begin{table*}[t]
	\caption{
		{
			Total communication cost to reach error level $\delta>0$ (top-$k$; full-batch; single local update).
		}
		\label{tab:total_comms}
	}
	\centering
	\renewcommand{\arraystretch}{1.2}
	\resizebox{\textwidth}{!}{
		\begin{tabular}{c|cccc}
			\specialrule{1.2pt}{0pt}{0pt}
			{$ \mathcal{C} $} & {\texttt{SVFL}~[19]} & {\texttt{CVFL}~[12]} & {\texttt{EF-VFL} (ours)} & {\texttt{EF-VFL} with private labels (ours)}\\
			\hline
			{Uplink (total)} & $N\bar{E}K \cdot \mathcal{O}(\sfrac{1}{\delta})$ & $ \min\{ k\cdot\mathcal{O}(\sfrac{1}{\sqrt{\delta}}),N\bar{E} \}\cdot K \cdot \mathcal{O}(\sfrac{1}{\delta^2}) $ & $ kK \cdot \mathcal{O}(\sfrac{1}{\delta})$ & $ kK \cdot \mathcal{O}(\sfrac{1}{\delta})$ \\
			{Downlink (total/broadcast)} & $N\bar{E}K \cdot \mathcal{O}(\sfrac{1}{\delta})$ & $\min\{ k\cdot\mathcal{O}(\sfrac{1}{\sqrt{\delta}}),N\bar{E} \}\cdot(K+1) \cdot \mathcal{O}(\sfrac{1}{\delta^2})$ & $k(K+1) \cdot \mathcal{O}(\sfrac{1}{\delta})$ & $kK \cdot\mathcal{O}(\sfrac{1}{\delta})$ \\
			\specialrule{1.2pt}{0pt}{0pt}
		\end{tabular}
	}
\end{table*}

{In Table~\ref{tab:total_comms}, where $\bar{E}=E_j$ for $j\in[K]$ is the embedding size for each sample at each client (assumed to match for simplicity), we present the total communication cost to reach $\frac{1}{T} \sum_{t=0}^{T-1} \mathbb{E}\left\lVert \nabla f \left( \bm{\theta}^{t}\right) \right\rVert^2\leq\delta$, where $\delta>0$ (top-$k$; full-batch; single local update). We discuss different downlink communication schemes for \texttt{EF-VFL} in Appendix~\ref{app:downlink_schemes}.}
	
	
\section{Experiments}
\label{sec:experiments}

We compare \texttt{EF-VFL} with two baselines: \textbf{(1)} standard VFL (\texttt{SVFL}), which corresponds to the method in~\citet{Liu2022} and is mathematically equivalent to stochastic gradient descent when a single local update is used, and \textbf{(2)} \texttt{CVFL}~\citep{Castiglia2022}, which recovers \texttt{SVFL} when an identity compressor is used (that is, without employing compression). All of our results correspond to the mean and standard deviation for five different seeds. We employ two popular compressors in our experiments.
\begin{itemize} [leftmargin=1.0em]
	\item Top-$k$ sparsification~\citep{alistarh2018convergence,stich2018sparsified} is a map $\topk\colon\mathbb{R}^d\mapsto\mathbb{R}^d$ defined as
	\[
	\topk(\bm{v})
	\coloneqq
	\bm{v} \odot \bm{u}(\bm{v}),
	\]
	where $\odot$ denotes the Hadamard product and $\bm{u}(\bm{v})$ is such that its entry $i$ is 1 if $v_i$ is one of the $k$ largest entries of $\bm{v}$ in absolute value and 0 otherwise. We have that \eqref{eq:biased_compressor} holds for $\alpha=k/d$.
	
	\item Stochastic quantization~\citep{alistarh2017qsgd} is a map $\text{qsgd}_s\colon\mathbb{R}^d\mapsto\mathbb{R}^d$, with $s>1$ quantization levels defined as
	\[
	\text{qsgd}_s(\bm{v})
	\coloneqq
	\frac{\lVert \bm{v} \rVert \cdot \text{sign}(\bm{v}) }{s\tau}
	\cdot \left\lfloor s \frac{\lvert \bm{v} \rvert}{\lVert \bm{v} \rVert} + \xi \right\rfloor,
	\]
	where $\tau=1+\min\{d/s^2,\sqrt{d}/s\}$ and $\xi\sim\mathcal{U}([0,1]^d)$, where $\mathcal{U}$ denotes the uniform distribution. In practice, we are interested in values of $s$ such that $s=2^b$, where $b$ is the number of bits. We have that \eqref{eq:biased_compressor} holds for $\alpha=1/\tau$.
	
\end{itemize}
For the sake of the comparison with \texttt{CVFL}, we employ compressors $\mathcal{C}=\mathcal{C}_2\circ\mathcal{C}_1$, where $\mathcal{C}_2$ is either $\topk(\bm{v})$ or $\text{qsgd}_s$ and $\mathcal{C}_1$ selects the rows in $\mathcal{B}^t$, similarly to \texttt{CVFL}. Further, as explained in Section~\ref{sec:efvfl}, our experiments focus on the compression of $\{\bm{H}_{k}(\bm{x}_k)\}_{k=1}^K$, and not $\bm{x}_0$.

\subsection{Comparison with \texttt{SVFL} and \texttt{CVFL}}
\label{sec:single_local_update}

The detailed hyperparameters for the following experiments can be found in the provided code.

\begin{figure*}[t!]
	\centering
	\subfloat[top-$k$ keeping $10\%$\vspace{-0mm} ]{\includegraphics[width=0.24\linewidth]{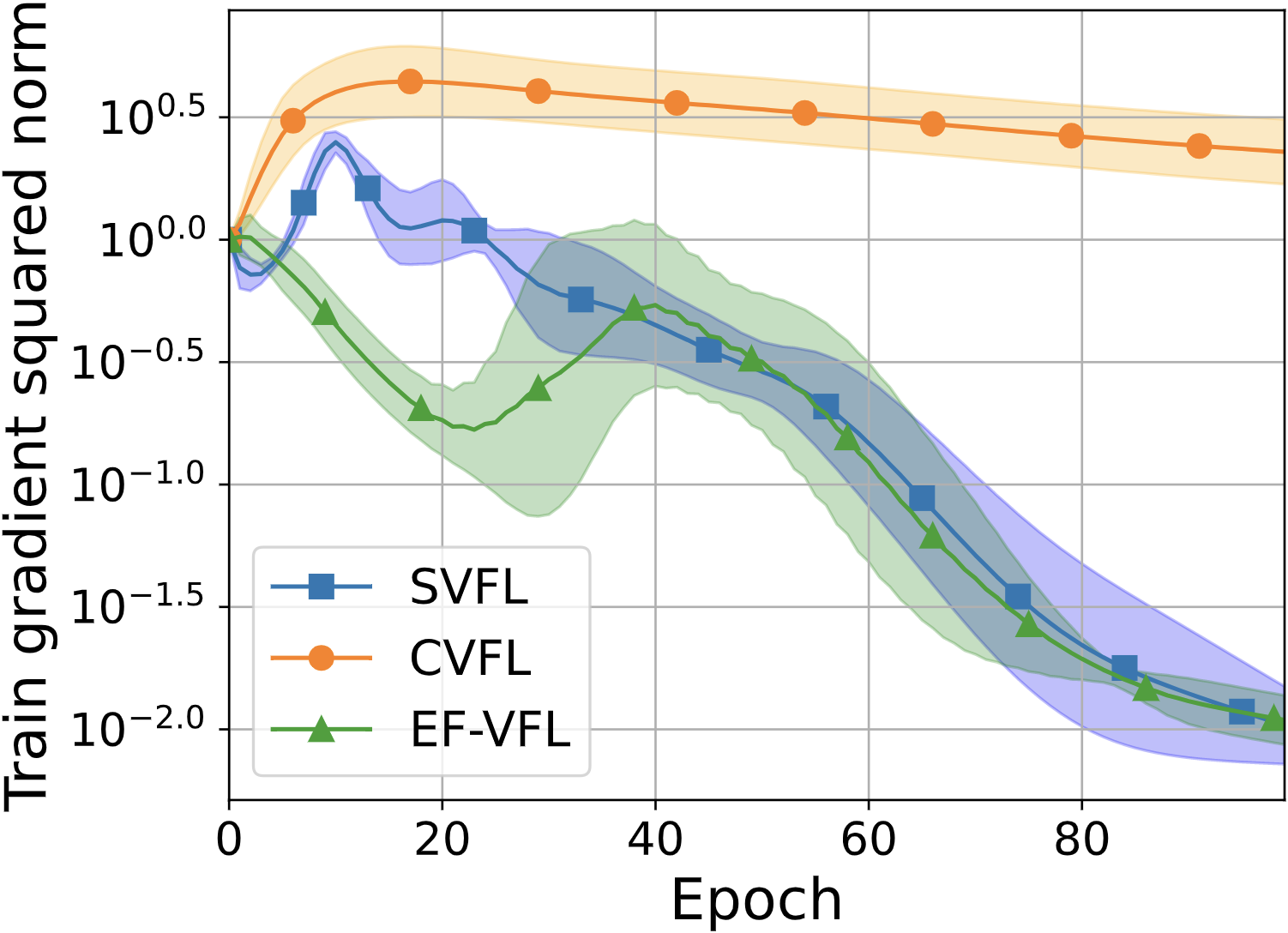}\hspace{1mm}\includegraphics[width=0.24\linewidth]{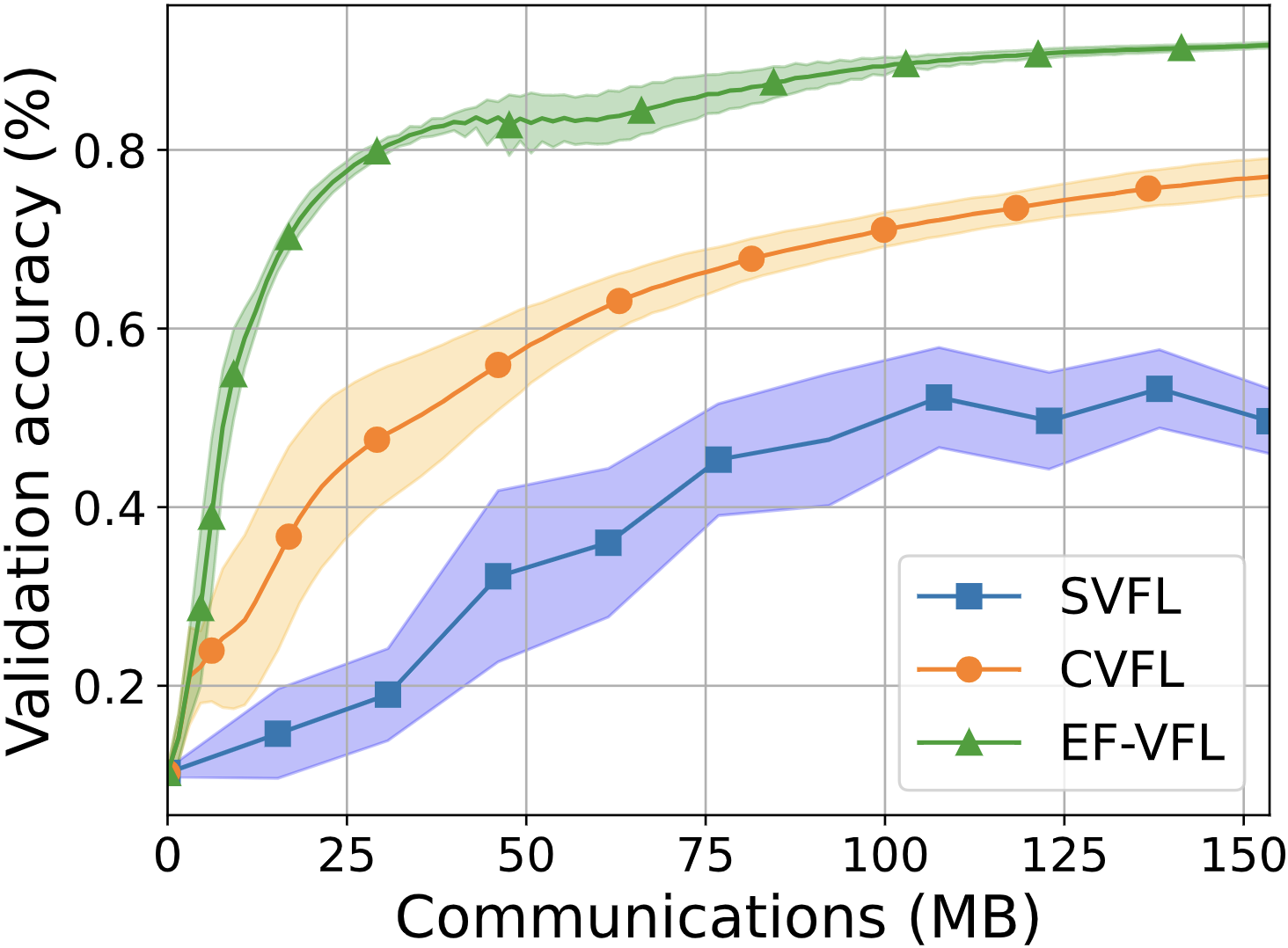}}
	\hfil
	\subfloat[quantization with $b=4$\vspace{-0mm} ]{\includegraphics[width=0.24\linewidth]{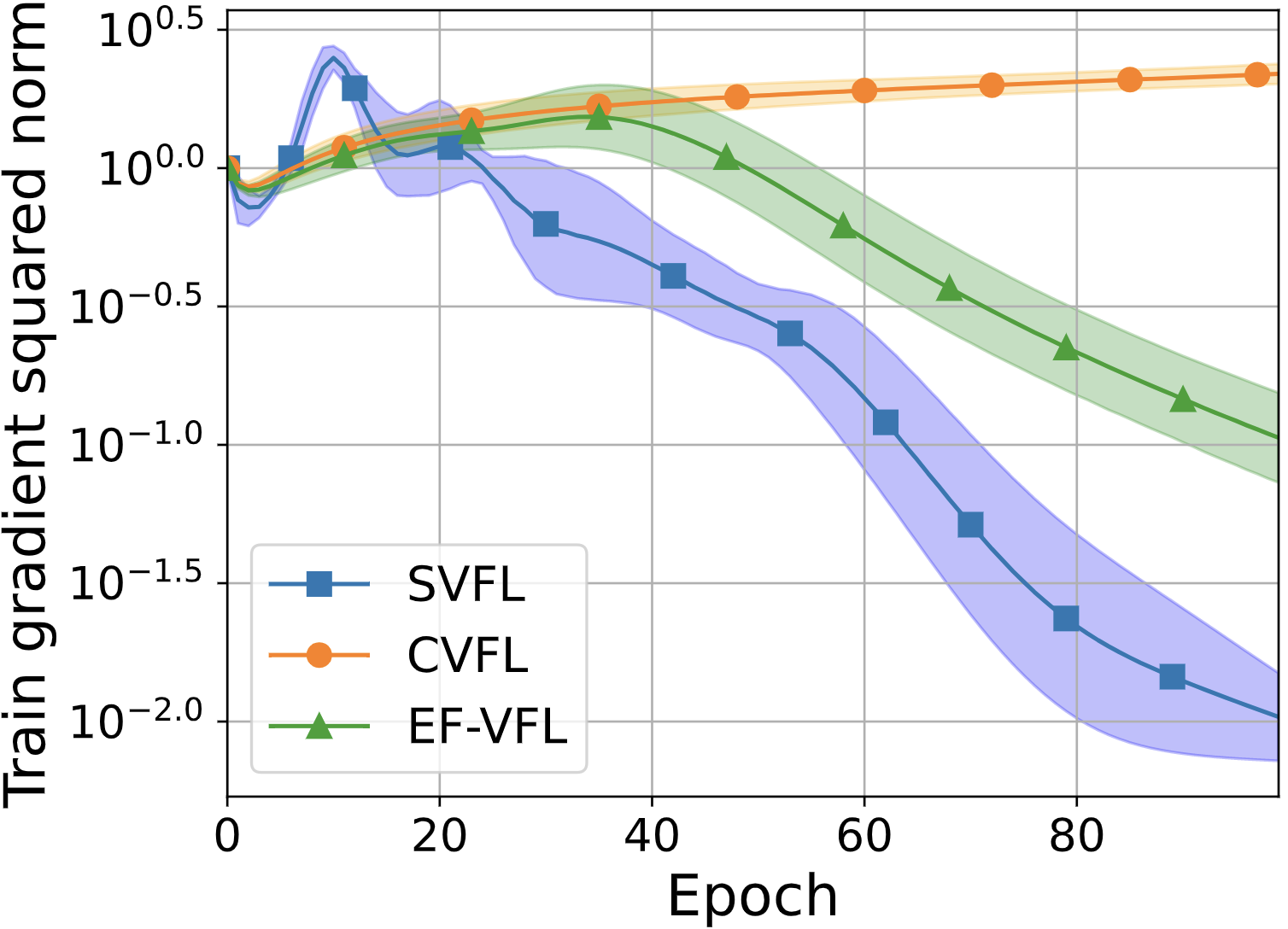}\hspace{1mm}\includegraphics[width=0.24\linewidth]{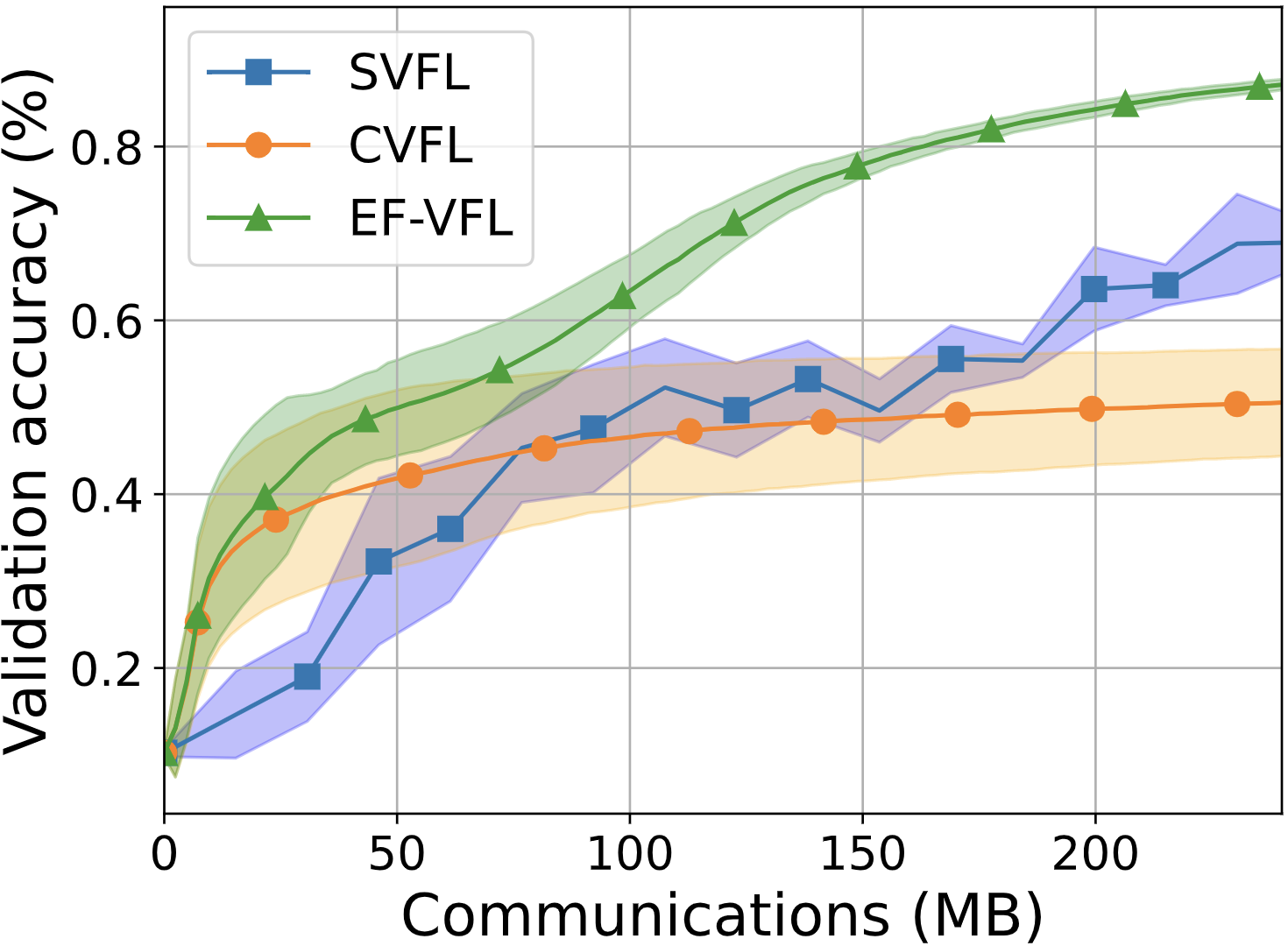}}
	\\
	\subfloat[top-$k$ keeping $1\%$\vspace{-0mm} ]{\includegraphics[width=0.24\linewidth]{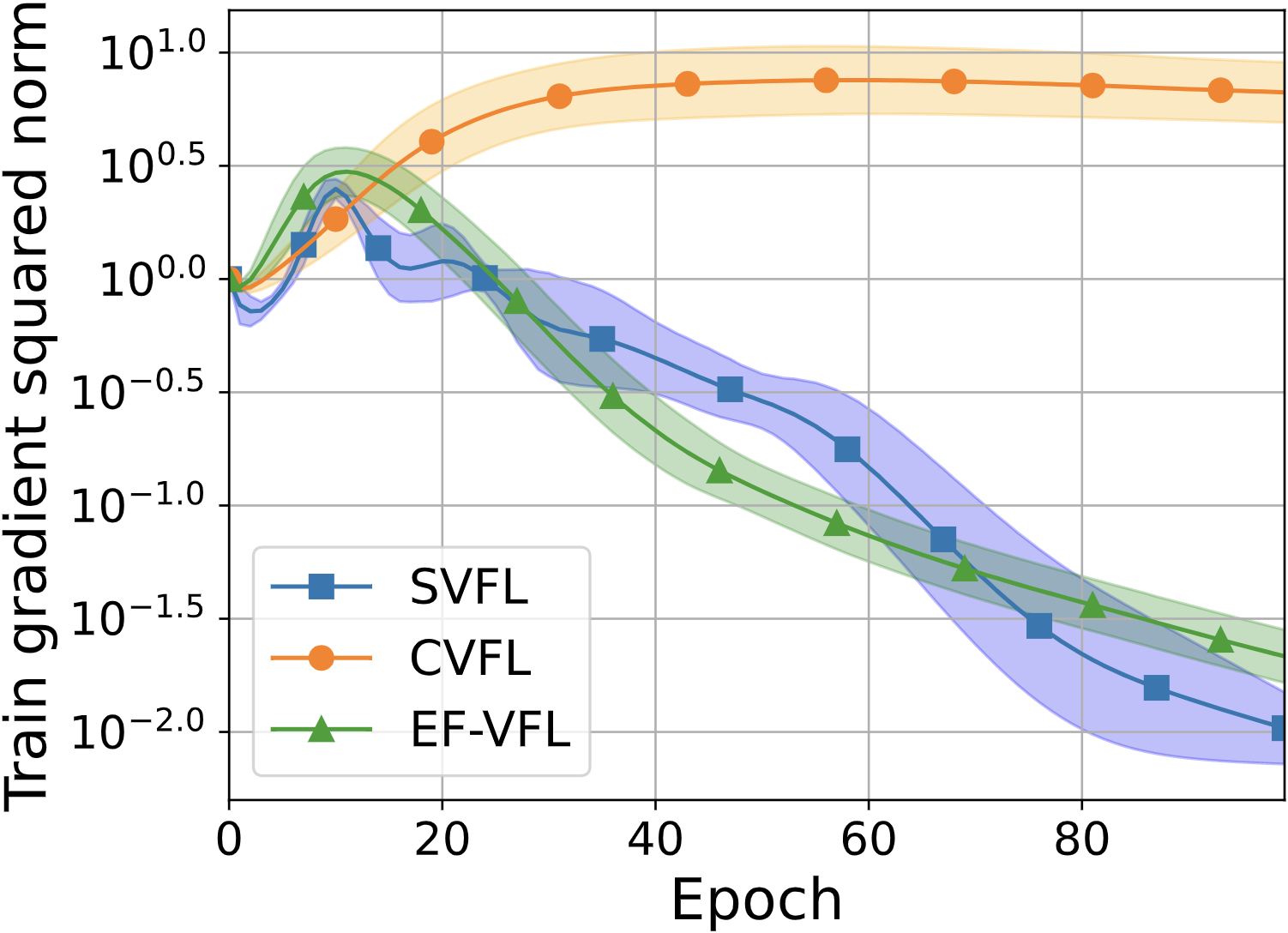}\hspace{1mm}\includegraphics[width=0.24\linewidth]{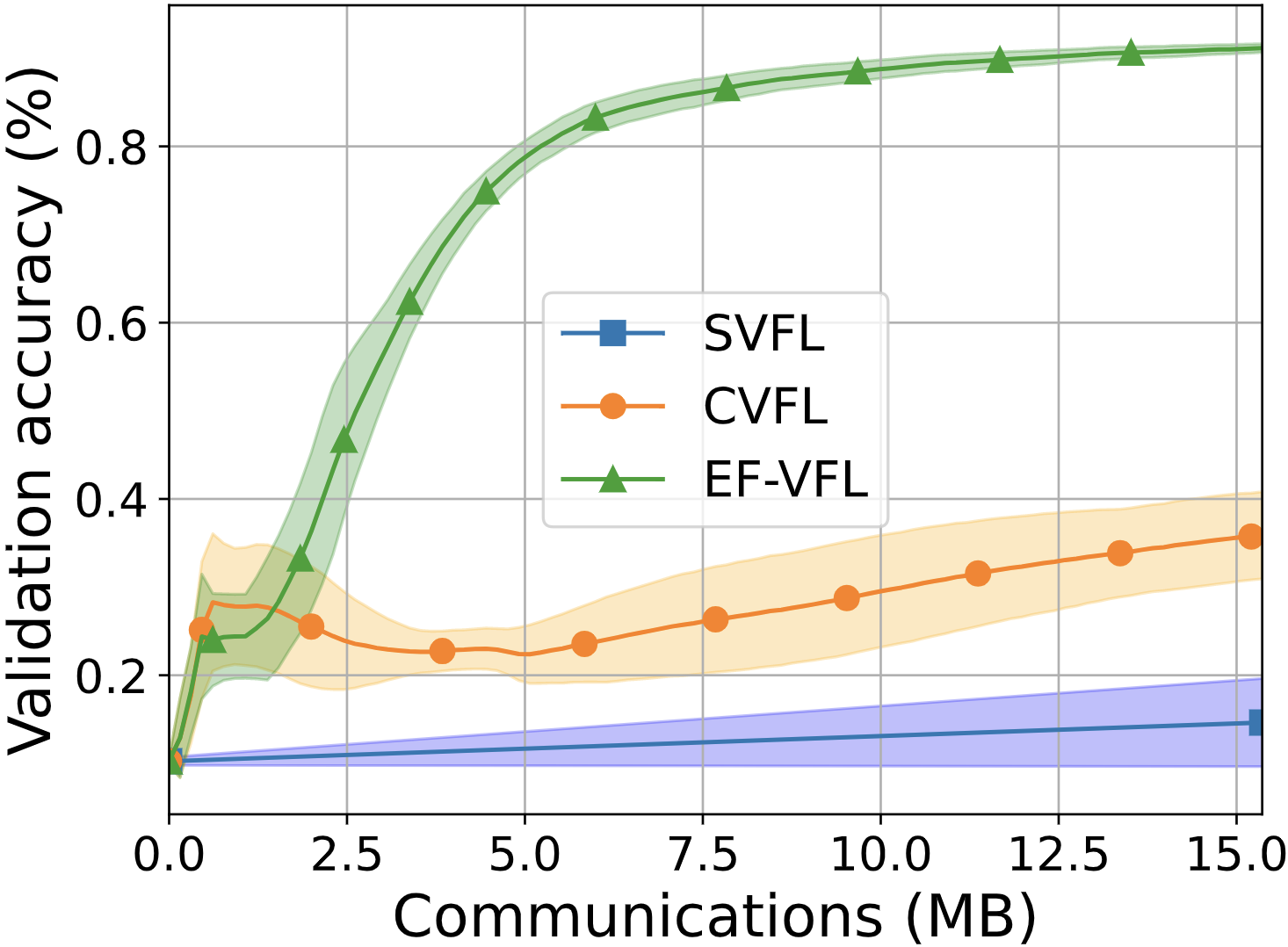}}
	\hfil
	\subfloat[quantization with $b=2$\vspace{-0mm} ]{\includegraphics[width=0.24\linewidth]{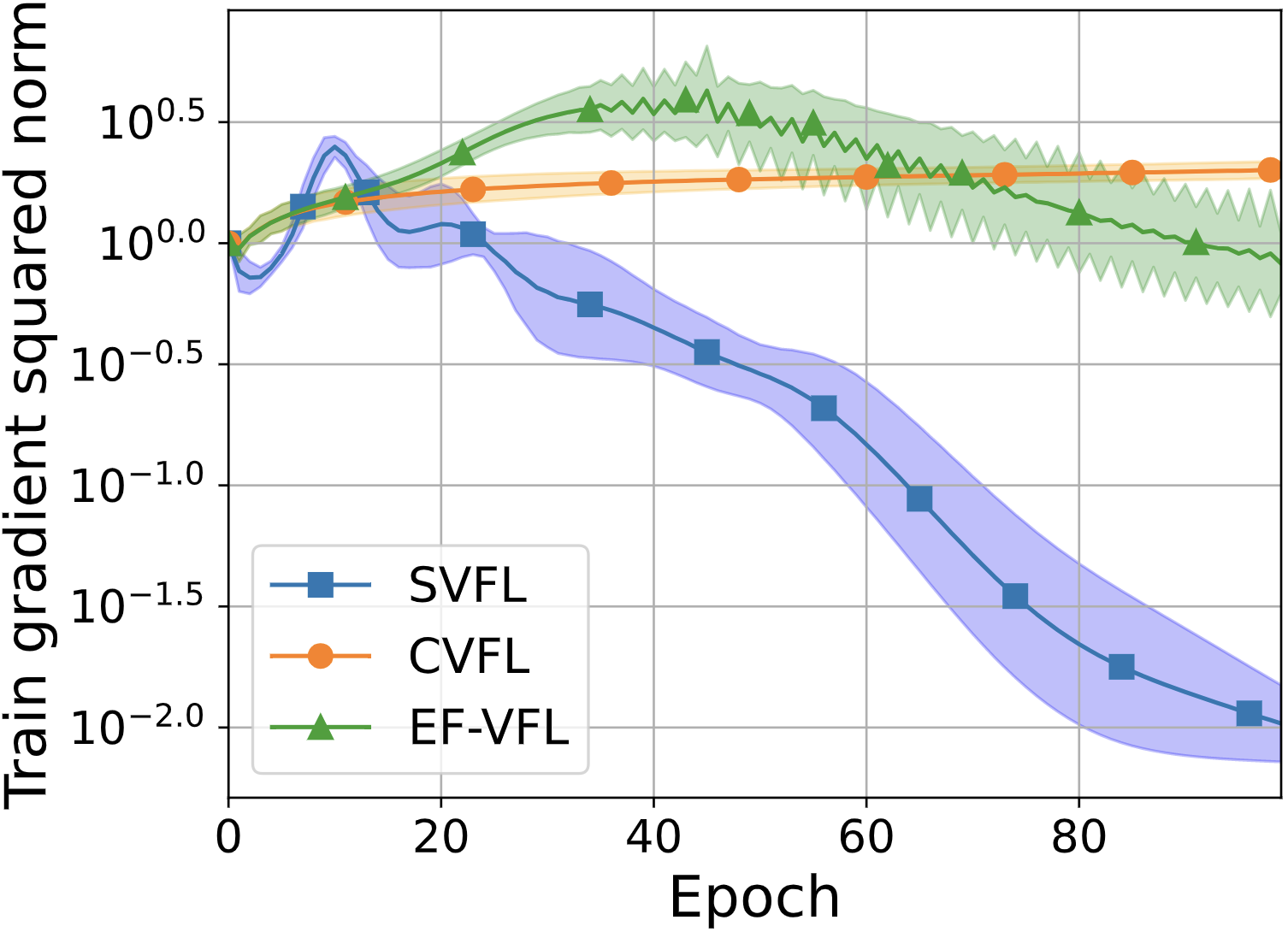}\hspace{1mm}\includegraphics[width=0.24\linewidth]{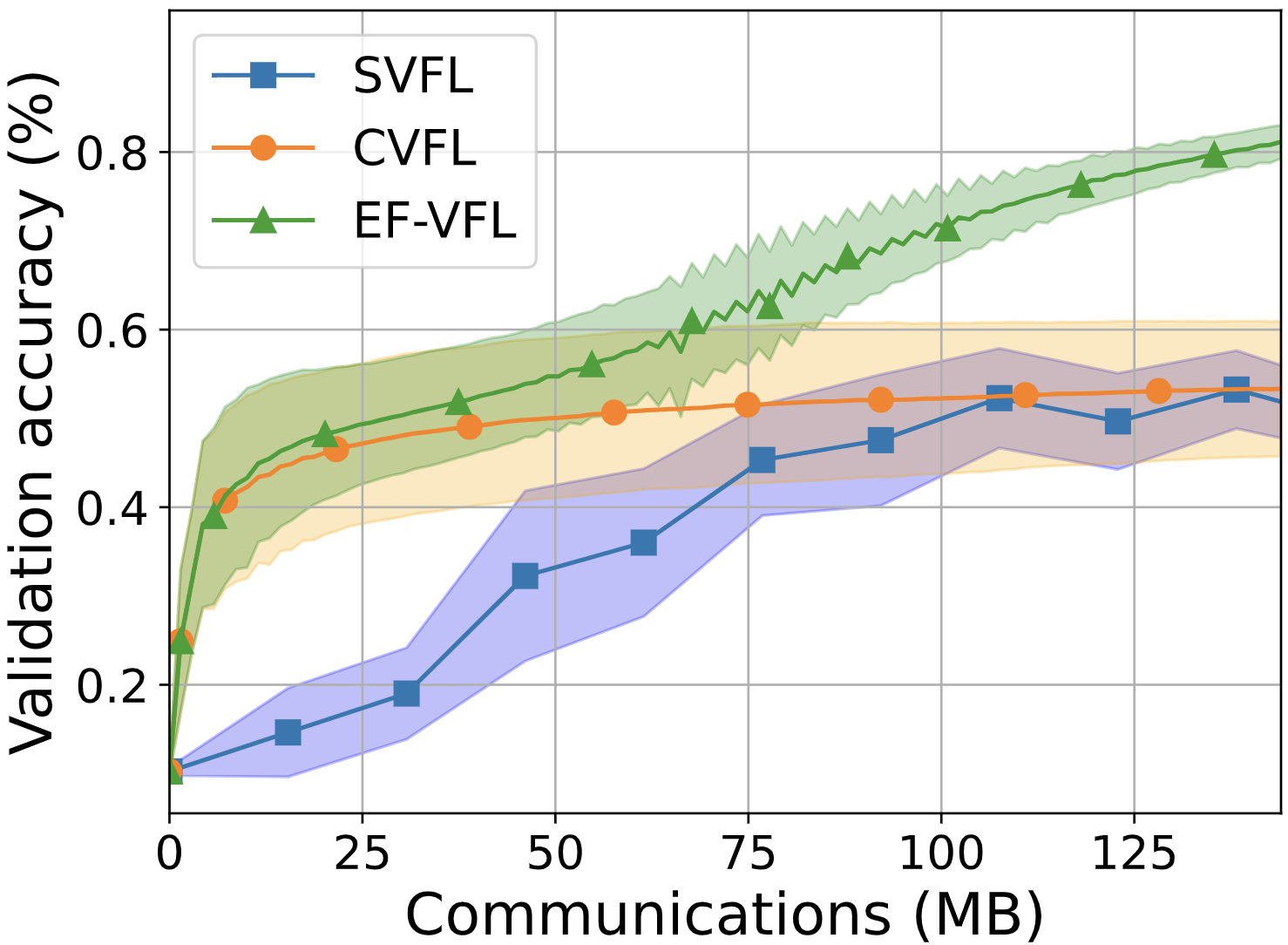}}
	\\
	\subfloat[top-$k$ keeping $0.1\%$]{\includegraphics[width=0.24\linewidth]{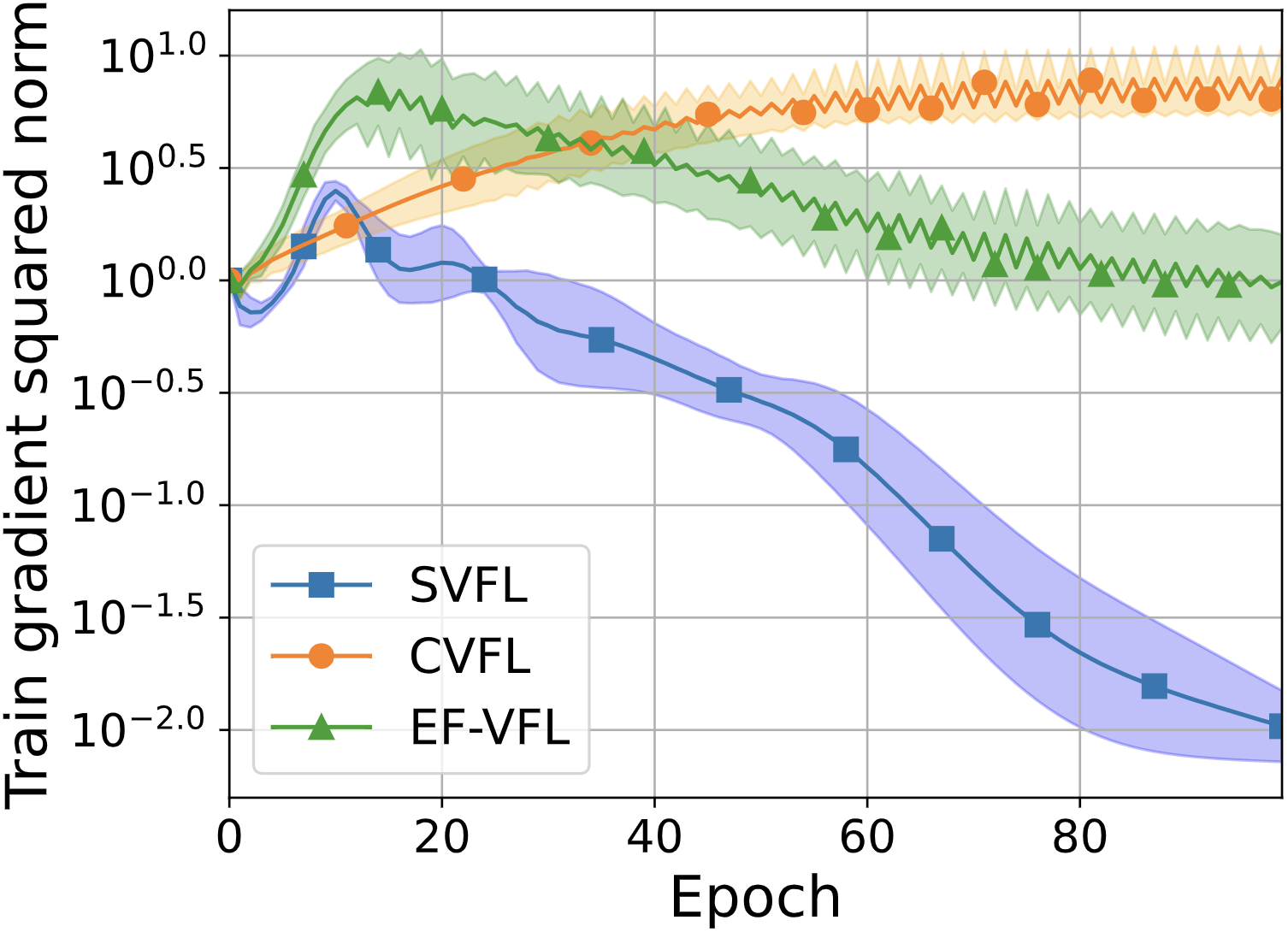}\hspace{1mm}\includegraphics[width=0.24\linewidth]{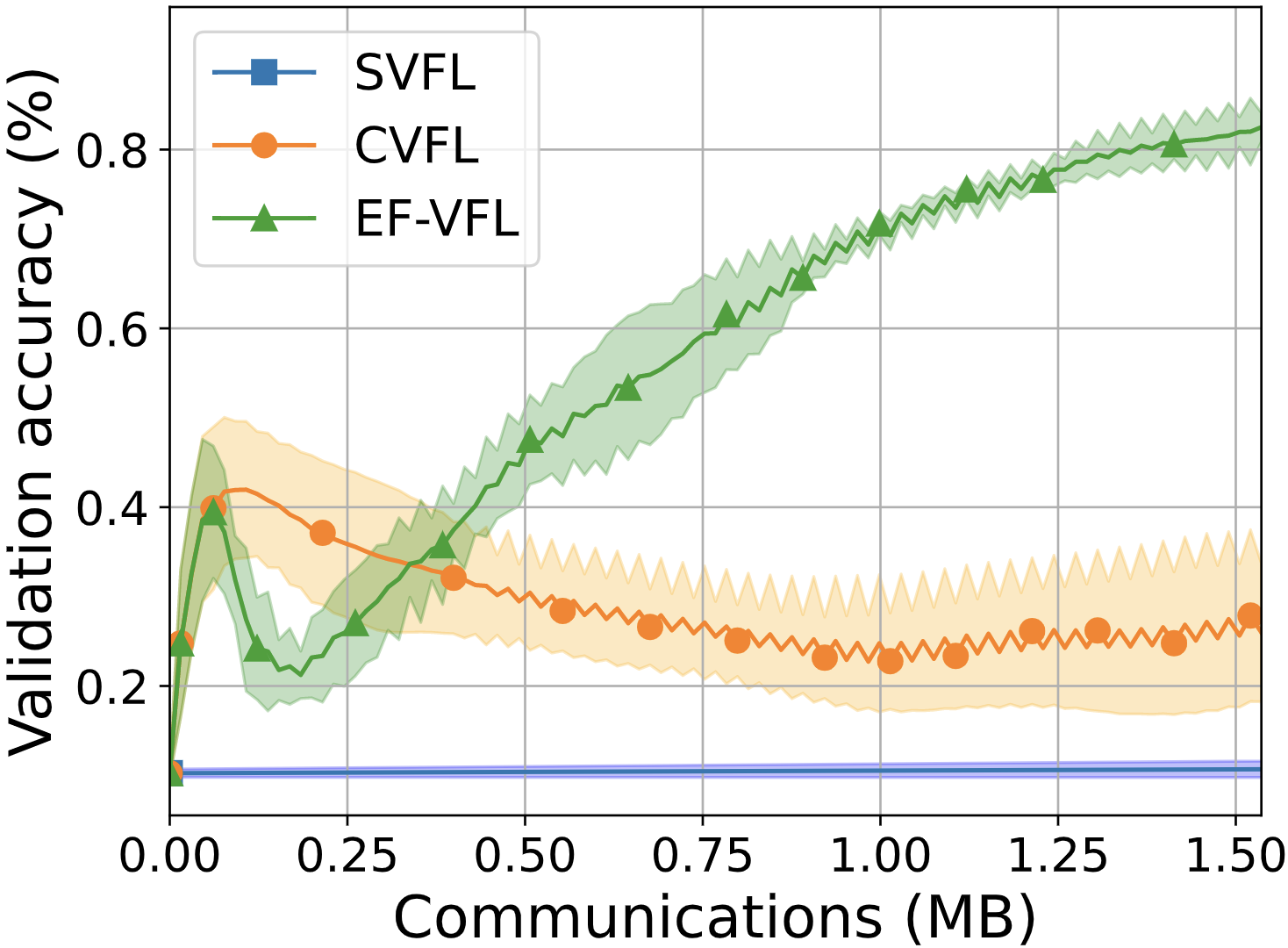}}
	\hfil
	\subfloat[quantization with $b=1$]{\includegraphics[width=0.24\linewidth]{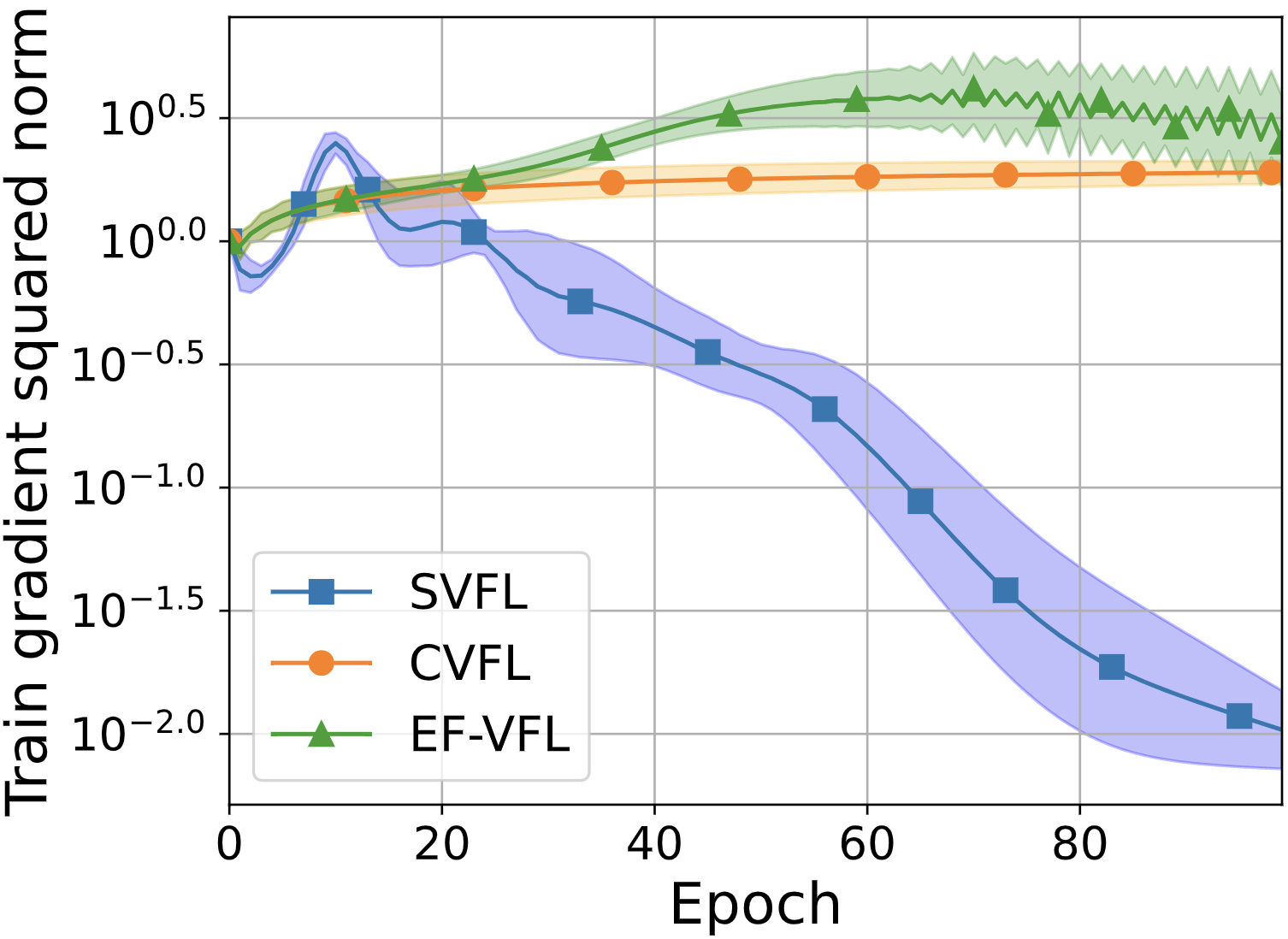}\hspace{1mm}\includegraphics[width=0.24\linewidth]{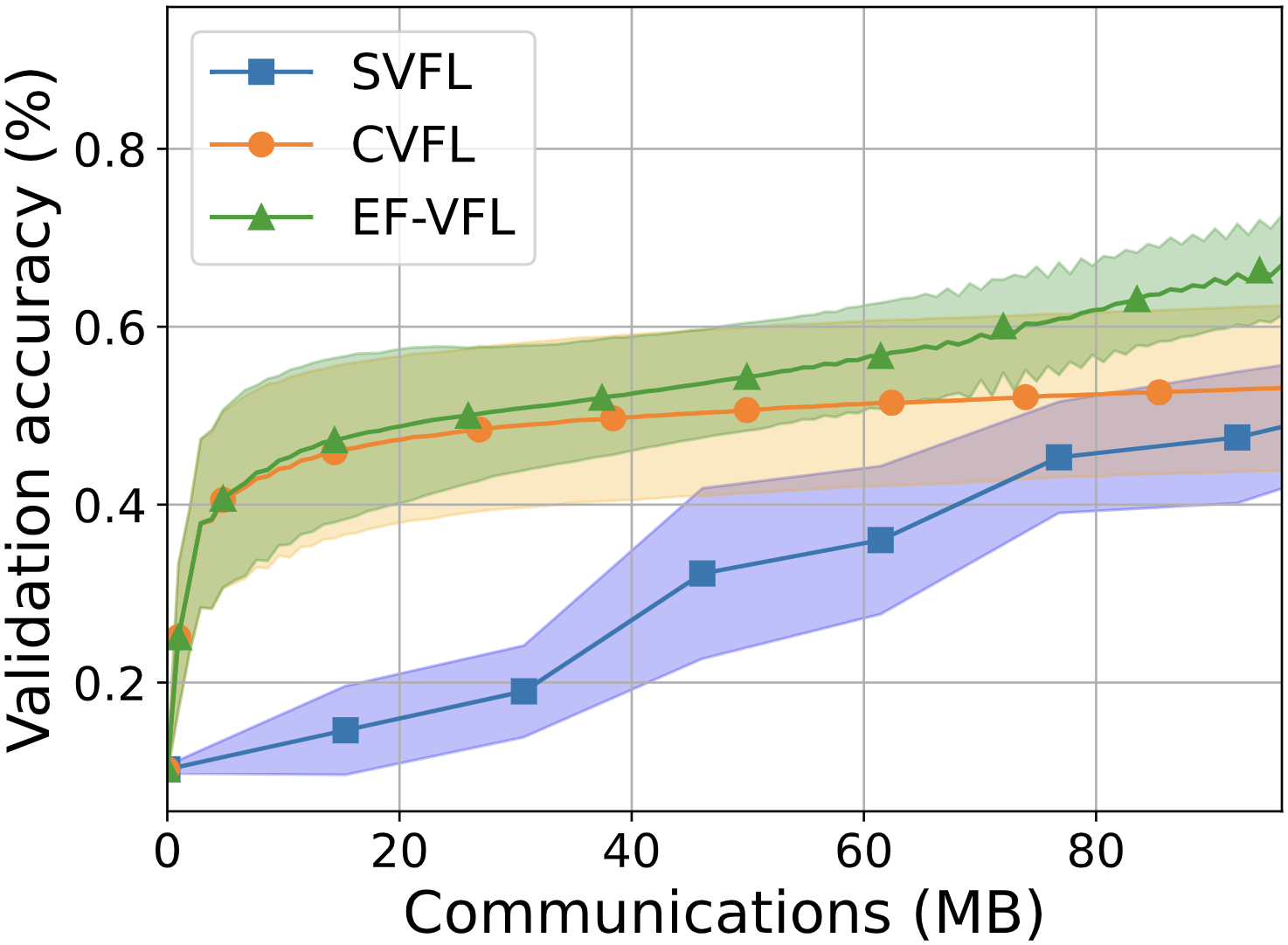}}
	\caption{The (relative) training gradient squared norm with respect to epochs
		and validation accuracy with respect to communication cost for the training of a shallow neural network on MNIST. On the left, \texttt{CVFL} and \texttt{EF-VFL} employ top-$k$ sparsification with a decreasing $k$ across rows. On the right, they employ stochastic quantization with a decreasing number of bits across rows. \texttt{SVFL} is the same throughout.
		\label{fig:mnist_exp}}
\end{figure*}

\vspace{2mm}
\paragraph{MNIST.}
We train a shallow neural network (one hidden layer) on the MNIST digit recognition dataset~\citep{lecun1998gradient}. The $28\times28$ images in the original dataset $\mathcal{D}$ are split into four local datasets $\mathcal{D}_k$ of $14\times14$ images, its quadrants ($K=4$). The local models~$\bm{h}_{kn}$ are maps $\bm{v}\mapsto \text{sigmoid}(\bm{W}_{k1}\bm{v})$, with $\bm{W}_{k1}\in\mathbb{R}^{128\times196}$, and the server model is $(\bm{v}_1,\dots,\bm{v}_4)\mapsto \bm{W}_2(\sum_{k=1}^4\bm{v}_k)$, with $\bm{W}_2\in\mathbb{R}^{10\times16}$. We use cross-entropy loss. In Figure~\ref{fig:mnist_exp}, we present the results for when \texttt{EF-VFL} and \texttt{CVFL} employ $\topk$, keeping $10\%$, $1\%$, and $0.1\%$ of the entries, and when they employ $\text{qsgd}_s$, sending $b\in\{4,2,1\}$ bits per entry, instead of the uncompressed $b=32$. In both figures, we see that \texttt{EF-VFL} outperforms \texttt{SVFL} and \texttt{CVFL} in communication efficiency. In terms of results per epoch, \texttt{EF-VFL} significantly outperforms \texttt{CVFL} and, for a sufficiently large $k$ (for $\topk$) or $b$ (for $\text{qsgd}_s$) \texttt{EF-VFL} achieves a similar performance to \texttt{SVFL}. As predicted in Section~\ref{sec:convergence_guarantees}, the train gradient squared norm during training goes to zero for \texttt{EF-VFL}, as it does for \texttt{SVFL}, but not for \texttt{CVFL}.

\vspace{2mm}
\paragraph{ModelNet10.}
We train a multi-view convolutional neural network (MVCNN)~\citep{su2015multiview} on ModelNet10~\citep{wu2015}, a dataset of three-dimensional CAD models. We use a preprocessed version of ModelNet10, where each sample is represented by 12 two-dimensional views. We assign a view per client ($K=12$). In Figure~\ref{fig:modelnet_all_exp}, we present the results for when \texttt{EF-VFL} and \texttt{CVFL} employ $\topk$, keeping $10\%$, $1\%$, and $0.1\%$ of the entries, and when they employ $\text{qsgd}_s$, with $b\in\{4,2,1\}$. We plot the train loss with respect to the number of epochs and the validation accuracy with respect to the communication cost. We observe that, for \texttt{EF-VFL}, the training loss decreases more rapidly than for \texttt{CVFL}. Further, if the compression is not excessively aggressive, \texttt{EF-VFL} performs similarly to \texttt{SVFL}. In terms of communication efficiency, \texttt{EF-VFL} outperforms both \texttt{SVFL} and \texttt{CVFL}.

\begin{figure*}[t!]
	\centering
	\subfloat[top-$k$ keeping $10\%$\vspace{-0mm} ]{\includegraphics[width=0.24\linewidth]{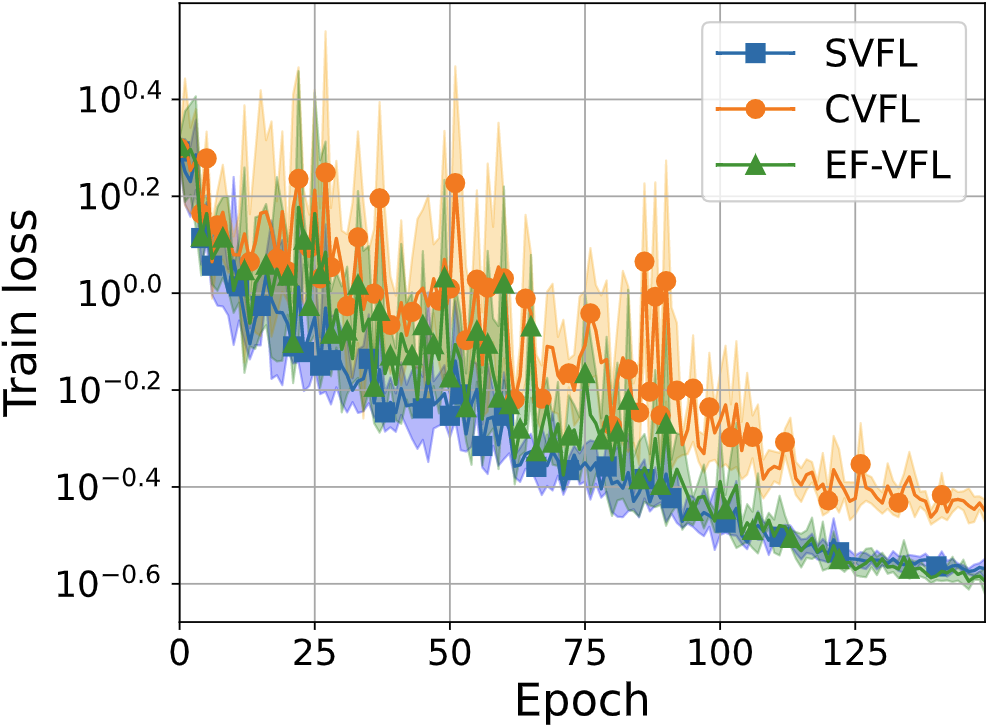}\hspace{1mm}\includegraphics[width=0.24\linewidth]{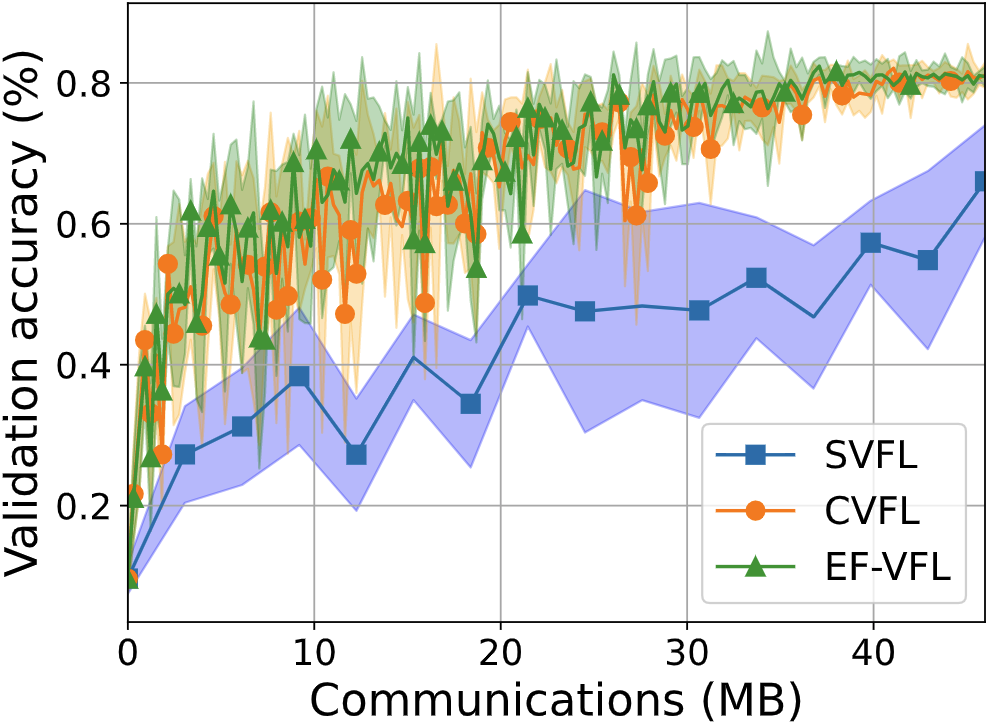}}
	\hfil
	\subfloat[quantization with $b=4$\vspace{-0mm} ]{\includegraphics[width=0.24\linewidth]{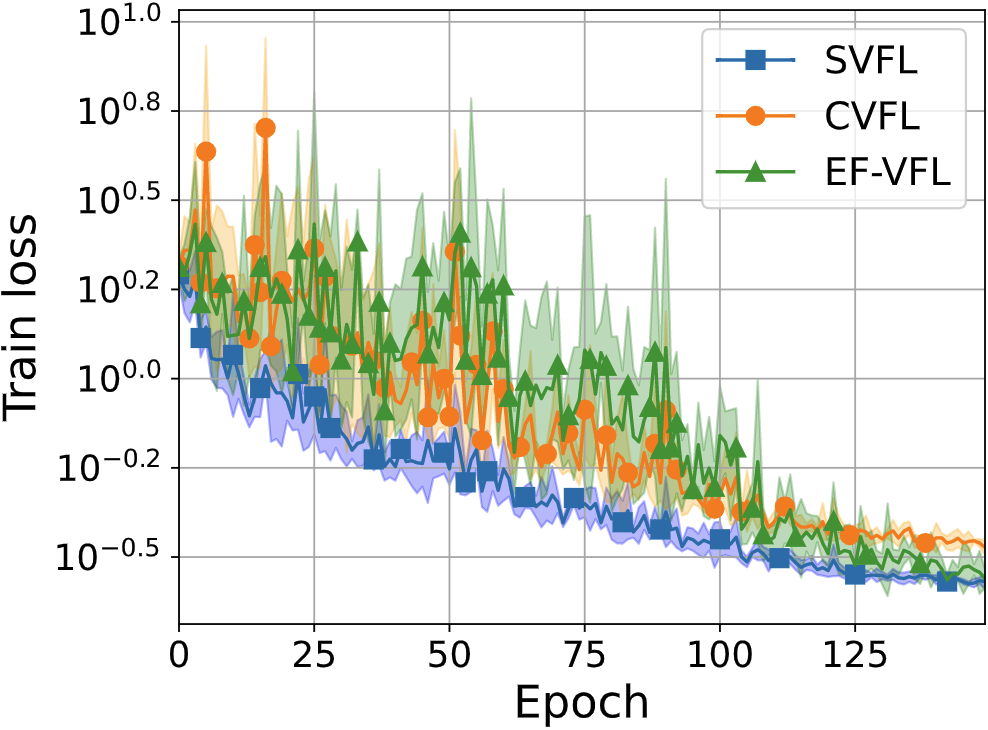}\hspace{1mm}\includegraphics[width=0.24\linewidth]{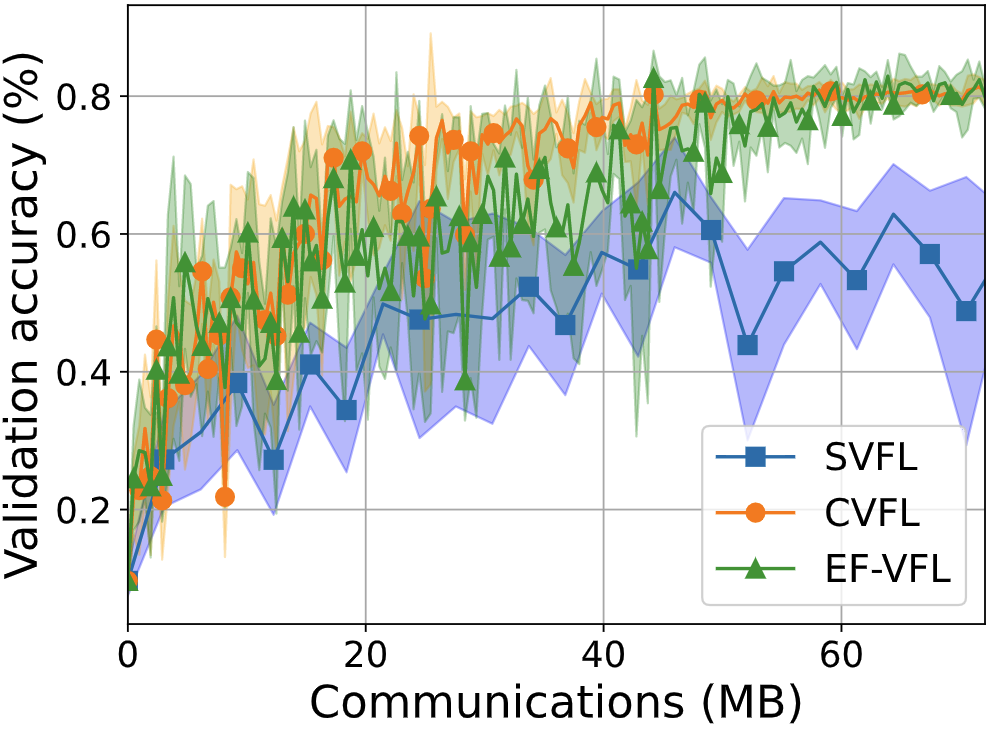}}
	\\
	\subfloat[top-$k$ keeping $1\%$\vspace{-0mm} ]{\includegraphics[width=0.24\linewidth]{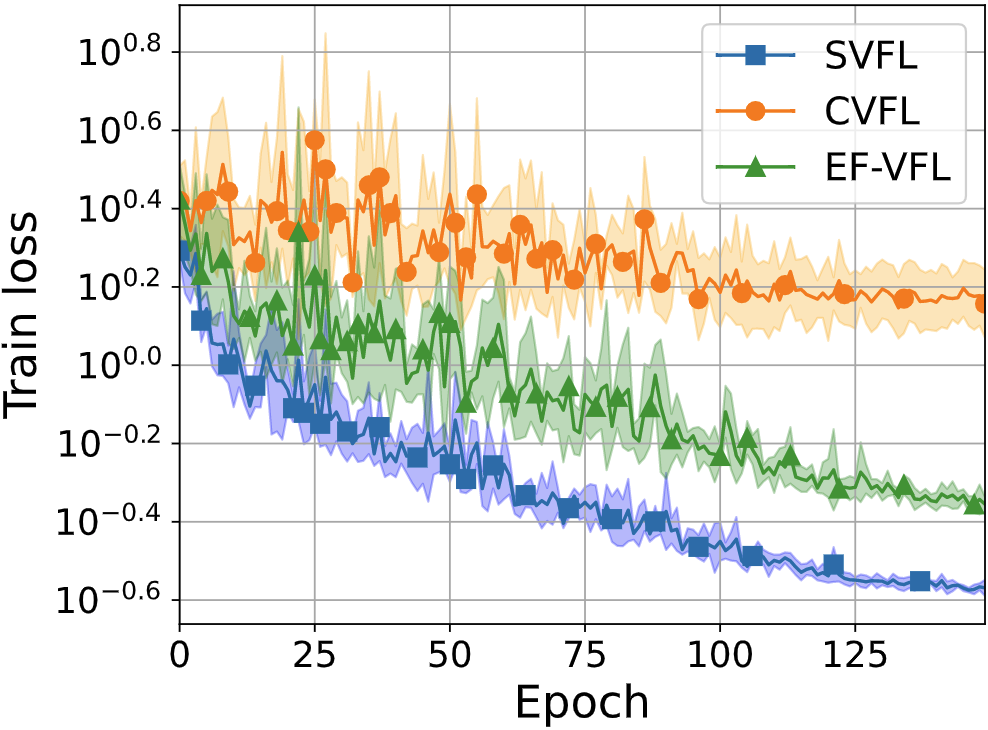}\hspace{1mm}\includegraphics[width=0.24\linewidth]{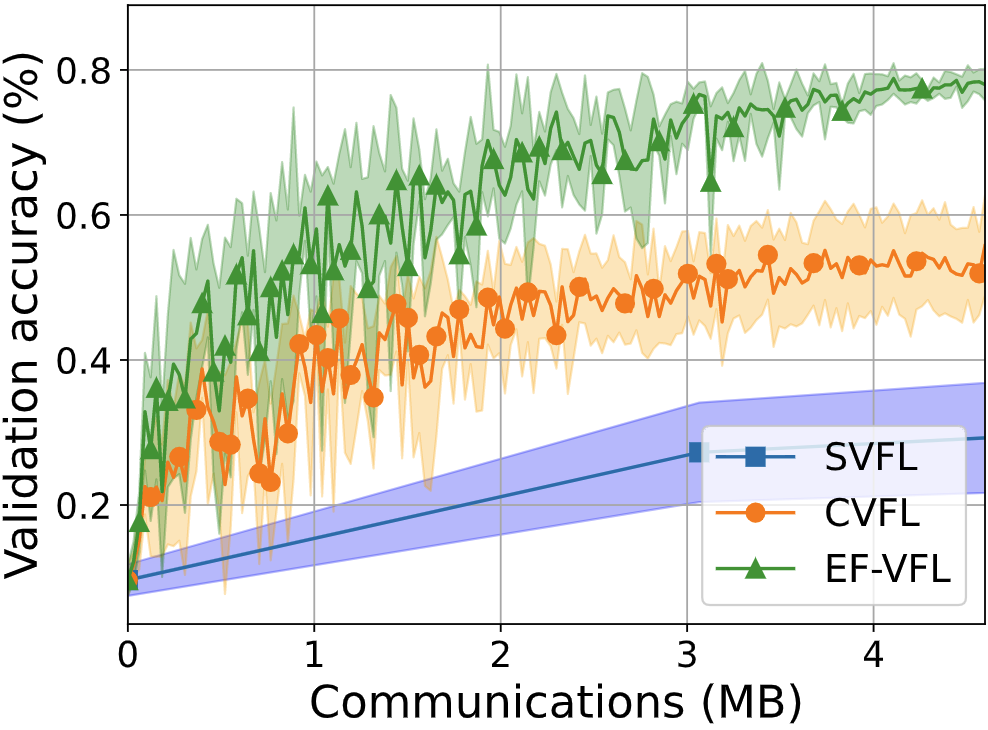}}
	\hfil
	\subfloat[quantization with $b=2$\vspace{-0mm} ]{\includegraphics[width=0.24\linewidth]{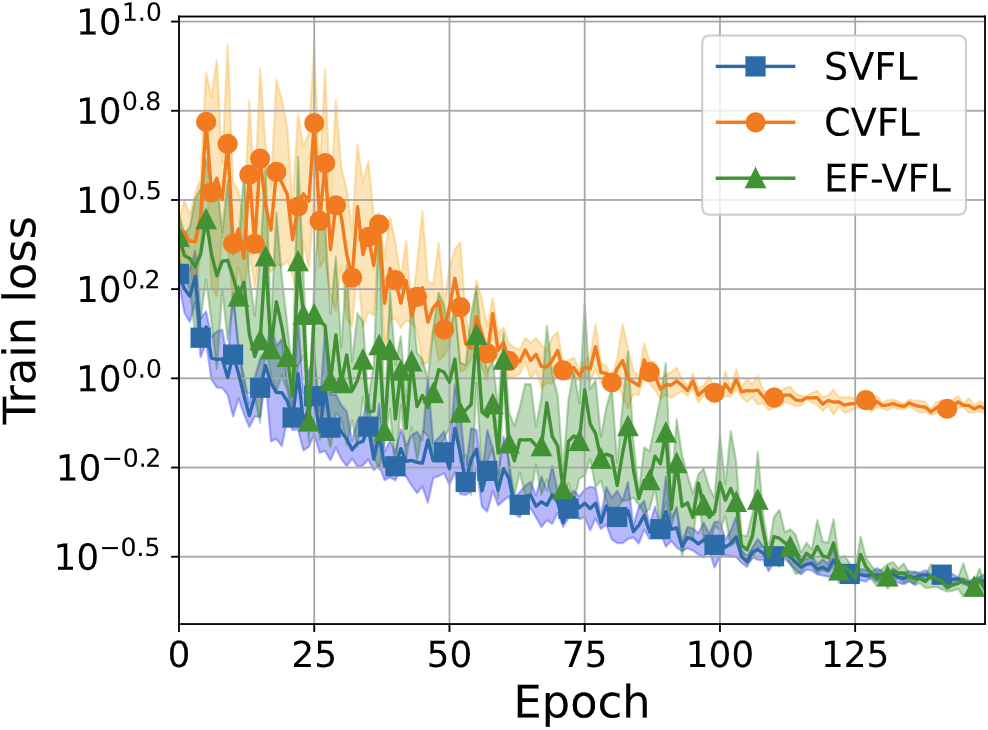}\hspace{1mm}\includegraphics[width=0.24\linewidth]{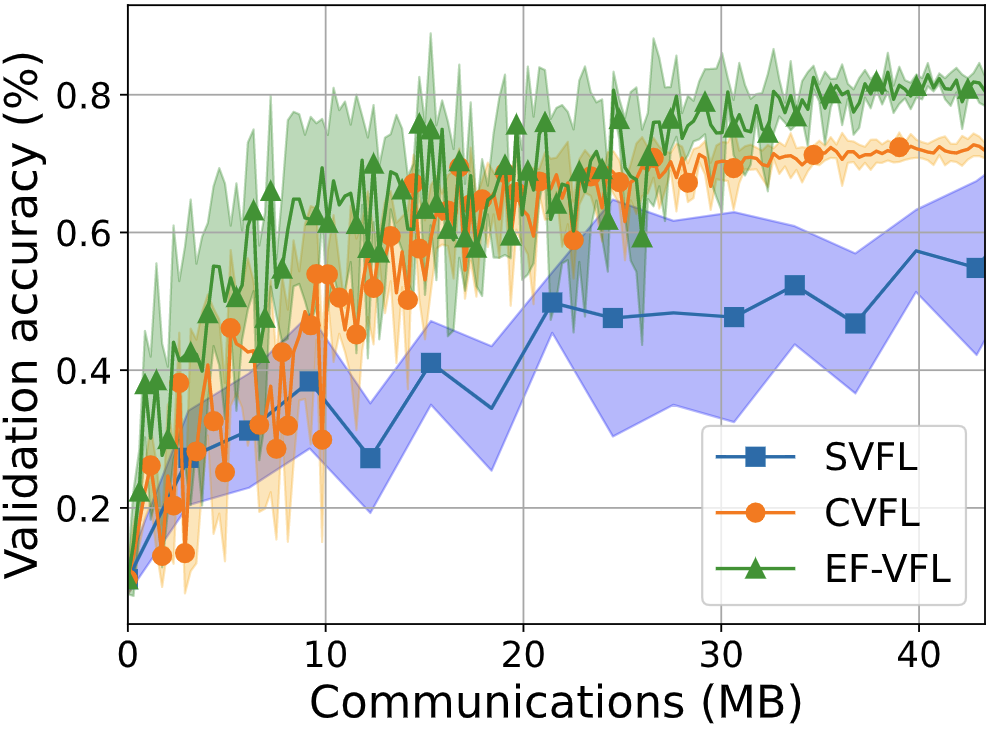}}
	\\
	\subfloat[top-$k$ keeping $0.1\%$]{\includegraphics[width=0.24\linewidth]{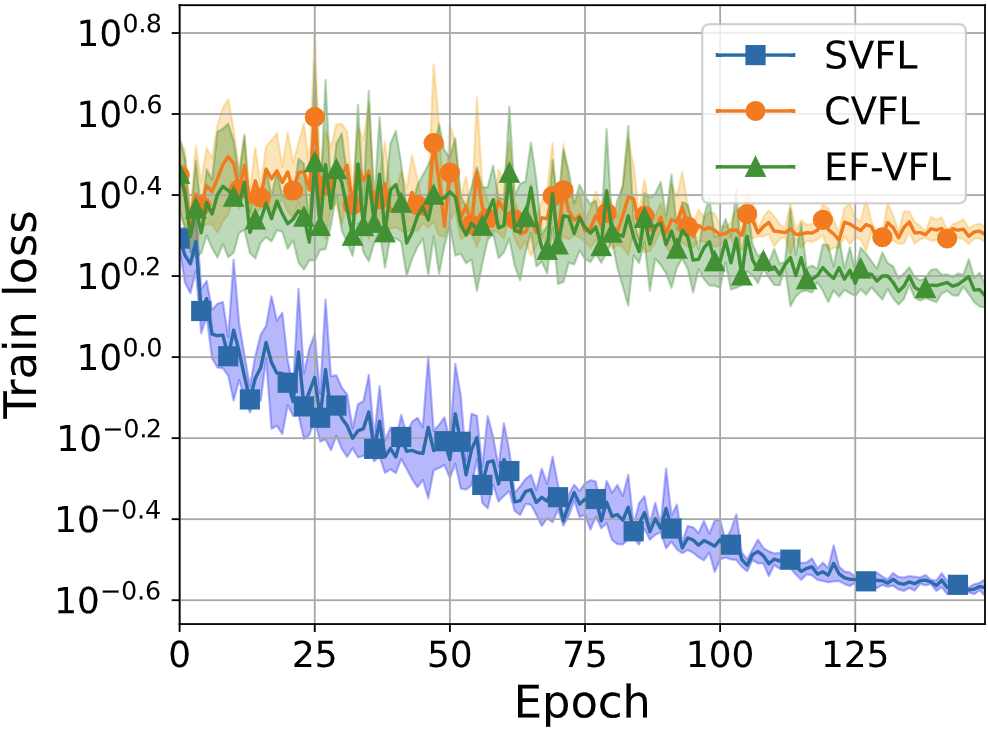}\hspace{1mm}\includegraphics[width=0.24\linewidth]{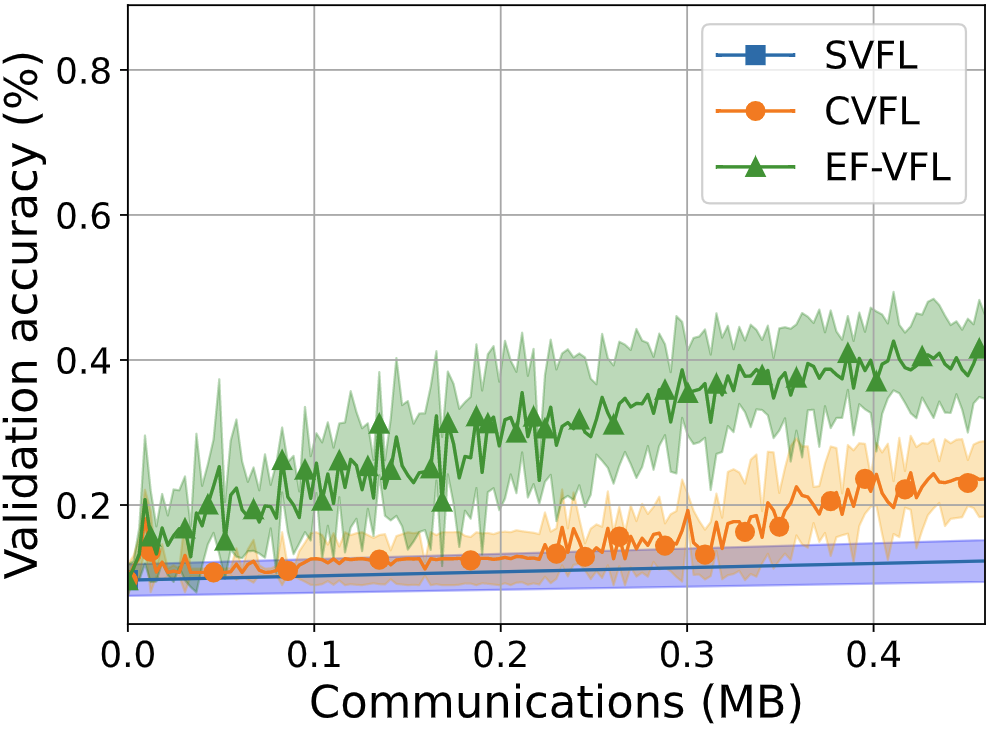}}
	\hfil
	\subfloat[quantization with $b=1$]{\includegraphics[width=0.24\linewidth]{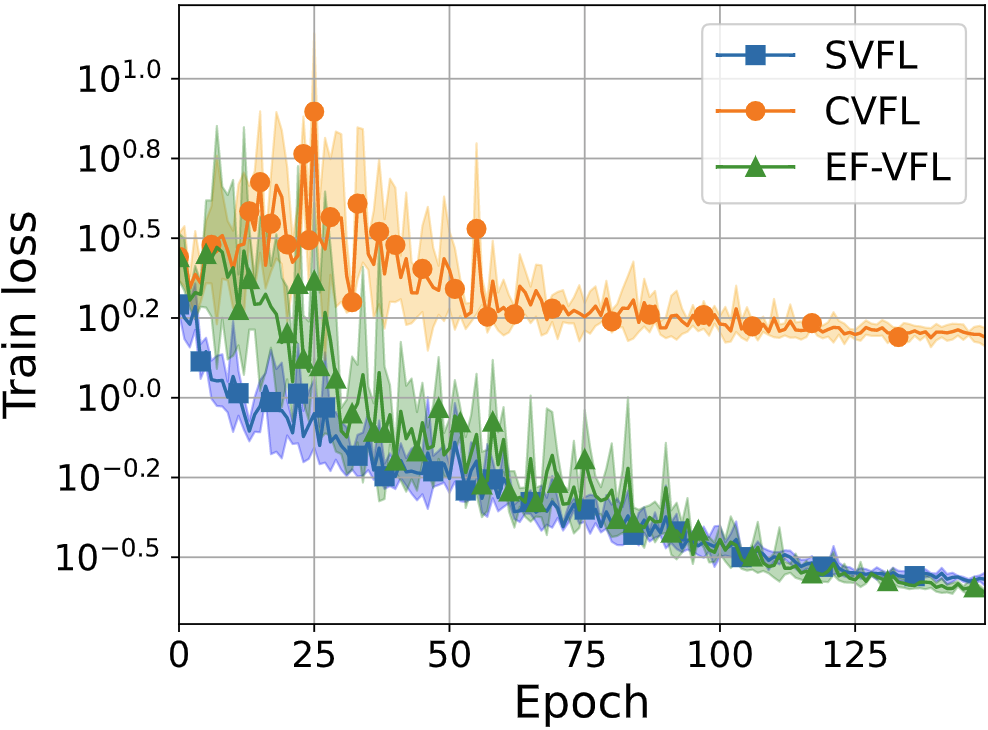}\hspace{1mm}\includegraphics[width=0.24\linewidth]{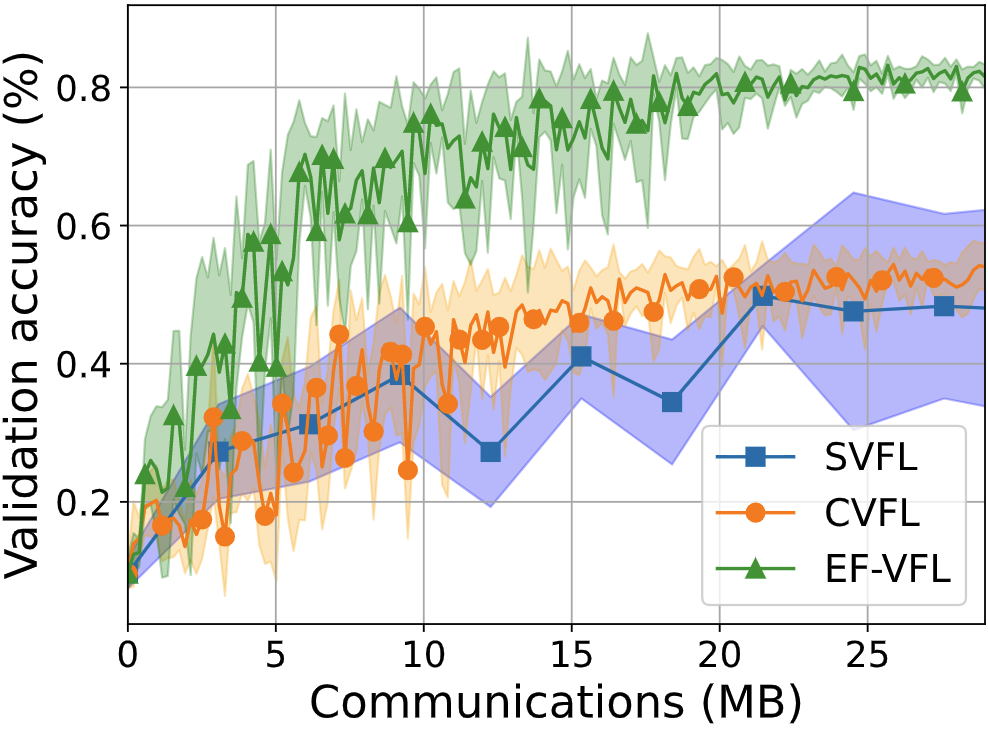}}
	\caption{Train loss with respect to the number of epochs and validation accuracy with respect to the communication cost for the training of an MVCNN on ModelNet10. On the left, \texttt{CVFL} and \texttt{EF-VFL} employ top-$k$ sparsification with a decreasing $k$ across rows. On the right, they employ stochastic quantization with a decreasing number of bits across rows. \texttt{SVFL} is the same throughout.
		\label{fig:modelnet_all_exp}}
\end{figure*}

\vspace{2mm}
\paragraph{CIFAR-100.}
We train a model based on a residual neural network, ResNet18~\citep{he2015deep}, on CIFAR-100~\citep{krizhevsky2009learning}. More precisely, we divide each image into four quadrants and allocate one quadrant to each client ($K=4$), with each client using a ResNet18 model as its local model. The server model is linear (a single layer). In Figure~\ref{fig:resnet_all_exp}, we present the results for when \texttt{EF-VFL} and \texttt{CVFL} employ $\topk$, keeping $10\%$, $1\%$, and $0.1\%$ of the entries, and when they employ $\text{qsgd}_s$, with $b\in\{4,2,1\}$. We plot the train loss with respect to the number of epochs and the validation accuracy with respect to the communication cost. Regarding the results with respect to the number of epochs, \texttt{EF-VFL} achieves a similar performance to that of \texttt{SVFL}, significantly outperforming \texttt{CVFL}. In terms of communication efficiency, \texttt{EF-VFL} outperforms both \texttt{SVFL} and \texttt{CVFL}.

{
	We summarize the test metrics for all three tasks in Table~\ref{tab:test_metrics}. In brief, while \texttt{CVFL} performs well for less aggressive compression is employed, the improved performance of \texttt{EF-VFL} is significant when more aggressive compression is employed.
}

\begin{table*}[t]
	\caption{
		{Test accuracy for \texttt{SVFL}, \texttt{CVFL}, and \texttt{EF-VFL} across different tasks, for a fixed number of epochs. We run reach experiment for 5 seeds and present the mean accuracy $\pm$ standard deviation, highlighting the highest accuracy in bold.}
		\label{tab:test_metrics}
	}
	\vspace{-1mm}
	\centering
		\begin{tabular}{l*{7}{c}}
			\specialrule{1.2pt}{0pt}{0pt}
			\noalign{\vskip 1mm}
			\multicolumn{8}{c}{$\qquad\qquad\qquad\quad${MNIST (accuracy, \%)}}\\
			\specialrule{1.2pt}{0pt}{0pt}
			&  & \multicolumn{3}{c}{{top-$k$ compressor}} & \multicolumn{3}{c}{{qsgd compressor}} \\
			\cline{3-8}
			\noalign{\vskip 1mm}
			& {Uncompressed} & {keep $10\%$} & {keep $1\%$} & {keep $0.1\%$} & {$ b=4 $} & {$ b=2 $} & {$ b=1 $} \\
			\specialrule{1.2pt}{0pt}{0pt}
			{\texttt{SVFL}} & $91.6\pm0.1$ & --- & --- & --- & --- & --- & --- \\
			{\texttt{CVFL}} & --- & $77.2\pm1.9$ & $35.7\pm5.0$ & $25.7\pm7.6$ & $50.3\pm6.1$ & $53.0\pm7.4$ & $52.7\pm8.7$ \\
			{\texttt{EF-VFL}} & --- & $\textbf{91.8}\pm0.1$ & $\textbf{91.1}\pm0.3$ & $\textbf{82.4}\pm2.1$ & $\textbf{87.2}\pm0.4$ & $\textbf{81.1}\pm1.6$ & $\textbf{66.8}\pm5.9$ \\
			\specialrule{1.2pt}{0pt}{0pt} 
			\noalign{\vskip 1mm}
			\multicolumn{8}{c}{$\qquad\qquad\qquad\quad$ModelNet10 (accuracy, \%)}\\
			\specialrule{1.2pt}{0pt}{0pt}
			&  & \multicolumn{3}{c}{{top-$k$ compressor}} & \multicolumn{3}{c}{{qsgd compressor}} \\
			\cline{3-8}
			\noalign{\vskip 1mm}
			& {Uncompressed} & {keep $10\%$} & {keep $1\%$} & {keep $0.1\%$} & {$ b=4 $} & {$ b=2 $} & {$ b=1 $} \\
			\specialrule{1.2pt}{0pt}{0pt} 
			{\texttt{SVFL}} & $81.2\pm0.8$ & --- & --- & --- & --- & --- & --- \\
			{\texttt{CVFL}} & --- & $\textbf{80.7}\pm3.1$ & $53.2\pm{6.5}$ & $24.5\pm4.2$ & $\textbf{80.3}\pm2.0$ & $70.7\pm1.6$ & $52.0\pm{3.2}$ \\
			{\texttt{EF-VFL}} & --- & $80.4\pm1.9$ & $\textbf{77.4}\pm2.5$ & $\textbf{40.3}\pm{4.8}$ & $79.4\pm{3.7}$ & $\textbf{80.4}\pm{2.7}$ & $\textbf{81.1}\pm2.8$ \\
			\specialrule{1.2pt}{0pt}{0pt} 
			\noalign{\vskip 1mm}
			\multicolumn{8}{c}{$\qquad\qquad\qquad\quad$CIFAR-100 (accuracy, \%)}\\
			\specialrule{1.2pt}{0pt}{0pt}
			&  & \multicolumn{3}{c}{{top-$k$ compressor}} & \multicolumn{3}{c}{{qsgd compressor}} \\
			\cline{3-8}
			\noalign{\vskip 1mm}
			& {Uncompressed} & {keep $10\%$} & {keep $1\%$} & {keep $0.1\%$} & {$ b=4 $} & {$ b=2 $} & {$ b=1 $} \\
			\specialrule{1.2pt}{0pt}{0pt}
			{\texttt{SVFL}} & $57.7\pm0.6$ & --- & --- & --- & --- & --- & --- \\
			{\texttt{CVFL}} & --- & $56.8\pm0.6$ & $45.1\pm{2.3}$ & $11.6\pm1.4$ & $19.2\pm{0.6}$ & $5.9\pm0.7$ & $2.0\pm0.1$ \\
			{\texttt{EF-VFL}} & --- & $\textbf{57.2}\pm{0.8}$ & $\textbf{54.8}\pm0.9$ & $\textbf{36.4}\pm{2.8}$ & $\textbf{57.8}\pm0.5$ & $\textbf{50.2}\pm{1.3}$ & $\textbf{34.7}\pm{2.8}$ \\
			\specialrule{1.2pt}{0pt}{0pt} 
		\end{tabular}
	\vspace{-3mm}
\end{table*}

\begin{figure*}[t!]
	\centering
	\subfloat[top-$k$ keeping $10\%$\vspace{-0mm} ]{%
		\includegraphics[width=0.24\linewidth]{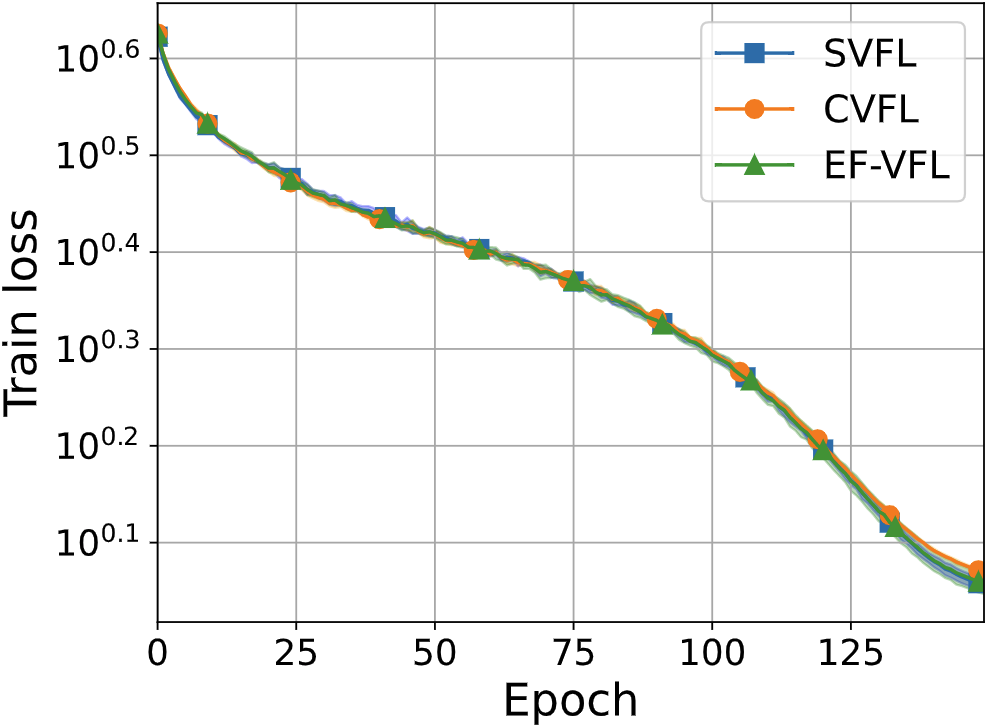}%
		\hspace{1mm}%
		\includegraphics[width=0.24\linewidth]{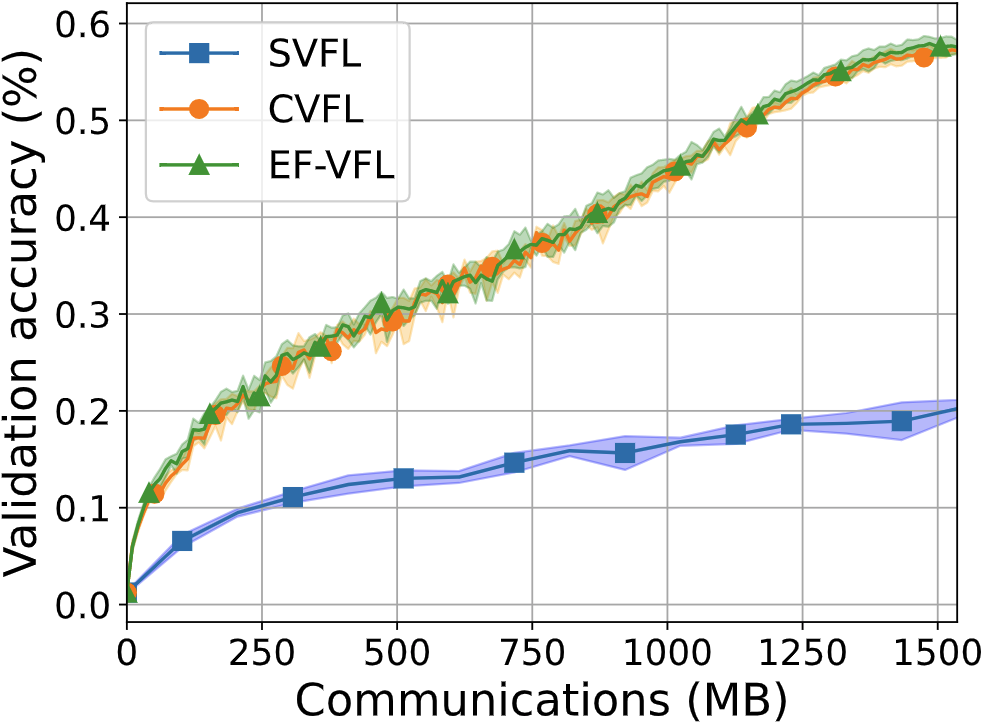}%
	}%
	\hfil%
	\subfloat[quantization with $b=4$\vspace{-0mm} ]{%
		\includegraphics[width=0.24\linewidth]{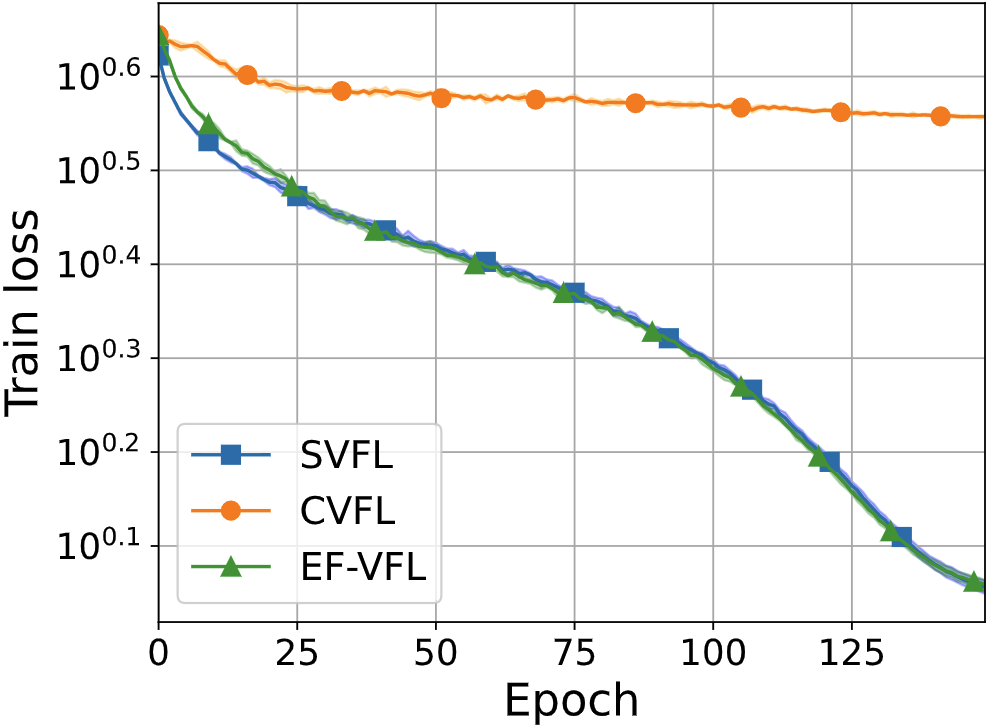}%
		\hspace{1mm}%
		\includegraphics[width=0.24\linewidth]{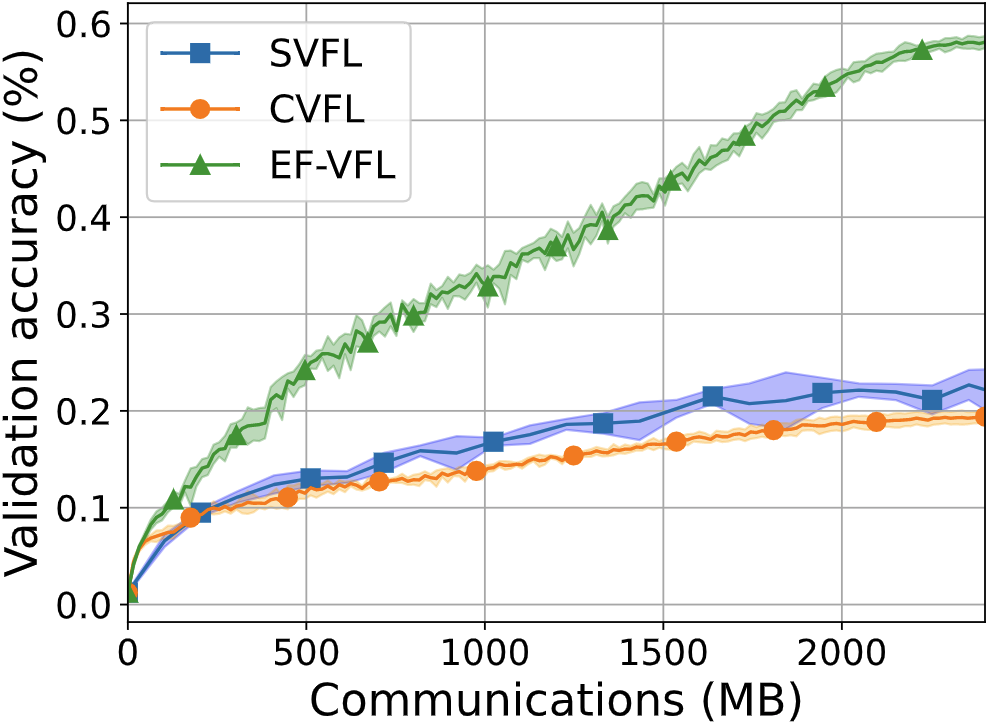}%
	}%
	\\
	\subfloat[top-$k$ keeping $1\%$\vspace{-0mm} ]{%
		\includegraphics[width=0.24\linewidth]{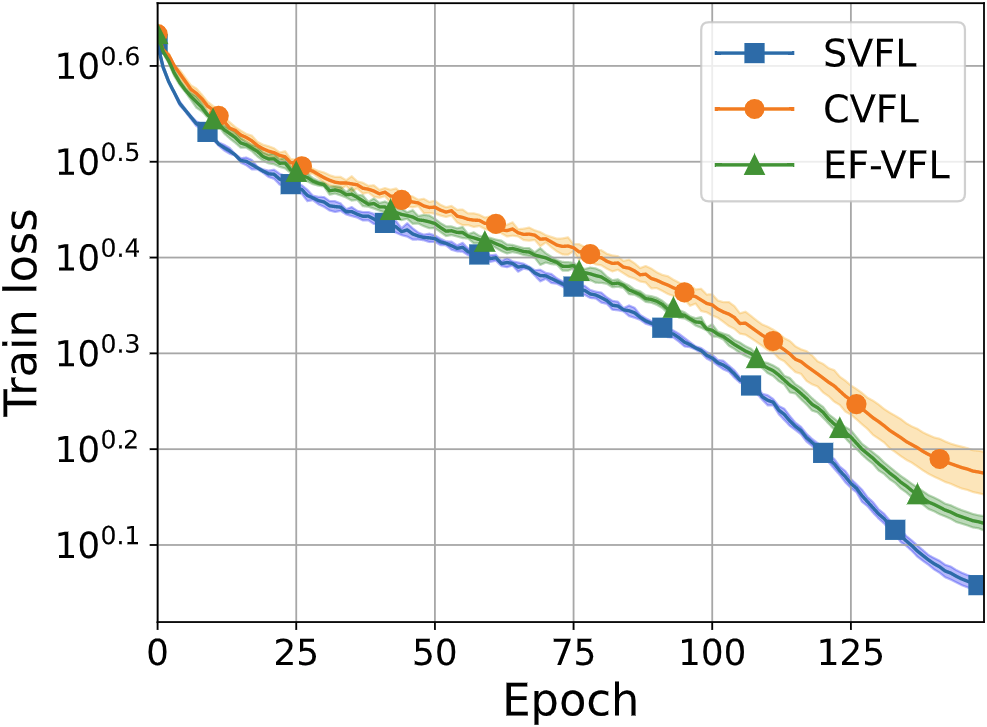}%
		\hspace{1mm}%
		\includegraphics[width=0.24\linewidth]{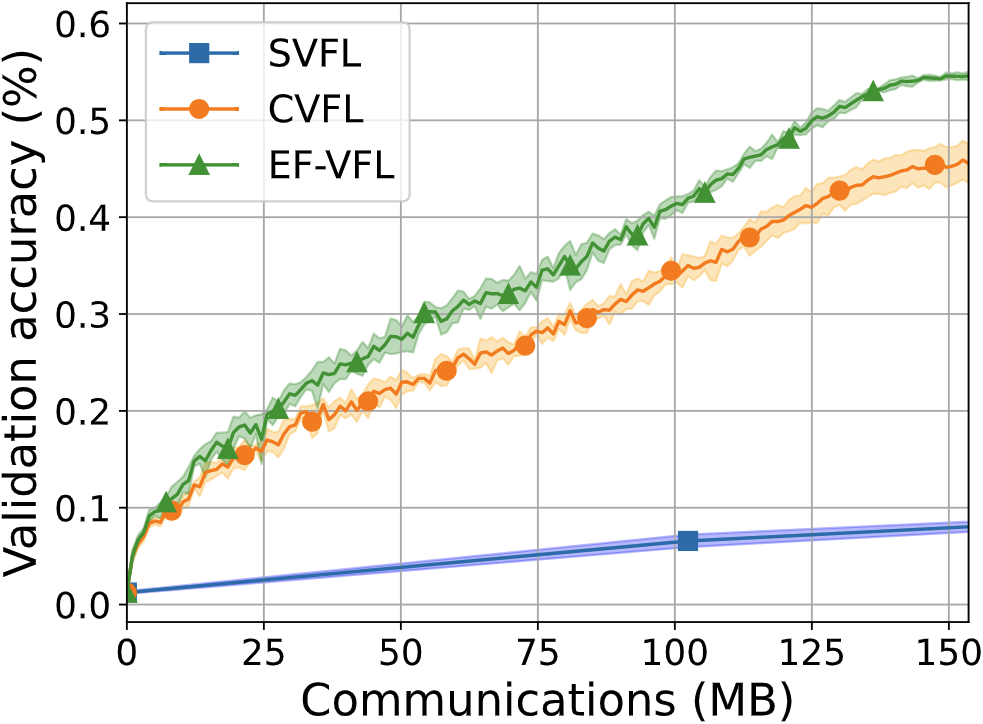}%
	}%
	\hfil%
	\subfloat[quantization with $b=2$\vspace{-0mm} ]{%
		\includegraphics[width=0.24\linewidth]{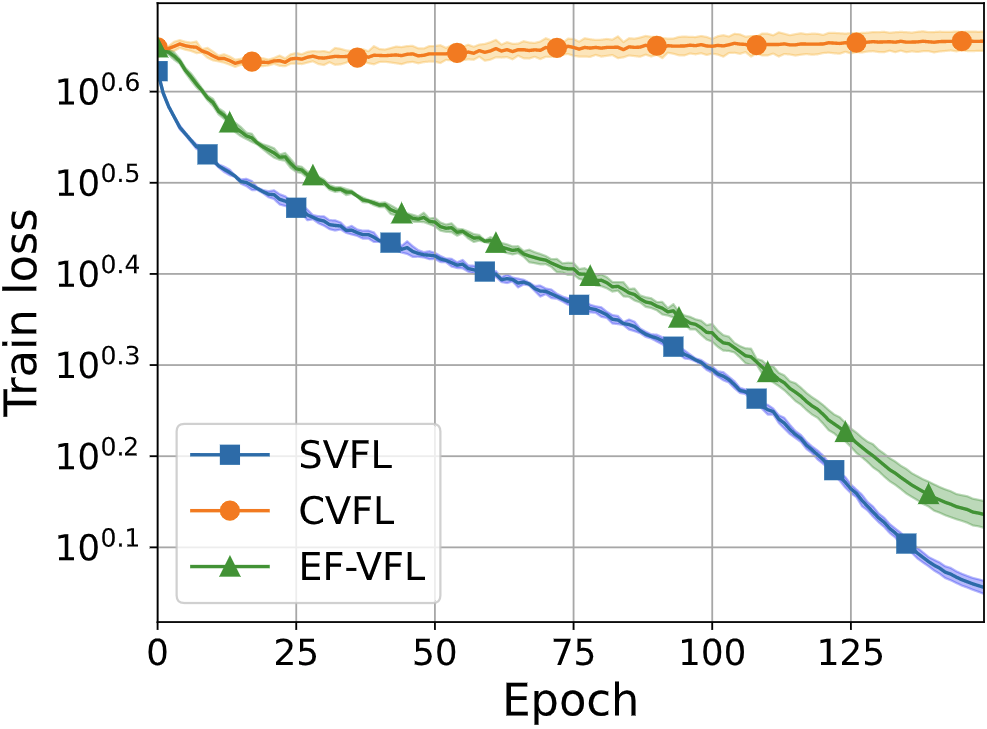}%
		\hspace{1mm}%
		\includegraphics[width=0.24\linewidth]{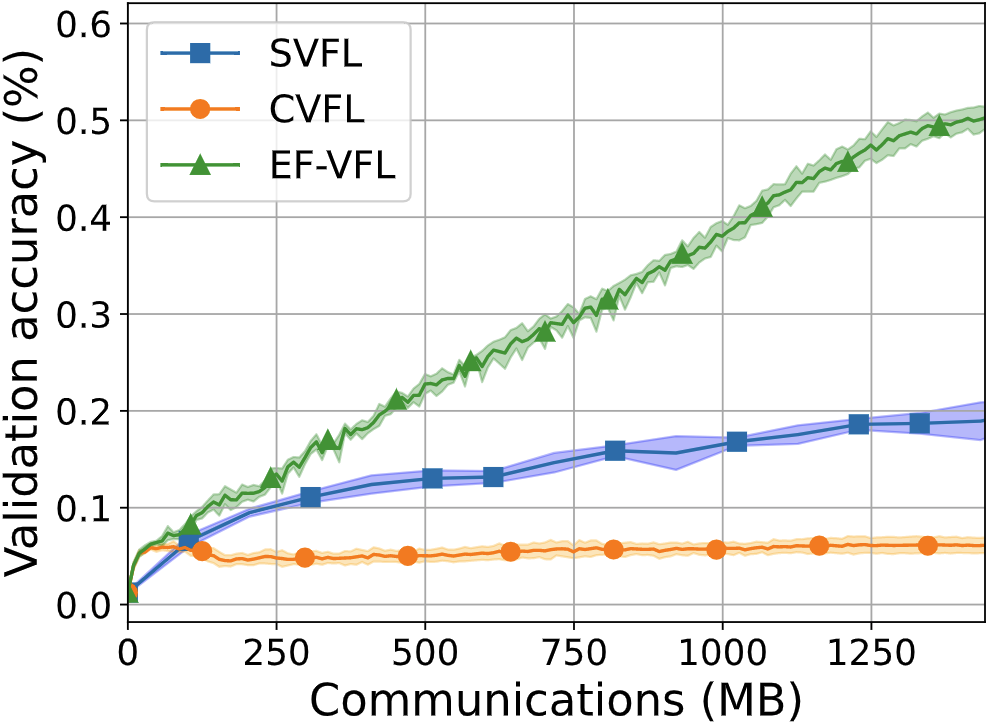}%
	}%
	\\
	\subfloat[top-$k$ keeping $0.1\%$]{%
		\includegraphics[width=0.24\linewidth]{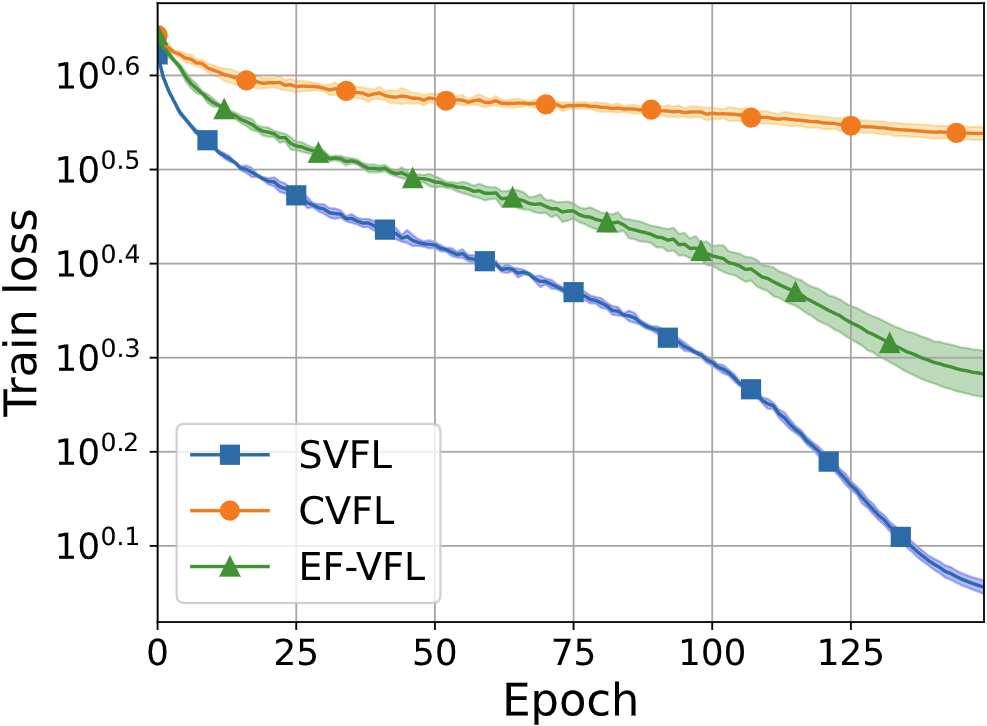}%
		\hspace{1mm}%
		\includegraphics[width=0.24\linewidth]{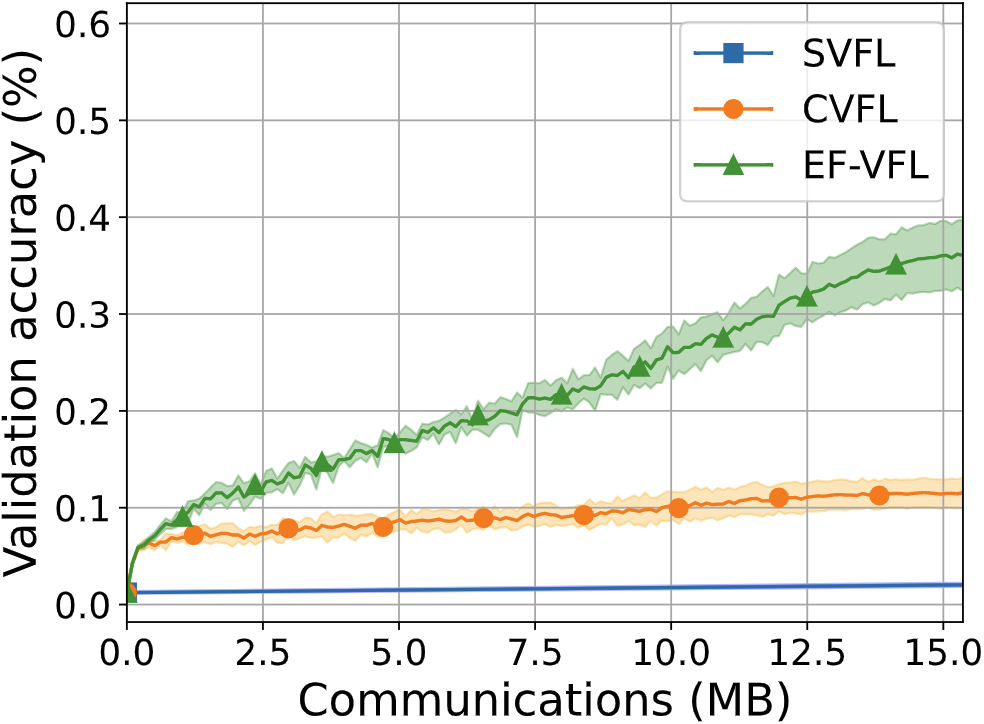}%
	}%
	\hfil%
	\subfloat[quantization with $b=1$]{%
		\includegraphics[width=0.24\linewidth]{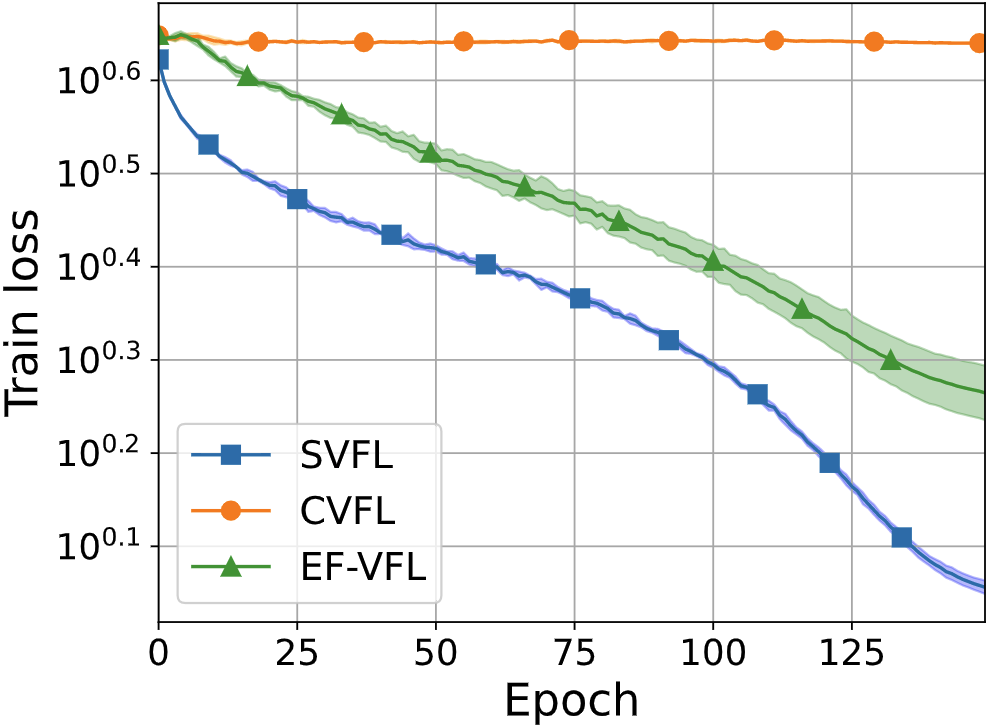}%
		\hspace{1mm}%
		\includegraphics[width=0.24\linewidth]{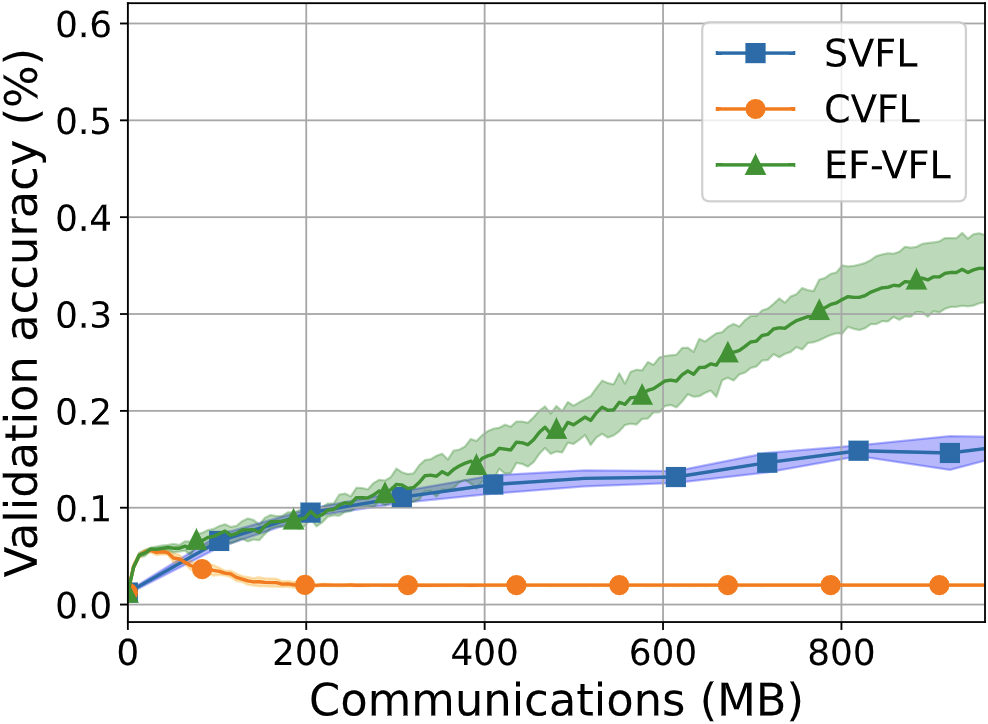}%
	}
	\caption{Train loss with respect to the number of epochs and the validation accuracy with respect to the communication cost for the training of a ResNet18-based model on CIFAR-100. On the left, \texttt{CVFL} and \texttt{EF-VFL} employ top-$k$ sparsification with a decreasing $k$ across rows. On the right, they employ stochastic quantization with a decreasing number of bits across rows. \texttt{SVFL} is the same throughout.\label{fig:resnet_all_exp}}
\end{figure*}

\subsection{Performance under private labels}
\label{sec:private_labels_exps}
In this section, we run experiments on the adaption of \texttt{EF-VFL} to handle private labels, proposed in Section~\ref{sec:efvfl_private_labels}.

In~\citet{Castiglia2022}, the authors assume that the labels are available at all clients and do not propose an adaptation of \texttt{CVFL} to deal with private labels. Yet, to get a baseline for a compressed VFL method allowing for label privacy, we adapt \texttt{CVFL} in a similar manner to how we adapt \texttt{EF-VFL} (that is, sending back the derivative from the server to the clients, instead of $\phi_n$ and $\bm{x}_0$, and without backpropagating through the compression operator).

We run an experiment on MNIST, training the same shallow neural network as in Section~\ref{sec:single_local_update}, with all the same settings, except for the use of batch size~$B=1024$. Further, we train a ResNet18-based model with a linear server model (a single layer) on CIFAR-100~\citep{krizhevsky2009learning}. For all the optimizers, we use an initial stepsize of $\eta=0.01$ and a cosine annealing scheduler with a minimum stepsize of $1/100$ the initial value and use a batch size~$B=128$ and a weight decay of $0.01$. The compressed-communication methods employ $\topk$, keeping $5\%$ of the entries.

In Figure~\ref{fig:private_labels_exp}, we observe that, although the modified \texttt{EF-VFL} for handling private labels converges noticeably slower than the original method, it still performs effectively. For both the MNIST experiment and the CIFAR-100 experiment, we see that, while adapting \texttt{CVFL} to handle private labels leads to a severe drop in performance, \texttt{EF-VFL} slows down much less noticeably. In fact, for the MNIST experiment, \texttt{EF-VFL} with private labels still outperforms \texttt{CVFL}, even with public labels.

\begin{figure*}[t!]
	\centering
	\subfloat[MNIST, top-$k$ keeping $5\%$\vspace{-0mm}]{\includegraphics[width=0.24\linewidth]{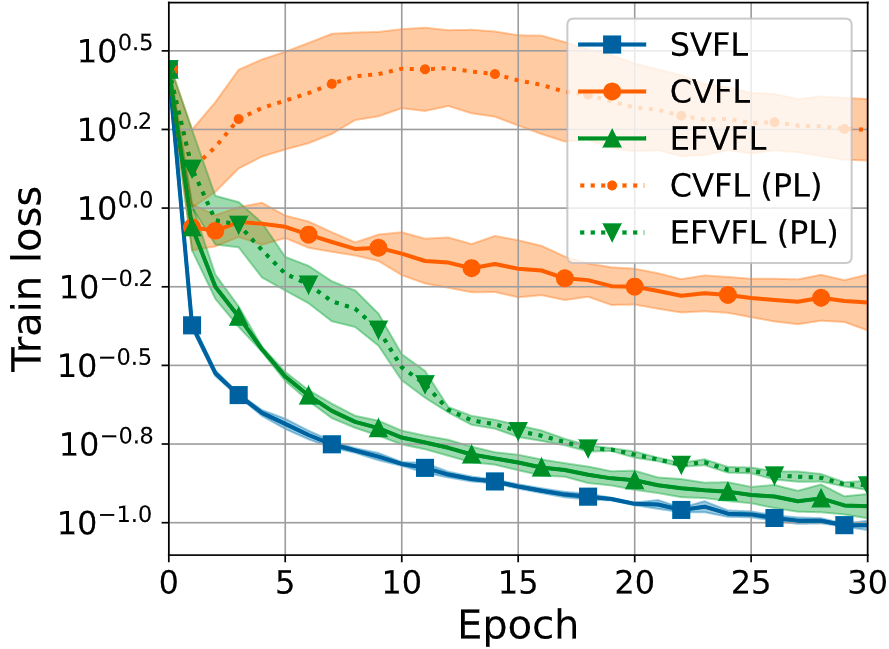}\hspace{1mm}\includegraphics[width=0.24\linewidth]{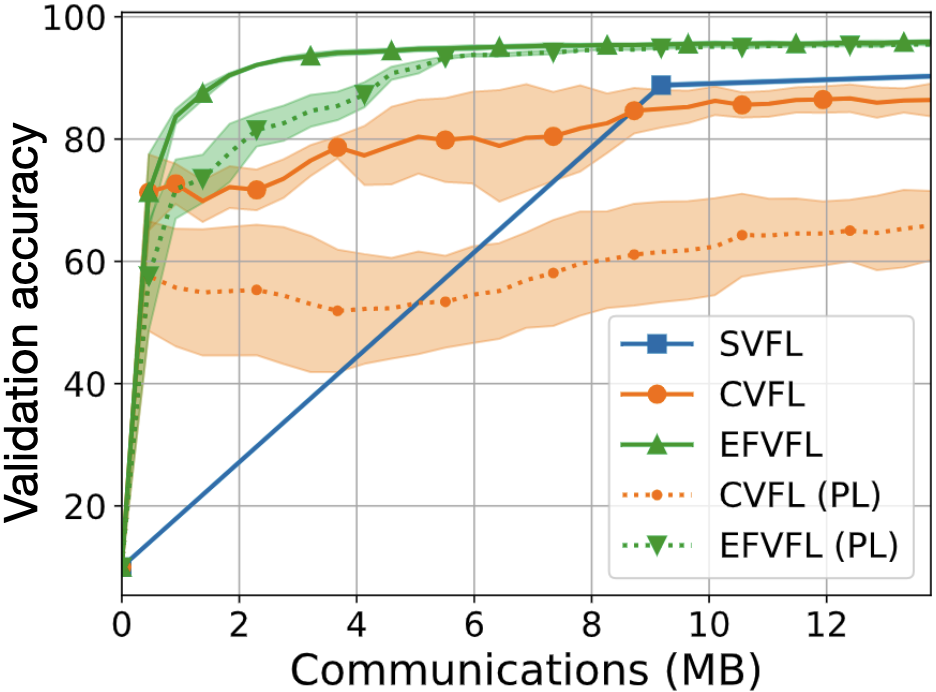}}
	\hfil
	\subfloat[CIFAR-100, top-$k$ keeping $5\%$\vspace{-0mm}]{\includegraphics[width=0.24\linewidth]{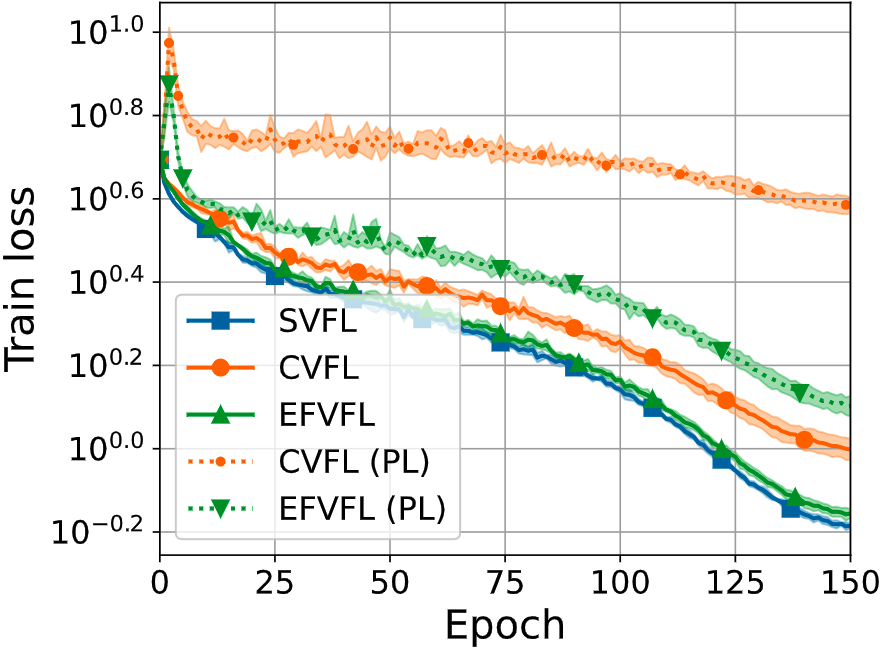}\hspace{1mm}\includegraphics[width=0.24\linewidth]{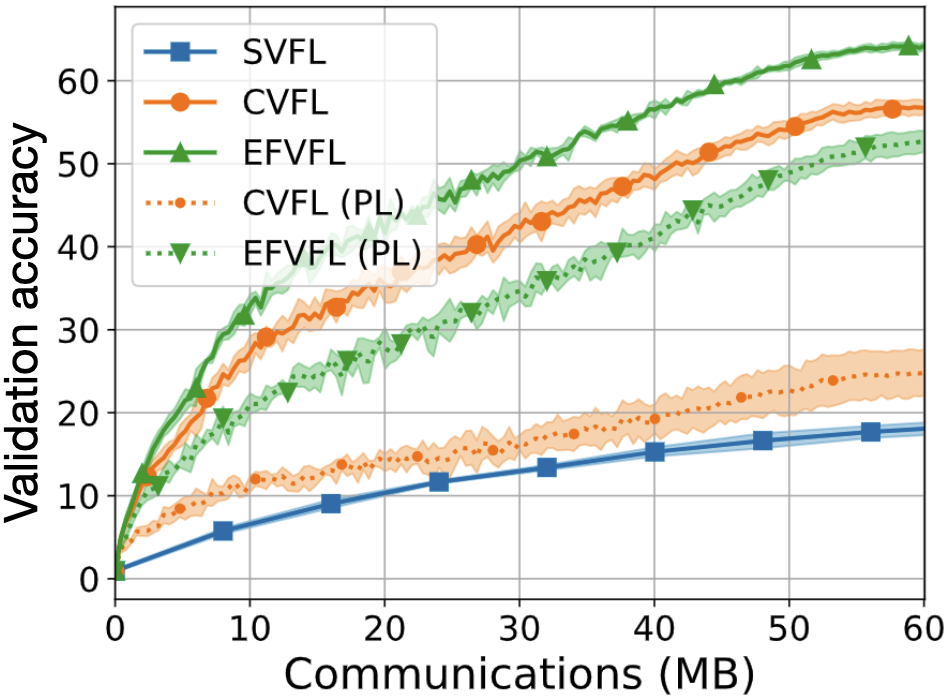}}
	\caption{Train loss with respect to the number of epochs and validation accuracy with respect to the communication cost for the training of a shallow neural network on MNIST and a ResNet18-based model on CIFAR-100. In the legend, PL stands for private labels. The communication compressed methods---\texttt{CVFL}, \texttt{EF-VFL}, \texttt{CVFL} (PL), and \texttt{EF-VFL} (PL)---employ top-$k$ sparsification.\label{fig:private_labels_exp}}
\end{figure*}

\subsection{Performance under multiple local updates}
As mentioned earlier, some VFL works employ $Q>1$ local updates per round~\citep{Liu2022}, using stale information from the other machines. We now show that, although our analysis focuses on the case where each client performs a single local update at each round of communications (that is, $Q=1$), \texttt{EF-VFL} performs well in the $Q>1$ case too. In particular, to study the performance of \texttt{EF-VFL} when carrying out multiple local updates, we train an MVCNN on ModelNet10 and a ResNet18 on CIFAR-10. 

For ModelNet10, all three VFL optimizers use a batch size~$B=128$, a stepsize~$\eta=0.004$, and a weight decay of 0.01. Further, we use a learning rate scheduler, halving the learning rate at epochs 50 and 75. The results are presented in Figure~\ref{fig:mvcnn_2q} and Figure~\ref{fig:mvcnn_4q}. For CIFAR-10, all three VFL optimizers use a batch size~$B=128$, a stepsize~$\eta=0.0025$, and a weight decay of 0.01. Further, we use a learning rate scheduler, halving the learning rate at epochs 40, 60, and 80. The results are presented in Figure~\ref{fig:resnet_2q} and Figure~\ref{fig:resnet_4q}.

For both ModelNet10 and CIFAR-10, we see that, similarly to the $Q=1$ case, our method outperforms \texttt{SVFL} and \texttt{CVFL} in communication efficiency. In terms of results per epoch, \texttt{EF-VFL} performs similarly to \texttt{SVFL} and significantly better than \texttt{CVFL}. Interestingly, for the CIFAR-10 task, \texttt{EF-VFL} even outperforms \texttt{SVFL} with respect to the number of epochs. We suspect this may be due to the fact that compression helps to mitigate the overly greedy nature of the parallel updates based on stale information.

\begin{figure*}[t]
	\centering
	\subfloat[ModelNet10 ($\text{qsgd}_s$ with $b=4$; $Q=2$)\label{fig:mvcnn_2q}]{\includegraphics[width=0.24\linewidth]{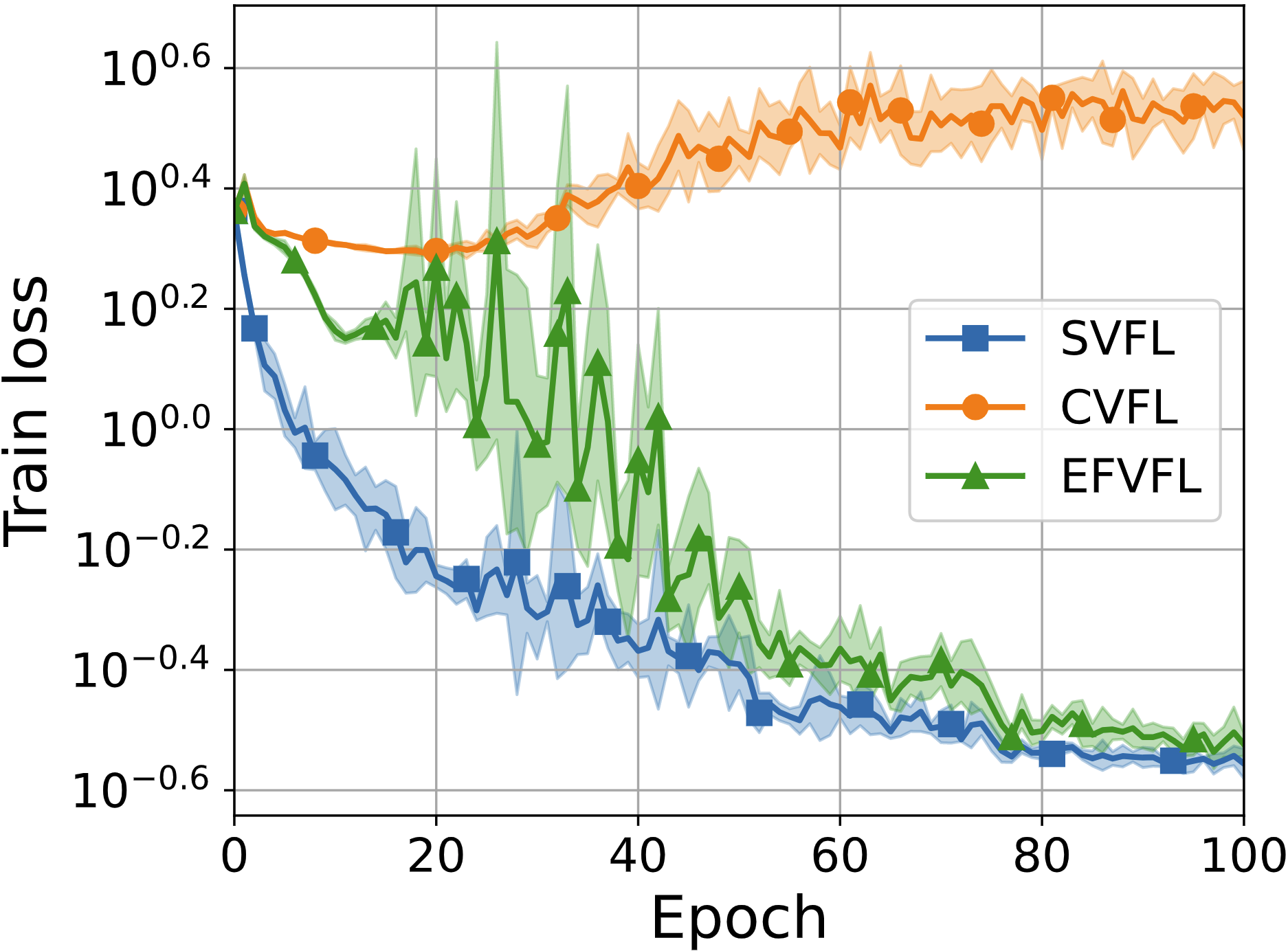}\hspace{1mm}\includegraphics[width=0.24\linewidth]{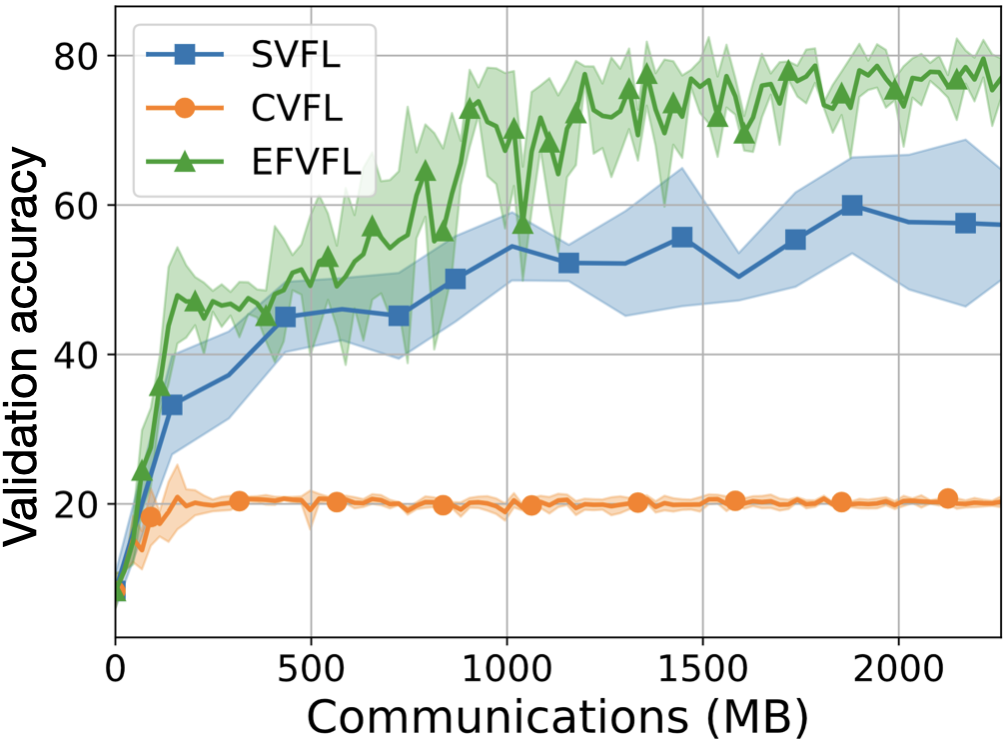}}
	\hfil
	\subfloat[ModelNet10 ($\text{qsgd}_s$ with $b=4$; $Q=4$)\label{fig:mvcnn_4q}]{\includegraphics[width=0.24\linewidth]{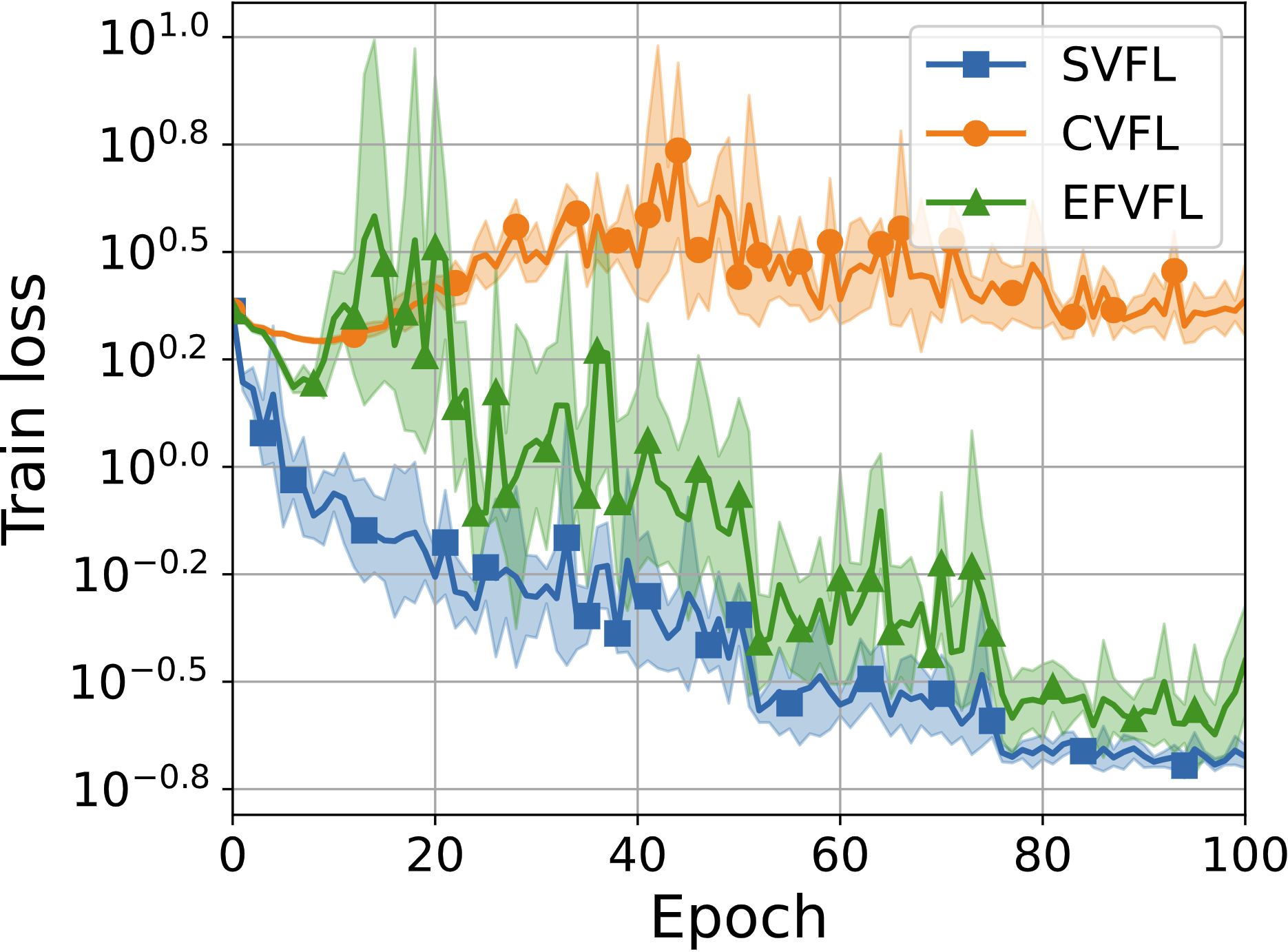}\hspace{1mm}\includegraphics[width=0.24\linewidth]{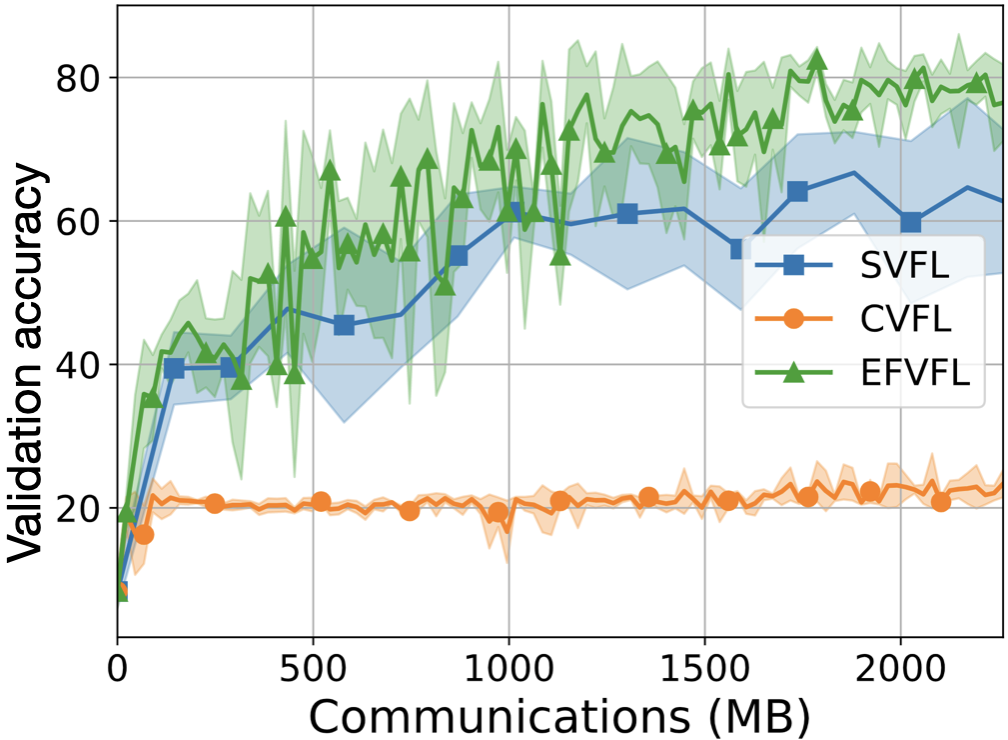}}
	\\
	\subfloat[CIFAR-10 ($\text{qsgd}_s$ with $b=1$; $Q=2$)\label{fig:resnet_2q}]{\includegraphics[width=0.24\linewidth]{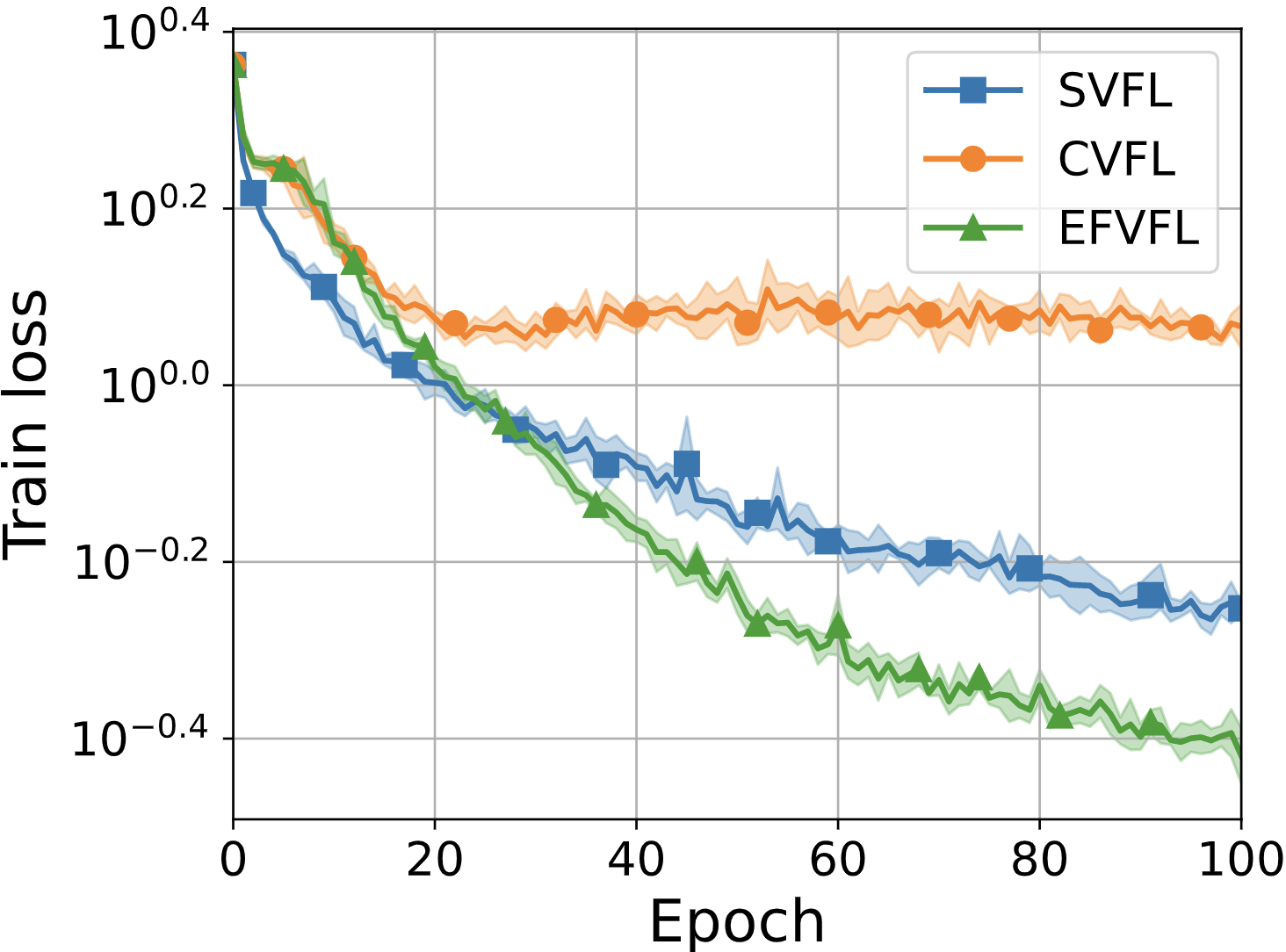}\hspace{1mm}\includegraphics[width=0.24\linewidth]{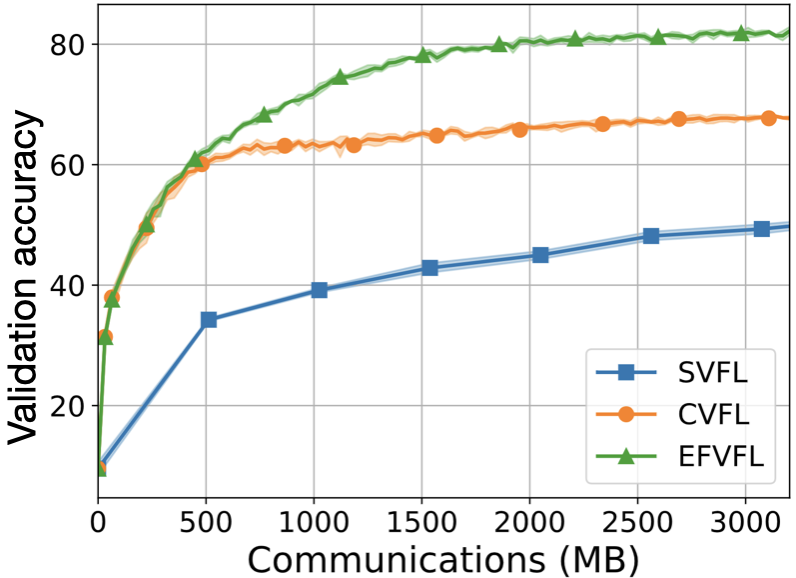}}
	\hfil
	\subfloat[CIFAR-10 ($\text{qsgd}_s$ with $b=1$; $Q=4$)\label{fig:resnet_4q}]{\includegraphics[width=0.24\linewidth]{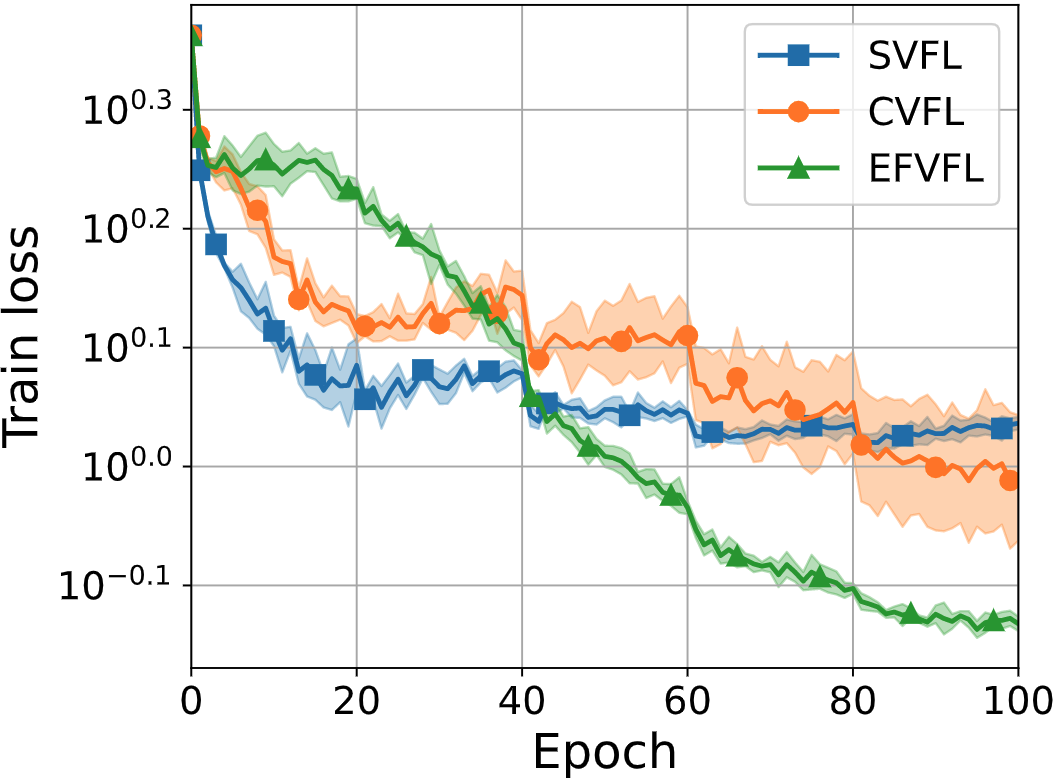}\hspace{1mm}\includegraphics[width=0.24\linewidth]{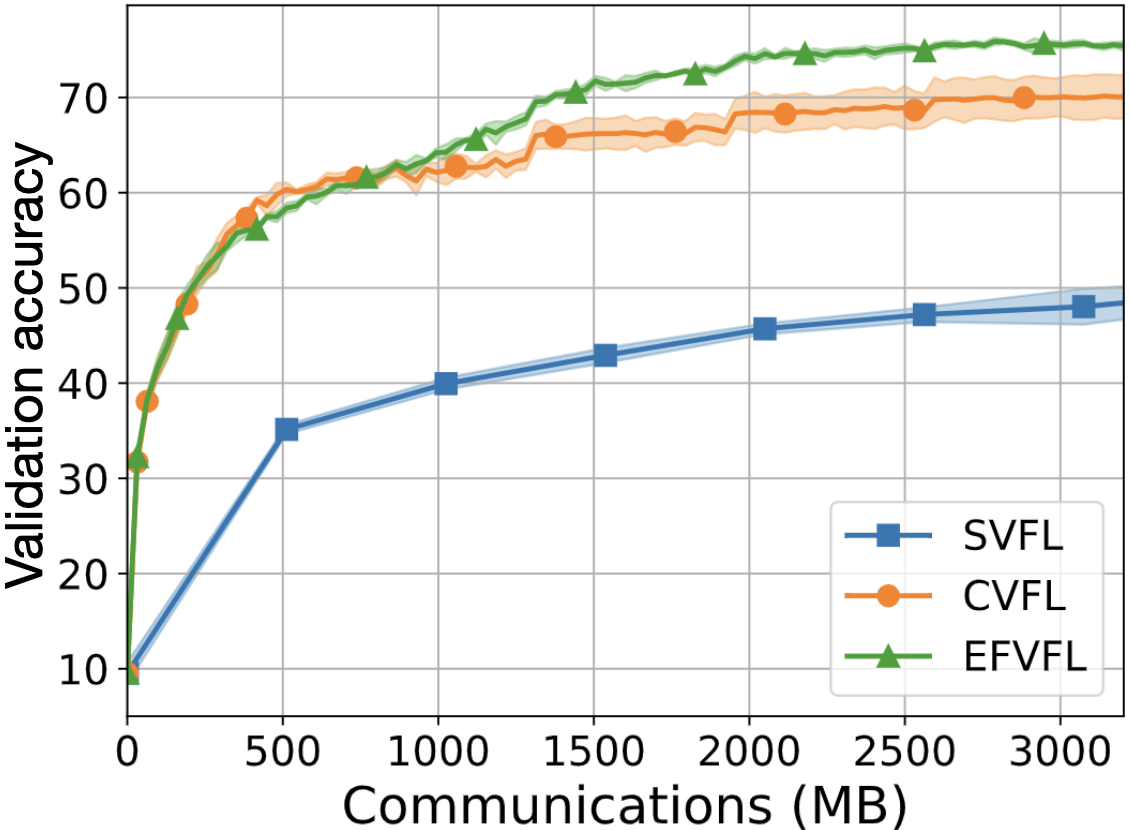}}
	\caption{Train loss with respect to the number of epochs and validation accuracy with respect to the communication cost for the training of a multi-view convolutional neural network on ModelNet10 and a residual neural network on CIFAR-10. \texttt{CVFL} and \texttt{EF-VFL} use stochastic quantization. On the left, all three vertical FL optimizers use $Q=2$ local updates and, on the right, they all use $Q=4$ local updates.\label{fig:resnet_multiple_local_updates}}
\end{figure*}

\section{Conclusions}
In this work, we proposed \texttt{EF-VFL}, a method for compressed vertical federated learning. Our method leverages an error feedback mechanism to achieve a $\mathcal{O}(1/T)$ convergence rate {for a sufficiently large batch size}, improving upon the state-of-the-art rate of $\mathcal{O}(1/\sqrt{T})$. Numerical experiments further demonstrate the faster convergence of our method. We further show that, under the PL inequality, our method converges linearly and introduce a modification of \texttt{EF-VFL} supporting the use of private labels. In the future, it would be interesting to study the use of error feedback based compression methods for VFL in the fully-decentralized and semi-decentralized settings{, in setups with asynchronous updates,} and in combination with privacy mechanisms, such as differential privacy as done in the horizontal setting \citep{li2022soteriafl,li2023convergence}.

	\section*{Acknowledgements}
	
This work is supported in part by the Fundação para a Ciência e a Tecnologia through the Carnegie Mellon Portugal Program under grant SFRH/BD/150738/2020; by the grants U.S. National Science Foundation CCF-2007911 and ECCS-2318441; by NOVA LINCS funding (DOI: 10.54499/UIDB/04516/2020 and 10.54499/UIDP/04516/2020); by LARSyS FCT funding (DOI: 10.54499/LA/P/0083/2020); by PT Smart Retail project [PRR - 02/C05-i11/2024.C645440011-00000062], through IAPMEI - Agência para a Competitividade e Inovação; and by TaRDIS Horizon2020 Contract ID: 101093006.

	\bibliography{efvfl_final}
	
	\appendix
	
\section{Preliminaries}
\label{sec:preliminaries}

If a function is $L$-smooth~\eqref{eq:lsmooth}, then the following quadratic upper bound holds:
\begin{equation} \label{eq:qub}
	\forall
	\bm{x},\bm{y}\in\mathbb{R}^d
	\colon
	\quad
	f(\bm{y})
	\leq
	f(\bm{x})
	+ \nabla f(\bm{x})^\top
	(\bm{y} - \bm{x})
	+\frac{L}{2}
	\lVert \bm{x} - \bm{y}\rVert^2.
\end{equation}

\noindent
It follows from Assumption~\eqref{ass:bounded_embedding} that the following inequality holds:
\begin{equation} \label{eq:lipschitz}
	\lVert \bm{H}_k(\bm{x}) - \bm{H}_k(\bm{y}) \rVert
	\leq
	H
	\lVert \bm{x} - \bm{y}\rVert,
	\quad
	\forall\;\bm{x},\bm{y}\in\mathbb{R}^{d_k}.
\end{equation}

\noindent
Letting $\epsilon>0$, we use the following standard inequality in our analysis:
\begin{equation} \label{eq:young_ineq}
	\forall
	\bm{x},\bm{y}\in\mathbb{R}^d
	\colon
	\quad
	\lVert \bm{x} + \bm{y} \rVert^2
	\leq
	(1+\epsilon) \lVert \bm{x} \rVert^2
	+(1+\epsilon^{-1}) \lVert \bm{y} \rVert^2.
\end{equation}

\noindent
We define the distortion associated with block~$k$ at time~$t$ as
\[
D_k^{(t)}\coloneqq\left\lVert\bm{G}_k^{t}-\bm{H}_k(\bm{x}_k^t)\right\rVert^2
\]
and, recall, we denote the total distortion at time~$t$ as $D^{(t)}=\sum_{k=0}^KD_k^{(t)}$. We also define $\nu\coloneqq(1-\alpha)\in[0,1)$.

In Section~\ref{sec:convergence_guarantees}, we introduced the following sigma-algebra
\[
\mathcal{F}_t= \sigma(\bm{G}^0,\bm{x}^1,\bm{G}^1,\dots, \bm{x}^{t},\bm{G}^{t}),
\]
where $\bm{G}^t=\{\bm{G}^t_0,\dots,\bm{G}^t_K\}$. We now further define
\[
\mathcal{F}^\prime_t\coloneqq \sigma(\bm{G}^0,\bm{x}^1,\bm{G}^1,
\dots, \bm{x}^{t},\bm{G}^{t},\bm{x}^{t+1}).
\]
Recall that we let $\mathbb{E}_{\mathcal{F}}$ denote the conditional expectation $\mathbb{E}[\cdot\mid \mathcal{F}]$.

Note that, while we write our proofs for Algorithm~\ref{alg:efvfl}, they can be easily adjusted to cover Algorithm~\ref{alg:efvfl_private_labels}. To do so, it suffices to adjust the notation, replacing $\bm{g}^t$ and $\tilde{\bm{g}}^t$ by $\bm{\nabla}^t_k$ and $\tilde{\bm{\nabla}}^t_k$, respectively, and to make minor changes to the proof of Lemma~\ref{lemma:surrogate_offset_bound}, which cause the constant $K$ in Lemma~\ref{lemma:surrogate_offset_bound} to be replaced with $K+1$. This, in turn, leads to a similar adjustment in the constants of our main theorems.

\section{Supporting Lemmas}

\subsection{Proof of Lemma~\ref{lemma:surrogate_offset_bound}} \label{app:lemma1_proof}

Decoupling the offset across blocks, we get that
\begin{align*}
	\lVert \bm{g}^{t} - & \nabla f(\bm{x}^{t}) \rVert^2\\
	&=
	\sum_{k=0}^K
	\left\lVert
	\bm{g}_k^{t}
	-\nabla_k f\left(\bm{x}_k\right)
	\right\rVert^2,
	\\
	&\leq
	\sum_{k=0}^K
	\lVert \nabla \bm{H}_k(\bm{x}_k) \rVert^2
	\left\lVert
	\tilde{\nabla}_k^t \Phi
	-
	\nabla_k \Phi
	(
	\{\bm{H}_k(\bm{x}_k^t)\}^K_{k=0}
	)
	\right\rVert^2,
\end{align*}
where we use the chain rule and the fact that $\lVert \bm{A}\bm{x}\rVert \leq \lVert \bm{A}\rVert \lVert \bm{x}\rVert$.
Now, it follows from the bounded gradient assumption~\eqref{eq:bounded_embedding} on $\{\bm{H}_k\}$ and the $L$-smoothness~\eqref{eq:lsmooth} of $\Phi$ that
\begin{align*}
	\lVert \bm{g}^{t} - \nabla f(\bm{x}^{t}) \rVert^2
	&\leq
	H^2L^2
	\sum_{k=0}^K
	\sum_{j\not=k}
	\left\lVert 
	\bm{G}_j^{t}
	-
	\bm{H}_j(\bm{x}_j^{t})
	\right\rVert^2
	\\
	&=
	H^2L^2
	\sum_{k=0}^K
	\sum_{j\not=k}
	D^{(t)}_j
	\\
	&=
	KH^2L^2D^{(t)},
\end{align*}
as we set out to prove. For Algorithm~\ref{alg:efvfl_private_labels}, the sum $\sum_{j\not=k}$ would instead be $\sum_{j=1}^K$, leading to $\lVert \bm{g}^{t} - \nabla f(\bm{x}^{t}) \rVert^2 \leq (K+1)H^2L^2D^{(t)}$. However, note that, for Algorithm~\ref{alg:efvfl_private_labels}, $D_0^{(t)}=0$.

\subsection{Proof of Lemma~\ref{lemma:distortion_recursive_bound_w_embedding_diff}} \label{app:lemma2_proof}

It follows from the definition of distortion and from the update of our compression estimate that
\begin{align*}
	\mathbb{E}_{\mathcal{F}^\prime_t}
	\left[D_k^{(t+1)}\right]
	&=
	\mathbb{E}_{\mathcal{F}^\prime_t} \left\lVert\bm{G}_k^{t+1}-\bm{H}_k(\bm{x}_k^{t+1})\right\rVert^2
	\\
	&=
	\mathbb{E}_{\mathcal{F}^\prime_t} \left\lVert
	\bm{G}_{k}^{t}+\mathcal{C}( \bm{H}_{k}(\bm{x}_k^{t+1}) -\bm{G}_{k}^{t})
	-\bm{H}_k(\bm{x}_k^{t+1})\right\rVert^2.
\end{align*}
Now, from the definition of contractive compressor~\eqref{eq:biased_compressor} and from \eqref{eq:young_ineq}, we have that
\begin{align*}
	\mathbb{E}_{\mathcal{F}^\prime_t}
	\left[D_k^{(t+1)}\right]
	&\leq
	\nu\left\lVert\bm{G}_k^{t}-\bm{H}_k(\bm{x}_k^{t+1})\right\rVert^2
	\\
	&\leq
	\nu(1+\epsilon)\left\lVert\bm{G}_k^{t}-\bm{H}_k(\bm{x}_k^{t})\right\rVert^2
	\\
	&\qquad+ \nu(1+\epsilon^{-1})\left\lVert\bm{H}_k(\bm{x}_k^{t+1}) - \bm{H}_k(\bm{x}_k^{t})\right\rVert^2,
\end{align*}
where, recall, $\nu=(1-\alpha)\in[0,1)$. Further, from the bounded gradient assumption---in particular, from~\eqref{eq:lipschitz}---we arrive at
\begin{align*}
	\mathbb{E}_{\mathcal{F}^\prime_t}
	\left[D_k^{(t+1)}\right]
	&\leq
	\nu(1+\epsilon) D_k^{(t)}
	+\nu(1+\epsilon^{-1}) H^2 \lVert \bm{x}_k^{t+1} - \bm{x}_k^{t} \rVert^2
	\\
	&=
	\nu(1+\epsilon) D_k^{(t)}
	+\nu(1+\epsilon^{-1}) \eta^2 H^2 \lVert \tilde{\bm{g}}_k^{t} \rVert^2,
\end{align*}
where, recall, $\tilde{\bm{g}}_k^{t}$ is our (possibly stochastic) update vector. Summing over $k=0,1,\dots,K$ and taking the nonconditional expectation of both sides of the inequality, we get that
\[
\mathbb{E} D^{(t+1)}
\leq
\nu(1+\epsilon) \mathbb{E} D^{(t)}
+\nu(1+\epsilon^{-1}) \eta^2 H^2
\mathbb{E} \lVert \tilde{\bm{g}}^{t} \rVert^2.
\]
Lastly, using the fact that, under~\eqref{eq:unbiased_grad}, \eqref{eq:bounded_var} is equivalent to $\mathbb{E}
\lVert \tilde{\bm{g}}^t \rVert^2
\leq
\mathbb{E}\lVert \bm{g}^t \rVert^2
+\frac{\sigma^2}{B}$, we arrive at~\eqref{eq:distortion_recursive_bound}.

\section{Main Theorems}

First, let us define some shorthand notation for terms we will be using throughout our proof, whose expectation is with respect to the (possible) randomness in the compression across all steps:
\begin{align*}
	\text{(compression error)}
	\qquad
	&
	\Omega_1^t
	\coloneqq
	\mathbb{E} D^{(t)},
	\\
	\text{(surrogate norm)}
	\qquad
	&
	\Omega_2^t
	\coloneqq
	\mathbb{E} \lVert \bm{g}^{t}\rVert^2.
\end{align*}

\subsection{Proof of Theorem~\ref{thm:efvfl_thm}}
\label{sec:efvfl_thm_proof}

From the $L$-smoothness of $f$---more specifically, from \eqref{eq:qub}---we have that
\[
\begin{split}
	f(\bm{x}^{t+1})-f(\bm{x}^{t})
	&\leq
	\langle \nabla f(\bm{x}^{t}), \bm{x}^{t+1} - \bm{x}^{t} \rangle
	+\frac{L}{2} \lVert \bm{x}^{t+1} - \bm{x}^{t} \rVert^2
	\\&=
	-\eta\langle \nabla f(\bm{x}^{t}), \tilde{\bm{g}}^t \rangle
	+\frac{\eta^2L}{2} \lVert \tilde{\bm{g}}^t \rVert^2.
\end{split}
\]
Taking the conditional expectation over the batch selection, it follows from the unbiasedness of $\tilde{\bm{g}}^t$~\eqref{eq:unbiased_grad} that
\[
\mathbb{E}_{\mathcal{F}_t}
f(\bm{x}^{t+1})-f(\bm{x}^{t})
\leq
-\eta\langle \nabla f(\bm{x}^{t}), \bm{g}^{t} \rangle
+\frac{\eta^2L}{2} \mathbb{E}_{\mathcal{F}_t} \lVert \tilde{\bm{g}}^t \rVert^2.
\]
From~\eqref{eq:bounded_var}, we have that 
$\mathbb{E}_{\mathcal{F}_t}
\lVert
\tilde{\bm{g}}^t - \bm{g}^t
\rVert^2
\leq
\frac{\sigma^2}{B}$, which, under~\eqref{eq:unbiased_grad}, is equivalent to
$\mathbb{E}_{\mathcal{F}_t}
\lVert \tilde{\bm{g}}^t \rVert^2
\leq
\lVert \bm{g}^t \rVert^2
+\frac{\sigma^2}{B}$, so
\begin{align*}
	\mathbb{E}_{\mathcal{F}_t} f(\bm{x}^{t+1}) &- f(\bm{x}^{t})\\
	&\leq
	-\eta\langle \nabla f(\bm{x}^{t}), \bm{g}^{t} \rangle
	+\frac{\eta^2L}{2} \lVert \bm{g}^t \rVert^2
	+\frac{\eta^2L\sigma^2}{2B}
	\\ & =
	-\frac{\eta}{2} \lVert \nabla f(\bm{x}^{t})\rVert^2
	-\frac{\eta}{2}(1-\eta L) \lVert \bm{g}^{t}\rVert^2
	\\ & \qquad\qquad
	+\frac{\eta}{2} \lVert \bm{g}^{t} -\nabla f(\bm{x}^{t})\rVert^2
	+\frac{\eta^2L\sigma^2}{2B},
\end{align*}
where the last equation follows from the polarization identity $\langle a,b\rangle=\frac{1}{2}(\lVert a\rVert^2+\lVert b\rVert^2-\lVert a-b\rVert^2)$. Now, using our surrogate offset bound~\eqref{eq:surrogate_offset_bound} and taking the (non-conditional) expectation, we get that:
\begin{equation} \label{eq:smoothness_and_lemma1}
	\begin{split}
		\mathbb{E}f(\bm{x}^{t+1})-\mathbb{E}f(\bm{x}^{t})
		&\leq
		-\frac{\eta}{2} \mathbb{E} \lVert \nabla f(\bm{x}^{t})\rVert^2
		-\frac{\eta}{2}(1-\eta L) \mathbb{E} \lVert \bm{g}^{t}\rVert^2
		\\&\quad
		+\frac{\eta KH^2L^2}{2} \mathbb{E} D^{(t)}
		+\frac{\eta^2L\sigma^2}{2B}.
	\end{split}
\end{equation}

Using the $\Omega_1^t$ and $\Omega_2^t$ notation defined earlier and recalling that $\nu=(1-\alpha)\in[0,1)$, we rewrite \eqref{eq:distortion_recursive_bound} and \eqref{eq:smoothness_and_lemma1}, respectively, as
\[
\Omega_1^{t+1}
\leq
\nu(1+\epsilon) \Omega_1^t
+\nu(1+\epsilon^{-1}) \eta^2 H^2 \Omega_2^t + \frac{\nu(1+\epsilon^{-1}) \eta^2 H^2\sigma^2}{B}
\]
and
\[
\begin{split}
	\mathbb{E}f(\bm{x}^{t+1})-\mathbb{E}f(\bm{x}^{t})
	&\leq
	-\frac{\eta}{2} \mathbb{E} \lVert \nabla f(\bm{x}^{t})\rVert^2
	+\frac{\eta KH^2L^2}{2} \Omega_1^t
	\\&\quad
	-\frac{\eta}{2}(1-\eta L) \Omega_2^t
	+\frac{\eta^2L\sigma^2}{2B}.
\end{split}
\]

Multiplying the first inequality by a positive constant $w$ and adding it to the second one, we get
\begin{equation} \label{eq:pre_descent_like_lemma}
	\begin{split}
		\mathbb{E}f(\bm{x}^{t+1})-\mathbb{E}f(&\bm{x}^{t})
		+ w \Omega_1^{t+1}
		-\psi_1(w)
		\Omega_1^{t}
		\\
		&\leq
		-\frac{\eta}{2}
		\mathbb{E} \lVert \nabla f(\bm{x}^{t})\rVert^2
		+ \psi_2(w)
		\Omega_2^{t}
		\\&\quad+
		\left( w\nu(1+\epsilon^{-1}) H^2
		+\frac{L}{2} \right)
		\frac{\eta^2\sigma^2}{B}
		,
	\end{split}
\end{equation}
where
\[
\psi_1 (w)
\coloneqq
w \nu(1+\epsilon) + \frac{\eta KH^2L^2}{2}
\]
and
\[ \psi_2 (w)
\coloneqq
w\nu(1+\epsilon^{-1})\eta^2 H^2 - \frac{\eta}{2}(1-\eta L). \]

Looking at \eqref{eq:pre_descent_like_lemma}, we see that, if $\psi_2 (w) \leq 0$, we can drop the $\Omega_2^t$ term. Further, if $\psi_1 (w)\leq w$, we can telescope the $\Omega_1^t$ term as we sum the inequalities for $t=0,\dots,T-1$, as we do for the $\mathbb{E}f(\bm{x}^{t})$ terms. We thus get that:
\begin{equation} \label{eq:pre_eta_w_ineq}
	\begin{split}
		\frac{1}{T} \sum_{t=0}^{T-1} \mathbb{E} \lVert \nabla f(\bm{x}^{t})\rVert^2
		&\leq
		\frac{2(f(\bm{x}^{0})-\mathbb{E}f(\bm{x}^{T}))}{\eta T}
		\\&\quad+
		\frac{2w (\Omega_1^0-\Omega_1^T)}{\eta T}
		\\&\quad+
		\left( 2w\nu(1+\epsilon^{-1}) H^2
		+L \right)
		\frac{\eta\sigma^2}{B},
	\end{split}
\end{equation}
for
\[
w \in \mathcal{W}_\epsilon
\coloneqq
\left\{
w
\colon
\frac{\eta KH^2L^2}{2(1-\nu(1+\epsilon))}
\leq
w
\leq
\frac{1-\eta L}{2\eta H^2 \nu(1+\epsilon^{-1})}
\right\}
,
\]
where the lower bound follows from $\psi_1 (w)\leq w$ and the upper bound from $\psi_2 (w) \leq 0$.

\vspace{1mm}
\paragraph{Bounding $\eta$ and choosing $\epsilon$.}
To ensure that $\mathcal{W}_\epsilon$ is not empty, we need
\[
\eta^2
\gamma(\epsilon)
L^2
+
\eta
L
\leq
1
\quad
\text{where}
\quad
\gamma(\epsilon)
\coloneqq
KH^4
\frac{\nu(1+\epsilon^{-1})}{1-\nu(1+\epsilon)}
.
\]
From Lemma 5 of~\citet{richtarik2021ef21}, we know that, if $a,b>0$, then $0\leq\eta\leq\frac{1}{\sqrt{a}+b}$ implies $a\eta^2+b\eta\leq1$. Thus, we can ensure that $\mathcal{W}_\epsilon$ is not empty by requiring
\[
\eta
\leq
\frac{1}{\sqrt{\gamma(\epsilon) L^2}+L}
=
\left( \sqrt{\gamma(\epsilon)} L+L\right)^{-1}
.
\]
Further, to ensure that all $w\in\mathcal{W}_\epsilon$ are positive, we need $\nu(1+\epsilon)<1$, which holds for $\epsilon<\frac{1-\nu}{\nu}$. Thus, to have the largest upper bound possible on the stepsize~$\eta$, we want $\epsilon$ to be the solution to the following optimization problem, solved in Lemma 3 of~\citet{richtarik2021ef21}:
\[
\begin{split}
	\epsilon^\star
	&\coloneqq
	\argmin_\epsilon
	\left\{
	\tilde{\gamma}(\epsilon)\coloneqq
	\frac{\nu(1+\epsilon^{-1})}{1-\nu(1+\epsilon)}
	\colon
	0<\epsilon<\frac{1-\nu}{\nu}
	\right\}
	\\&=
	\frac{1}{\sqrt{\nu}}-1.
\end{split}
\]
It follows that $\sqrt{\tilde{\gamma}(\epsilon^\star)}=\frac{1+\sqrt{1-\alpha}}{\alpha}-1$ and thus $\gamma(\epsilon^\star)=KH^4\left(\frac{1+\sqrt{1-\alpha}}{\alpha}-1\right)^2\eqqcolon\rho_{\alpha1}$. We therefore need
\[
\eta
\leq
\left( \sqrt{\gamma(\epsilon^\star)} L+L\right)^{-1}
=
\left( \sqrt{\rho_{\alpha1}} L+L\right)^{-1}
.
\]
Note that, for $\alpha=1$, we recover $\eta\leq1/L$.


\vspace{1mm}
\paragraph{Choosing $w$.}
Now, since $f^\star\leq f(\bm{x})$ for all $\bm{x}$ and $\Omega_1^T\geq 0$, we have from~\eqref{eq:pre_eta_w_ineq} that, for all $ w\in\mathcal{W}_\epsilon $:
\[
\begin{split}
	\frac{1}{T} \sum_{t=0}^{T-1} \mathbb{E} \lVert \nabla f(\bm{x}^{t})\rVert^2
	&\leq
	\frac{2\Delta}{\eta T}
	+ \frac{2w \Omega_1^0}{\eta T}
	\\&\quad+
	\left( 2w\nu(1+\epsilon^{-1}) H^2
	+L \right)
	\frac{\eta\sigma^2}{B},
\end{split}
\]
where $\Delta\coloneqq f(\bm{x}^{0})-f^\star$. From the inequality above, we see that we want $w\in\mathcal{W}_\epsilon$ to be as small as possible. Therefore, we take $w$ to be the lower bound in $\mathcal{W}_\epsilon$. Since $1-\nu(1+\epsilon^\star)=1-\sqrt{\nu}$, this corresponds to setting $w=\frac{\eta KH^2L^2}{2(1-\sqrt{\nu})}$. Recalling that $\Omega_1^t=\mathbb{E} D^{(t)}$ and $\nu=1-\alpha$, we thus arrive at~\eqref{eq:nonconvex_thm}:
\[
\begin{split}
	\frac{1}{T} \sum_{t=0}^{T-1} \mathbb{E} \lVert \nabla f(\bm{x}^{t})\rVert^2
	&\leq
	\frac{2\Delta}{\eta T}
	+
	\frac{KH^2L^2}{1-\sqrt{1-\alpha}}
	\cdot
	\frac{\mathbb{E} D^{(0)}}{T}
	\\&\quad+
	\left(
	\eta L
	\rho_{\alpha1}
	+1 \right)
	\frac{\eta L\sigma^2}{B}
	.
\end{split}
\]

\subsection{Proof of Theorem~\ref{thm:pl}}
\label{sec:efvfl_thm_proof_pl}

Recall that, using the $\Omega_1^t$ and $\Omega_2^t$ notation, we can rewrite \eqref{eq:distortion_recursive_bound} and \eqref{eq:smoothness_and_lemma1}, respectively, as
\[ \Omega_1^{t+1}
\leq
\nu(1+\epsilon) \Omega_1^t
+\nu(1+\epsilon^{-1}) \eta^2 H^2 \Omega_2^t + \frac{\nu(1+\epsilon^{-1}) \eta^2 H^2\sigma^2}{B}\]
and
\[
\begin{split}
	\mathbb{E}f(\bm{x}^{t+1})-\mathbb{E}f(\bm{x}^{t})
	&\leq
	-\frac{\eta}{2} \mathbb{E} \lVert \nabla f(\bm{x}^{t})\rVert^2
	+\frac{\eta KH^2L^2}{2} \Omega_1^t
	\\&\quad-
	\frac{\eta}{2}(1-\eta L) \Omega_2^t
	+\frac{\eta^2L\sigma^2}{2B}.
\end{split}
\]

Now, from our earlier introduced Lyapunov function~\eqref{eq:lyapunov}, $V_{t}=
\mathbb{E}f(\bm{x}^t)-f^\star+c\Omega_1^t$, we have that:
\begin{align*}
	V_{t+1}
	&=
	\mathbb{E}f(\bm{x}^{t+1})-f^\star+c\Omega_1^{t+1}
	\\
	& \stackrel{\text{(i)}}{\leq}
	\mathbb{E}f(\bm{x}^{t}) - f^\star
	-\frac{\eta}{2} \mathbb{E} \lVert \nabla f(\bm{x}^{t})\rVert^2
	\\&\quad+ \left(\frac{\eta KH^2L^2}{2}+c\nu(1+\epsilon)\right) \Omega_1^{t}
	+ \psi_2 (c) \Omega_2^{t}
	\\&\quad+(L+2c\nu(1+\epsilon^{-1})H^2)\frac{\eta^2\sigma^2}{2B}
	\\
	& \stackrel{\text{(ii)}}{\leq}
	(1-\eta\mu) (\mathbb{E}f(\bm{x}^{t}) - f^\star)
	\\&\quad+ \left(\frac{\eta KH^2L^2}{2}+c\nu(1+\epsilon)\right) \Omega_1^{t}
	\\&\quad+ \psi_2 (c) \Omega_2^{t}
	+(L+2c\nu(1+\epsilon^{-1})H^2)\frac{\eta^2\sigma^2}{2B}
	\\
	& =
	(1-\eta\mu)V_{t}
	\\&\quad+ \underbrace{\left(\frac{\eta KH^2L^2}{2}+c\nu(1+\epsilon)-c(1-\eta\mu)\right)}_{\eqqcolon\psi_3(c)} \Omega_1^{t}
	\\&\quad+ \psi_2(c) \Omega_2^{t}
	+(L+2c\nu(1+\epsilon^{-1})H^2)\frac{\eta^2\sigma^2}{2B},
\end{align*}
where (i) follows from \eqref{eq:distortion_recursive_bound}, \eqref{eq:smoothness_and_lemma1}, $c>0$, and $\psi_2 (w) = w\nu(1+\epsilon^{-1})\eta^2 H^2 - \frac{\eta}{2}(1-\eta L)$ and (ii) follows from the PL inequality~\eqref{eq:pl}. Looking the inequality above, we see that, if there is a $c$ such that $\psi_2(c),\psi_3(c)\leq0$, then
\begin{equation}\label{eq:recursive_lyapunov_bound}
	V_{t+1}
	\leq
	(1-\eta\mu)V_{t} +(L+2cH^2\nu(1+\epsilon^{-1}))\frac{\eta^2\sigma^2}{2B}
	.
\end{equation}
Note that, similarly to what we had in the proof for Theorem~\ref{thm:efvfl_thm}, $\psi_2(c)\leq0$ corresponds to a upper bound on $c$, while $\psi_3(c)\leq0$ corresponds to an lower bound on $c$. We therefore want $c\in
\mathcal{W}_\epsilon^\prime$, where
\[
\mathcal{W}_\epsilon^\prime
\coloneqq
\left\{
c
\colon
\frac{\eta KH^2L^2}{2(1-\nu(1+\epsilon)- \eta\mu )}
\leq
c
\leq
\frac{1-\eta L}{2\eta\nu(1+\epsilon^{-1})H^2}
\right\}
.
\]
Recursing~\eqref{eq:recursive_lyapunov_bound}, we get
\begin{equation} \label{eq:pl_bound}
	\begin{split}
		V_{T}
		& \leq
		(1-\eta\mu)^TV_{0}
		+(L+2cH^2\nu(1+\epsilon^{-1}))\frac{\eta^2\sigma^2}{2B}
		\sum_{t=0}^{T-1}(1-\eta\mu)^t \notag
		\\
		& \leq
		(1-\eta\mu)^TV_{0}
		+(L+2cH^2\nu(1+\epsilon^{-1}))\frac{\eta\sigma^2}{2B\mu},
	\end{split}
\end{equation}
where the second inequality follows from the sum of a geometric series, arriving at the result we set out to prove.


\vspace{1mm}
\paragraph{Choosing $\epsilon$ and bounding $\eta$ so that $\mathcal{W}_\epsilon^\prime$ is nonempty.}
Note that the lower bound defining $\mathcal{W}_\epsilon^\prime$ is positive if $\eta<\frac{1-\nu(1+\epsilon)}{\mu}$, where $1-\nu(1+\epsilon)>0$ as long as $\epsilon<\frac{1-\nu}{\nu}$. Further, $\mathcal{W}_\epsilon^\prime$ is not empty, as long as
\[
\frac{\eta KH^2L^2}{2(1-\nu(1+\epsilon)- \eta\mu )}
\leq
\frac{1-\eta L}{2\eta\nu(1+\epsilon^{-1})H^2},
\]
which is equivalent to
\[
\eta^2
L^2
\left(\beta_\epsilon(\alpha) K H^4-\mu/L\right)
+
\eta
L(\theta_\epsilon(\alpha)+\mu/L)
\leq
\theta_\epsilon(\alpha),
\]
where
\[
\theta_\epsilon(\alpha)=1-(1-\alpha)(1+\epsilon)
\quad
\text{and}
\quad
\beta_\epsilon(\alpha)=(1-\alpha)(1+\epsilon^{-1}).
\]
If $\epsilon\leq\min\left\{\frac{1-\alpha}{\alpha},\frac{\alpha}{1-\alpha} \right\}$, we have that $\beta_\epsilon(\alpha)\geq1$ and $\theta_\epsilon(\alpha)\geq0$ for all $\alpha$. It follows from $\beta_\epsilon(\alpha)\geq1$ that $\beta_\epsilon(\alpha) K H^4\geq K H^4\geq H^4\geq 1$, where the last inequality follows from \eqref{eq:bounded_embedding} holding for $\bm{H}_0$. We thus get that
\begin{equation} \label{eq:pre_stepsize_bound}
	\eta^2 L^2 \left(1-\mu/L\right) + \eta L(\mu/L)
	\leq
	\theta_\epsilon(\alpha)
	.
\end{equation}
Choosing $\epsilon$ to be
\[
\epsilon^\star
=
\begin{cases}
	\alpha, \quad & 0<\alpha\leq1/2, \\
	1-\alpha, \quad & 1/2<\alpha\leq1,
\end{cases}
\]
we get that
\[
\theta_{\epsilon^\star}(\alpha)
=
\begin{cases}
	\alpha^2, \quad & 0<\alpha\leq1/2, \\
	-1+3\alpha-\alpha^2, \quad & 1/2<\alpha\leq1.
\end{cases}
\]
Since $\alpha^2\leq-1+3\alpha-\alpha^2$ for $\alpha\in(1/2,1]$, we have that $\alpha^2\leq\theta_{\epsilon^\star}(\alpha)$
for all $\alpha\in(0,1]$. Thus, \eqref{eq:pre_stepsize_bound} holds if
\begin{equation} \label{eq:stepsize_bound}
	\eta^2 L^2 \left(1-\mu/L\right) + \eta L(\mu/L)
	\leq
	\alpha^2
	.
\end{equation}
Further, from \eqref{eq:lsmooth} and \eqref{eq:pl}, we get that $0\leq\mu/L\leq 1$. Therefore, for a sufficiently small $\eta$, there exists a positive $c\in\mathcal{W}_\epsilon^\prime$ such that $\psi_2(c),\psi_3(c)\leq0$.
Lastly, note that we can also guarantee that
$\eta<\frac{1-\nu(1+\epsilon^\star)}{\mu}
=\theta_{\epsilon^\star}(\alpha)/\mu$ by having $\eta<\alpha^2/\mu$, which follows from \eqref{eq:stepsize_bound}.

{\paragraph{Choosing $c$ to minimize the upper bound.}
	From~\eqref{eq:pl_bound}, we see that we want $c$ to be as small as possible. So, we choose $c$ as the lower bound in the definition of $\mathcal{W}_\epsilon^\prime$, arriving at 
	\[
	\begin{split}
		V_{T}
		&\leq
		(1-\eta\mu)^TV_{0}\\
		&+
		\left(
		L+2
		\left(\frac{\eta KH^2L^2}{2(1-\nu(1+\epsilon)- \eta\mu )}\right)
		H^2\nu\left(1+\epsilon^{-1}\right)
		\right)
		\frac{\eta\sigma^2}{2B\mu}.
	\end{split}
	\]
	Now, we know that, for $\epsilon=\epsilon^\star$ and $\eta^2 L^2 \left(1-\mu/L\right) + \eta \mu \leq \alpha^2$, the lower bound in the definition of $\mathcal{W}_\epsilon^\prime$ is less than or equal to the upper bound. We therefore have that
	\[
	2
	\left(\frac{\eta KH^2L^2}{2(1-\nu(1+\epsilon)- \eta\mu )}\right)
	H^2\nu\left(1+\epsilon^{-1}\right)
	\leq
	\frac{1-\eta L}{\eta}.
	\]
	Using this inequality in the bound above it follows that
	\[
	V_{T}
	\leq
	(1-\eta\mu)^TV_{0}
	+\frac{\sigma^2}{2B\mu},
	\]
	thus arriving at the statement that we set out to prove.}


\section{Comparison of different downlink communication schemes} \label{app:downlink_schemes}

{
	As in most communication-compressed optimization literature~\citep{haddadpour2021federated}, our primary concern is uplink communications, which are typically the main bottleneck in training. Nevertheless, this appendix discusses three alternative downlink communication schemes in \texttt{EF-VFL}: \textbf{1)} the one in Algorithm~\ref{alg:efvfl}, \textbf{2)} the one in Algorithm~\ref{alg:efvfl_private_labels}, and \textbf{3)} a modified version of the one in Algorithm~\ref{alg:efvfl} for common VFL fusion models, enabling broadcasts of a size that is independent of the number of clients, $K$. Approaches \textbf{1)} and \textbf{3)} are mathematically equivalent, yet Approach \textbf{2)} is not, as discussed earlier. Each approach has its pros and cons, making it suitable for different applications. For simplicity, this discussion focuses on top-$k$ sparsification and the full-batch case.
}

{
	\textbf{1)}
	In Algorithm~\ref{alg:efvfl}, each round of downlink communications consists of a broadcast of size $k(K+1)$—a compressed object of size $k$ for each client (the intermediate representations) and one for the server (the fusion model).
}

{
	\textbf{2)}
	In Algorithm~\ref{alg:efvfl_private_labels}, each client receives only the derivative of the loss function with respect to its representation, resulting in a total downlink communication cost of $kK$. Recall that this is only an option when performing a single local update.
}

{
	Approach~\textbf{2)} avoids the dependency of the downlink communications to each client on $K$, seen in Approach~\textbf{1)}, but requires $K$ different communications (one to each client), rather than a single broadcast, thus the total communication cost still depends on $K$. The one-to-many nature of broadcasting makes it is more appropriate to compare broadcasted information with the total downlink communications, rather than the communication to a single client, as the latter ignores the cost of contacting the other $K-1$ clients.
}

{
	\textbf{3)}
	To ensure that the downlink communication cost does not depend on $K$, we can often exploit the structure of the fusion model~$\phi$. In particular, a common choice is $\phi(\bm{x})=\phi_2(\bm{x}_0,\phi_1(\bm{H}_1(\bm{x}_1),\dots,\bm{H}_K(\bm{x}_K)))$, where $\phi_1$ is a nonparameterized representation aggregator, such as a sum or an average, and $\phi_2$ is a map parameterized by $\bm{x}_0$. In this case, instead of broadcasting $K+1$ objects, as in Approach~\textbf{1)}, the server can broadcast the aggregation of the representations, $\phi_1(\{\bm{H}_j(\bm{x}_j^t)\})$. This allows us to collapse the dimension of length $K$, as long as each client $i$ can replace $ \bm{H}_i(\bm{x}_i^t) $ with $ \bm{H}_i(\bm{x}_i^{t+1}) $ in $\phi_1(\{\bm{H}_j(\bm{x}_j^t)\})$ using its local knowledge of its own representation. For example, if $\phi_1$ is a sum, client $i$ can subtract its previous intermediate representation and add the updated one to obtain an updated aggregation. This allows client~$i$ to perform forward and backward passes over both its local model and the fusion model, and thus perform multiple local updates without requiring further communications. Yet, this downlink communication of the aggregated representations will no longer be in the range of the compressor. For example, if $\bm{v}_1$ and $\bm{v}_2$ are within the range of top-$k$, their sum, $\bm{v}_1 + \bm{v}_2$, will generally not be. Therefore, we have a broadcast of up to size $N\bar{E}+d_0$, where $d_0$ is the size of the parameters of the fusion model. That is, we avoid the dependency on $K$, but this comes at the cost of losing the compressed nature of the downlink communications. (This sum may still lie in a lower-dimensional manifold, but this typically recovers the dependency on $K$, e.g., for top-$k$ sparsification, we can upper bound the number of nonzero entries of the sum of $K$ $k$-sparse vectors by $\min\{kK,N\bar{E}\}+d_0$.) Like Approach~\textbf{1)}, Approach~\textbf{3)} does not allow for private labels.
}

{
	We present Approach \textbf{1)} in Algorithm~\ref{alg:efvfl}, rather than Approach \textbf{3)}, because most VFL applications are in the cross-silo setting~\citep{liu2024vertical} and thus the number of clients $K$ is small, therefore $k(K+1)\ll N\bar{E}+d_0$. Yet, for applications where $K$ is large, Approach \textbf{3)} may be preferable.
}

	

	
\end{document}